\documentclass{article}
\usepackage[sorting=none,doi=false,isbn=false,giveninits=true,style=nature,url=false]{biblatex} 
\usepackage{csquotes}
\usepackage[english]{babel}
\usepackage{authblk}
\date{}
\usepackage[a4paper,top=2cm,bottom=2cm,left=3cm,right=3cm,marginparwidth=1.75cm]{geometry}

\usepackage{amsmath}
\usepackage{amssymb}
\usepackage{amsfonts}
\usepackage{mathtools}
\usepackage{amsthm}
\usepackage{graphicx}
\usepackage{bbm}
\usepackage{MnSymbol}
\usepackage{algorithm}
\usepackage{algpseudocode}
\usepackage{float}
\usepackage{enumerate}
\usepackage{enumitem}
\usepackage{changepage}
\usepackage{scalerel}
\usepackage{tikz}
\usepackage{subcaption}
\usetikzlibrary{patterns.meta}
\usetikzlibrary{decorations.pathreplacing}
\usepackage[colorlinks=true, allcolors=blue]{hyperref}

\usepackage{cleveref}

\newtheorem{theorem}{Theorem}[section]

\newtheorem{assumption}{Property}
\newtheorem{proposition}[theorem]{Proposition}
\newtheorem{corollary}[theorem]{Corollary}
\newtheorem{remark}[theorem]{Remark}
\newtheorem{lemma}[theorem]{Lemma}

\newcommand{\ma}[2]{\overline{m}^{#2}\left(#1\right)}

\newcommand{\mb}[2]{\overline{m}_{#1}\left(#2\right)}

\newcommand{\ya}[2]{\overline{y}^{#2}\left(#1\right)}
\newcommand{\yb}[2]{\overline{y}_{#1}\left(#2\right)}
\newcommand{\mmed}[3]{\iota_{#1}^{#2}\left(#3 \right)}
\newcommand{\med}[1]{\iota\left(#1\right)}

\newcommand{\ALGO}{\texttt{SoHLoB}}
\newcommand{\ALGOS}{\texttt{MultSoHLoB}}

\title{Optimal level set estimation for non-parametric tournament and crowdsourcing problems}

\author[1]{Maximilian Graf}
\author[1]{Alexandra Carpentier}
\author[2]{Nicolas Verzelen}

\affil[1]{\textit{\scriptsize{Institut für Mathematik, Universität Potsdam, Potsdam, Germany.}}}
\affil[2]{\textit{\scriptsize{INRAE, MISTEA, Univ. Montpellier, Montpellier, France}} }

\begin{document}

\maketitle 
\begin{abstract}
    Motivated by crowdsourcing, we consider a problem where we  partially observe the correctness of the answers of $n$ experts on $d$ questions. In this paper, we assume that both the experts and the questions can be ordered, namely that the matrix $M$ containing the probability that expert $i$ answers correctly to question $j$ is bi-isotonic up to a permutation of it rows and columns. When $n=d$, this also encompasses the strongly stochastic transitive (SST) model from the tournament literature.  Here, we focus on the relevant problem of deciphering small entries of $M$ from large entries of $M$, which is key in crowdsourcing for efficient allocation of workers to questions. More precisely, we aim at recovering a (or several) level set $p$ of the matrix up to a precision $h$, namely recovering resp.~the sets of positions $(i,j)$ in $M$ such that $M_{ij}>p+h$ and $M_{i,j}<p-h$. We consider, as a loss measure, the number of misclassified entries. As our main result,  we construct an efficient polynomial-time algorithm that turns out to be minimax optimal for this classification problem. This heavily contrasts with existing literature in the SST model where, for the stronger reconstruction loss, statistical-computational gaps have been conjectured. More generally, this shades light on the nature of statistical-computational gaps for permutations models.
\end{abstract}

\section{Introduction}

Ranking problems have spurred a lot of interest both in the statistics and machine learning communities. Applications of these problems include a variety of things ranging from tournament problems~\cite{buhlmann1963pairwise}, pairwise comparisons~\cite{furnkranz2010preference}, to crowdlabeling problems~\cite{raykar2011ranking}.

In tournament problems with $n$ players, we observe results of games between players and the general objective is to gain knowledge on the ranking between the players. More formally, this amounts to have noisy partial observations from an unknown matrix $M\in [0,1]^{n\times n}$ where $M_{ij}$ corresponds to the probability that player $i$ beats player $j$. In crowdsourcing problems, $n$ experts (or workers) are faced to $d$ types of questions (or tasks).  Here, the $n\times d$ matrix $M$ encodes the fact that $M_{ij}$ is the probability that expert $i$ correctly answers  the question $j$. Based on noisy observations of the matrix $M$, the objective is also to rank the experts or/and the questions. In this manuscript, we consider both crowdsourcing and tournament data, the tournament problem being a specific case where $n=d$ and $M_{ij}=1-M_{ji}$ (skew symmetry).

Earlier models for tournament problems are parametric in nature. Among others, 
the Bradley-Luce-Terry model \cite{bradley1952rank,luce2012individual} has prompted a lot of works, even recently~\cite{hunter2004mm,gao2023uncertainty,bong2022generalized}. In particular, computationally efficient and statistically optimal parameter estimation methods have been introduced. Other parametric models such as the noisy sorting\footnote{this model is precisely defined in Section~\ref{sec:Discussion}} are also well understood~\cite{braverman2008noisy,mao2018minimax}. However, it has been observed that these simple parametric models are often unrealistic~\cite{mclaughlin1965stochastic,ballinger1997decisions} and do not tend to fit the data well.  This has lead to a recent line of literature in tournament and crowdsourcing where strong parametric assumptions are replaced by shape-constrained non-parametric assumptions  on the matrix $M$~\cite{shah2015estimation,shah2016stochastically,shah2019feeling,shah2020permutation, mao2020towards,mao2018breaking,liu2020better,flammarion2019optimal,bengs2021preference,saad2023active}. Arguably, the most popular model in this field are the strong stochastically transitive (SST)~\cite{shah2016stochastically} model for tournament and the bi-isotonic model for crowdsourcing problems. The SST model presumes that the square matrix $M$ is, up to a common permutation $\pi$ of the rows and of the columns, bi-isotonic, that is $M_{\pi^{(-1)}(i)\pi^{(-1)}(j)}\geq \max(M_{\pi^{(-1)}(i+1)\pi^{(-1)}(j)}, M_{\pi^{(-1)}(i)\pi^{(-1)}(j+1)})$.
From a modeling perspective, this corresponds to assuming that, if player $i$ is better than $j$, then it has a larger probability than player $j$ of beating a third player $k$. 
Similarly, the bi-isotonic models in crowdsourcing data, subsumes that, up to a permutation $\pi$ of the rows and $\eta$ of the columns, the matrix $M$ is bi-isotonic. 

Most works in the recent literature~\cite{shah2015estimation,shah2016stochastically,shah2019feeling,shah2020permutation, mao2020towards,mao2018breaking,liu2020better,pilliat2022optimal} have focused on estimating the matrix $M$ in Frobenius distance. As recalled in Section~\ref{sec:preliminaries}, obtaining a good estimator of $M$ mostly boils down to estimating the permutation $\pi$ (and $\eta$ for crowdsourcing problems) with respect to some $l_2$ type distance. Although optimal rates for estimation of $M$ have been pinpointed in the earlier paper of Shah et al.~\cite{shah2016stochastically}, there remains a large gap between these optimal rates and the best known performances of polynomial time algorithms. This has led to  conjecture the existence of a statistical-computational gap~\cite{mao2020towards, liu2020better}.

\subsection{Localizing large entries of $M$}

 In this manuscript, we move aside from the problem of estimating the matrix $M$ in Frobenius distance to the related problem of deciphering small entries of $M$ from large entries of $M$.  Given $M\in [0,1]^{n\times d}$, some threshold $p\in [0,1]$ and some tolerance $h\in [0,1]$, we define the classification matrix $R^*_{p,h}$ by $(R^*_{p,h})_{ij}= 1$ if $M_{ij}\geq p+h$, $(R^*_{p,h})_{ij}= 0$ if $M_{ij}\leq p+h$, and $(R^*_{p,h})_{ij}= \texttt{NA}$ otherwise. In tournament problems, the matrix $R^*_{p,h}$ encodes the set of games $(i,j)$ such that the probability $M_{ij}$ that $i$ beats $j$ either exceeds $p+h$ or are below $p-h$. In crowdsourcing problems, $R^*_{p,h}$ encodes the sets of experts/question such that the probability of obtaining the right answer is above (resp. below) $p+h$ (resp. $p-h$). For instance, in tournament problems, finding $R^*_{p,h}$ is relevant for betting purposes. In crowdlabeling applications, the knowledge of $R^*_{p,h}$ is important to assign experts/workers to given tasks/questions. Our main objective in this paper is to recover the matrix $R^*_{p,h}$ from noisy and partial observations of $M$. In the sequel, we refer to this problem as the classification problem $R^*_{p,h}$.

Given an estimator $\hat{R}_{p,h}$ of $R^*_{p,h}$ we quantify its error by 
\begin{equation}\label{eq:loss:reconstruction}
L_{0,1,\texttt{NA}}[\hat{R}_{p,h}]= \sum_{(i,j): (R^*_{p,h})_{ij}\in \{0,1\}} |R^*_{p,h}- \hat{R}_{p,h}|\enspace .
\end{equation}
This loss simply counts the number of classification error among the entries which are outside the region of tolerance. Estimating the classification matrix $R^*_{p,h}$ is, in some sense, a weaker problem than estimating the full matrix $M$ in Frobenius norms. Indeed, given an estimator $\hat{M}$ we can define the plug-in estimator $R_{p}(\hat{M})$ by $[R_{p}(\hat{M})]= 1_{\hat{M}_{ij}\geq p}$. Then, one can easily deduce that 
\begin{equation}\label{eq:upper_Frobenius}
L_{0,1,\texttt{NA}}[ R_{p}(\hat{M})]\leq \frac{1}{4h^2}\|\hat{M}- M\|_F^2 \ . 
\end{equation}
If $M$ is an SST matrix, it is well known that the optimal rate of convergence for $\|\hat{M}- M\|_F^2$ with full observations is of the order of $n$~\cite{shah2015estimation}, so that from an information-theoretical point of view, it is possible to classify the large entries, that is to estimate $R^*_{p,h}$, with a loss of the order of $n/h^2$. However, as explained above, no polynomial time-estimator of $M$ achieves the $n$ rate in Frobenius norm and, as of today, the best achievable rate is of the order of $n^{7/6}$~\cite{liu2020better,pilliat2024optimal}. The above observations raise the two following important questions:  
\begin{enumerate}
    \item[(a)] In crowdsourcing and tournaments problems, what is the optimal error for classification? In particular, for tournament problems with full observation, is this optimal rate of the order of $n/h^2$ as suggested by~\eqref{eq:upper_Frobenius}? 
    \item[(b)] Is there a polynomial time algorithm achieving this rate? 
\end{enumerate}

\subsection{Our contribution}

In this manuscript, we answer by the affirmative to both questions by characterizing the optimal rate for estimating the classification $R^*_{p,h}$ and by introducing a new computationally efficient and statistically optimal estimator. In contrast, relying on available state-of-the-art polynomial time ranking estimators such as those in~\cite{pilliat2022optimal,liu2020better} would lead to a larger loss by a multiplicative factor at $(n\vee d)^{1/6}$.

As a notable consequence of our results, we establish the absence of a computational gap for reconstruction $M$ in Frobenius distance when the matrix $M$ is restricted to take a finite number of values. This entails that the conjectured computation gap for SST matrix estimation, if it exists, only arises for multi-scale matrices $M$.

From a technical perspective, we introduce a novel procedure to estimate the permutations $\pi$ (and possibly $\eta$). The general idea is to iteratively localize the level set at height $p$ of the matrix $M$. Intuitively, at each step, we consider groups $E$ of rows that remain to be compared and ranked, and based on our partial knowledge of the ranking, we select an envelope $Q$ of columns, that is a small set $Q$, such that, for the rows $E$, the threshold $p$ is achieved on the set $Q$. Then, by looking at the noisy observations of $M$ on $E\times Q$, we are able to gain additional knowledge on the mutual ordering of the rows in $E$. While the idea of iteratively refining a partial ordering on the rows by further localizing the columns of interests is not new~\cite{liu2020better,pilliat2024optimal}, there are some important differences both in the algorithm and its see analysis. See Section~\ref{sec:Discussion}  for further discussion.

\subsection{Organization and notation}

In Section~\ref{sec:preliminaries}, we formally introduce the observation model, and we also reduce the problem of estimating $R^*_{p,h}$ to that of estimating the permutations $\pi$ and $\eta$ with respect to a suitable loss $\mathcal{L}_{p,h}$. In Section~\ref{sec:main_results}, we describe our main results, whereas Section~\ref{Sec:Algos} and Section~\ref{sec:estimation_R} are dedicated to the description of our polynomial time estimator procedure. Finally, we discuss our results and further compare them to the literature in Section~\ref{sec:Discussion}. All the proofs are postponed to the appendix.

In the sequel, we write $[n]$ for the set $\{1,\ldots n\}$. We write $\mathcal{S}_n$ for the set of permutations of $[n]$. Besides, $\mathrm{id}_{[n]}$ stands for the identity permutation. For a permutation $\pi$ on $[n]$, we say $i$ is below $j$ if $\pi(i)< \pi(j)$. Since our work is motivated by crowdsourcing problems, expert $i$ henceforth refers to the $i$-th question of $M$ and question $j$ to the $j$-th column of $M$. Given $\sigma>0$, a mean-zero random variable is said to be $\sigma^2$-subGaussian if it satisfies $\mathbb{E}[\exp(tW)]\leq\exp(t^2\sigma^2)/2$ for $t\in \mathbb{R}$. We write $\mathrm{SG}(\sigma^2)$ for the class of centered $\sigma^2$-subGaussian distributions.  Given a matrix $M$, we write $\|M\|_F$ for its Frobenius norm. For two quantities $x$ and $y$, $x\vee y$ and $x\wedge y$ respectively refer to $\max(x,y)$ and $\min(x,y)$. We write $\lfloor\cdot\rfloor$ and $\lceil\cdot \rceil$ for the lower and upper integer parts. Also, $\log_a(x)$ refers to $\log(x)/\log(a)$.
Finally, $x \lesssim y$, means that there exists a numerical constant $c$ such that $x\leq cy$.

\section{Preliminaries}\label{sec:preliminaries}
\subsection{Problem formulation}\label{Def:Model}

We consider, without loss of generality, the general crowdsourcing setting where $M\in [0,1]^{n\times d}$, the tournament setting being a specific case where $n=d$ and we have skew-symmetry. In the sequel, we write $\mathbb{C}_{\mathrm{Biso}}(\mathrm{id}_{[n]},\mathrm{id}_{[d]})\subset[0,1]^{n\times d}$ the collection of bi-isotonic matrices, that is matrices satisfying the inequality $M_{ij}\geq \max(M_{i+1j}, M_{ij+1})$. Given some unknown permutation $\pi\in \mathcal{S}_n$ and $\eta\in \mathcal{S}_d$, we assume henceforth that the matrix $M_{\pi^{-1}, \eta^{-1}}$ defined by $(M_{\pi^{-1}, \eta^{-1}})_{ij}=M_{\pi^{-1}(i) \eta^{-1}(j)}$ is bi-isotonic. For fixed $\pi$ and $\eta$, we write $\mathbb{C}_{\mathrm{Biso}}(\pi,\eta)$ for the corresponding collection of matrices. Finally, 
we write $\mathbb{C}_{\mathrm{Biso}}:= \cup_{\pi,\eta}\mathbb{C}_{\mathrm{Biso}}(\pi,\eta)$ for the collection of bi-isotonic matrices up to two permutation. 
Similarly, we define $\mathbb{C}'_{\mathrm{Biso}}(\pi)$ the set of matrices $M$ such that, up to the permutation $\pi$, $M$ is non-increasing on each row and non-decreasing on each column. Equipped with this definition, we can define the collection of strongly stochastic transitive (SST) matrices as 
$\mathbb{C}_{\mathrm{SST}}= \bigcup \mathbb{C}'_{\mathrm{Biso}}(\pi)\cap \{M: M+M^T = 1\}$.

As usual in the literature --e.g.~\cite{mao2020towards}--, we use the Poissonization trick to model the partial observations on the matrix $M$. Given some $\lambda_0>0$, which is henceforth referred as the sampling effort, and $ N=\lambda_0 nd$, we have $N' \sim Poi(N)$ observations of the form
\begin{equation}\label{eq:observation}
		Y'_t = M_{i_tj_t} + W_t'\enspace ,
\end{equation}
for $t \in \{1, \ldots, N'\}$, where the sequences $I = (i_t)_t, J = (j_t)_t$ are independent and uniformly distributed on~$[n]$ and $[d]$, and where $( W_t')_t$ is an independent $\sigma^2$-subGaussian noise for $\sigma >0$. The data in this model is therefore of the form $(N',I,J,Y')$, where $Y' = (Y'_t)_t$. In particular, the observation model~\eqref{eq:observation} allows for binary observations where  $Y'_t$ is a Bernoulli random variable with parameter $M_{i_tj_t}$ in which case we have $\sigma^2=1/4$.

With this observation scheme, each entry $(i,j)$ of $M$ is observed in a noisy way in expectation $\lambda_0$ times. If $\lambda_0$ is much smaller than $1$, $\lambda_0$ is to be interpreted as the probability of observing any given entry and we do not have any observation on most entries, making the ranking task more challenging. The assumption that the total sample size $N'$ is distributed as a Poisson random variable is questionable, but as usual for this type of problems, this can be leveraged quite easily. We mostly keep it here for the sake of simplicity.

Recall that, for $p$, $h\in [0,1]$,  we refer to $p$ as the \textit{threshold} and $h$ as the  \textit{tolerance}. In the sequel, given some observations $(N',I,J,Y')$ sampled from~\eqref{eq:observation}, our objective is to infer the classification matrix $R^*_{p,h}$.

\subsection{Permutation loss and reduction}

For estimating $R^*_{p,h}$, the main challenge is to suitably estimate the ranking, that is the permutations $\pi$ and $\eta$ such that $M\in \mathbb{C}_{\mathrm{Biso}}(\pi,\eta)$. For any estimators $\hat{\pi}$ and $\hat{\eta}$, we define the loss $ \mathcal{L}_{p,h}(\hat \pi,\hat \eta )$ as
\begin{align}
    \mathcal{L}_{p,h}(\hat \pi,\hat \eta )\coloneqq& \left|\left\{(i,j)\in [n]\times [d]:\ M_{\pi^{-1}(i)\eta^{-1}(j)}\leq p-h,\ M_{\hat\pi^{-1}(i)\hat \eta^{-1}(j)}\geq p+h\right\}\right| \nonumber\\
	&+\left|\left\{(i,j)\in [n]\times [d]:\ M_{\pi^{-1}(i)\eta^{-1}(j)}\geq p+h,\ M_{\hat\pi^{-1}(i)\hat \eta^{-1}(j)}\leq p-h\right\}\right|\enspace . \label{eq:loss}
\end{align}
Note that this loss depends on the true value of the matrix $M$. It counts the number of times that entries smaller than $p-h$ get confused with entries larger than $p+h$, if we sort $M$ by $\hat{\pi}$ and $\hat{\eta}$ instead of the oracle permutations $\pi$ and $\eta$.

 In previous works where the focus was to estimate $M$ in Frobenius norm, the overall challenge was to  estimate the permutations $\pi$ and $\eta$~\cite{mao2020towards,liu2020better} with respect to  the stronger loss $\mathcal{L}_{F}$ defined by 
\begin{equation}\label{eq:definition:LF}
\mathcal{L}_{F}(\hat \pi,\hat \eta ) := \|M_{\hat \pi^{-1}(.), \hat \eta^{-1}(.)} - M_{\pi^{-1}(.), \eta^{-1}(.)}\|_F^2\enspace , 
\end{equation}
which measures how close the matrix re-ordered by the estimated permutations is to the perfectly re-ordered matrix. Obviously, we have $\mathcal{L}_{p,h}(\hat \pi,\hat \eta )\leq (4h)^{-2}
\mathcal{L}_{F}((\pi,\eta),(\hat \pi,\hat \eta ))$, so that it suffices to bound the latter to control the former loss. However, as alluded in the introduction, we are able in the next section to craft polynomial time algorithms whose performances with respect to $\mathcal{L}_{p,h}$ are much better than what is suggested by the previous bound.

\section{Main results}\label{sec:main_results}

Our main contribution is the construction of a polynomial time  estimator of the permutation $(\hat\pi, \hat{\eta})$, that turns out to be optimal in the minimax sense with respect to loss $\mathcal{L}_{p,h}$.  In turn, this allows us to easily derive an optimal estimator classification matrix $R^*_{p,h}$.  To ease the reading, we mainly state risk bounds in this section and we postpone the definition of our procedures to the next two sections.

\subsection{Minimax lower bound}

We first state a minimax lower bound both for the permutation estimation problem with the loss $\mathcal{L}_{p,h}$ and for the classification matrix $R^*_{p,h}$, with respect to the loss $L_{0,1,\texttt{NA}}$.  For the purpose of the following lower bound, we write $\mathbb{E}_M$ for the expectation with respect to the data $(N',I,J,Y')$ sampled with a given $\lambda_0$ and $M$, and where the noise in~\eqref{eq:observation} is normally distributed with variance $\sigma^2$. 

\begin{theorem}\label{Thm:LowerBound}
There exist universal constants $c,c',c''>0$, such that the following holds for any  $\sigma>0$, $\lambda_0$, $p\in [0,1]$, and $h\in (0,\min(p,1-p))$, and $n$, $d$ such that $n\vee d\geq 2$.  If $\lambda_0 h^2\leq c\sigma^2$, then
    \begin{align*}
    \inf_{\hat{R}_{p,h}} \quad    \sup_{M\in \mathbb{C}_{\mathrm{Biso}}} \mathbb{E}_{M}\left[L_{0,1,\texttt{NA}}[\hat{R}_{p,h}]\right]\geq \Big[c' \frac{\sigma^2 }{\lambda_0h^2}(n\vee   d)\Big] \wedge (nd)\enspace  .\\
        \inf_{\hat{\pi},\hat{\eta}}   \quad  \sup_{ M\in \mathbb{C}_{\mathrm{Biso}}} \mathbb{E}_{M}\left[\mathcal{L}_{p,h}(\hat{\pi},\hat\eta)\right]\geq \Big[c'' \frac{\sigma^2 }{\lambda_0h^2}(n\vee  d)\Big] \wedge (nd)\enspace  .
    \end{align*}
\end{theorem}
The condition $\lambda_0 h^2 \leq c\sigma^2$ is really mild, because in the most relevant setting, we have $\lambda_0<1$ (less than $1$ observation in expectation per entry) and $\sigma$ of the order of a constant, as for Bernoulli observations. In the regime  where $\lambda_0< 1$, Mao et al.~\cite{mao2020towards} introduced exponential-time least-square type estimators $\hat{M}$ and $(\tilde{\pi},\tilde{\eta})$ that satisfy 
\begin{align*}
    \mathbb{E}[\|\hat{M}-M\|_F^2]\leq  \left(c\sigma^2 \log^{3/2}(nd)\frac{n\vee d}{\lambda_0}\right)\wedge nd\ ,\\
\mathbb{E}[\|M_{\tilde{\pi}^{-1},\tilde{\eta}^{-1}}-M_{\pi^{-1},\eta^{-1}}\|_F]\leq \left(c' \sigma^2 \log^{3/2}(nd)\frac{n\vee d}{\lambda_0}\right)\wedge nd\enspace , 
\end{align*}
for some universal constants $c,c'>0$. We have explained in the previous subsection how we can deduce bounds with respect to   the losses $L_{0,1,\texttt{NA}}$ and $\mathcal{L}_{p,h}$ from the above equations. This implies that Theorem~\ref{Thm:LowerBound} is tight, up to polylogarithmic terms, and characterizes the minimax risk for reconstructing the classification matrix.

\subsection{Permutation and classification matrix estimation}

In order to estimate $\pi$ and $\eta$, we introduce in the next section a polynomial time estimator $\hat{\pi}_S$, which depends on $p$, $h$ and $\sigma^2$ and a tuning parameter $\delta\in (0,1)$, which corresponds to a probability of error.

\begin{theorem}\label{Thm:ErrorBound}
    There exist two universal constants $c, c'>0$ such that the following holds for any $\delta>0$, $\sigma>0$, $\lambda_0\in (0,\log(nd)]$, $p\in[0,1]$, $h\in [0,1]$, $\pi \in \mathcal{S}_n$, $\eta \in \mathcal{S}_d$, $M\in \mathbb{C}_{\mathrm{Biso}}(\pi,\eta)$. With probability higher than $1-\delta$, the estimator $(\hat\pi_S,\hat\eta_S)$ --defined in~\eqref{eq:definition_hat_pi}-- with tuning parameter $\delta$ satisfies
    \begin{align*}
           \mathcal{L}_{p,h}(\hat\pi_S,\hat\eta_S)\leq{} c(\sigma^2\vee 1) \log^{5/2}(nd/\delta)\frac{n\vee d}{\lambda_0 h^2}\enspace ,
    \end{align*}
     If we fix $\delta = 1/(nd)$, we also have
    \begin{align}\label{eq:upper_risk}
        \mathbb{E}\left[\mathcal{L}_{p,h}(\hat\pi_S,\hat\eta_S)\right]\leq\left(c'(\sigma^2\vee 1)\log^{5/2}(nd)\frac{n\vee d}{\lambda_0 h^2}\right)\wedge (nd)\enspace .
    \end{align}
\end{theorem}
In the above theorem, we are only assuming that the sampling effort $\lambda_0\leq  \log(nd)$ so that there are in expectation less than $\log(nd)$ observations per entry. As the sparse case $\lambda_0<1$ is arguably the most relevant, this is not really restrictive. In fact, we would need to use a variant of our procedure to better handle the case of very large sampling effort ($\lambda_0\geq \log(nd)$); we restricted ourselves to the sparser case for the sake of conciseness. 

Comparing~\eqref{eq:upper_risk} with Theorem~\ref{Thm:LowerBound}, we observe that the estimator $(\hat{\pi},\hat{\eta})$ is, up to a polylog factor, minimax optimal for any $n$, $d$, $\lambda_0\leq \log(nd)$ and for any noise level $\sigma$ bounded away from zero. 
Also, Theorem~\ref{Thm:ErrorBound} handles the case of sparse observations, the convergence rates being optimal and non-trivial for $\lambda_0$ as small $\mathrm{Polylog}(nd)/[h^2 (n\wedge d)]$, which corresponds to the challenging situation where  there are a logarithmic  number of observations on each row (resp. column) if $n\geq d$ (resp. $n\leq d$).

We have stated the previous theorem for a single threshold/tolerance $(p,h)$, but in fact, we can construct our estimator for a collections $(\underline{p},\underline{h})=(p_1,h_1),\ldots, (p_m,h_m)$.
\begin{corollary}\label{cor:error_bound:permutation:multiple}
	There exists an universal constant $c>0$ such that the following holds for any $\delta>0$, $\sigma>0$, $\lambda_0\in (0,\log(nd)]$, $(\underline{p},\underline{h})=(p_1,h_1),\ldots, (p_m,h_m)$,  $M\in \mathbb{C}_{\mathrm{Biso}}$. With probability higher than $1-\delta$, the estimator $(\hat{\pi}_S,\hat{\eta}_{S})$ defined in~\eqref{eq:definition_hat_pi} with $(\underline{p},\underline{h})=(p_1,h_1),\ldots, (p_m,h_m)$ and $\delta/m$
	satisfies, simultaneously for all $l=1,\ldots, m$.
		\begin{align*}
			   \mathcal{L}_{p_l,h_l}(\hat\pi_{S},\hat\eta_{S})\leq &\left[c(\sigma^2\vee 1) \log^{5/2}\left(\frac{nd m}{\delta}\right)\frac{n\vee d}{\lambda_0 h_l^2}\enspace \right] \wedge (nd) \enspace . 
		\end{align*}
	\end{corollary}
	In comparison to the single choice of threshold, tolerance $(p,h)$ (Theorem~\ref{Thm:ErrorBound}), we only pay a mild logarithmic price with respect to the number $m$ of thresholds. For $(p,h)$ and $(p',h')$ such that $[p-h,p+h]\subset [p'-h',p+h']$ and for any $(\pi',\eta')$ we have, by definition, $\mathcal{L}_{p',h'}(\pi',\eta')\leq \mathcal{L}_{p,h}(\pi',\eta')$. This allows us to get a simultaneous control over all losses $\mathcal{L}_{p,h}$ as explained in the following remark.

	\begin{remark}\label{remark:multiple:ph}
 By building a regular grid of thresholds $p$ of width $1/(nd)$, and for each threshold, a dyadic grid $(1/2, 1/4, 1/8,\ldots, 1/2^{\lceil \log_2(nd)\rceil})$ of tolerance, we deduce from the above corollary that $(\hat{\pi}_S,\hat{\eta}_{S})$ defined in~\eqref{eq:definition_hat_pi} with this choice of $(\underline{p},\underline{h})$ and $\delta/(nd\log_2(nd))$, satisfies,  with probability higher than $1-\delta$,
	\begin{align*}
			   \mathcal{L}_{p,h}(\hat\pi_{S},\hat\eta_{S})\leq &\left[c'(\sigma^2\vee 1) \log^{5/2}\left(\frac{nd }{\delta}\right)\frac{n\vee d}{\lambda_0 h^2}\enspace \right] \wedge (nd) \enspace ,
	\end{align*}
	simultaneously for all $(p,h)\in [0,1]$, where $c'>0$ is a universal constant. 
	\end{remark}

As a consequence, the above estimator of the permutations turns out to be simultaneously optimal over all thresholds and tolerances. Finally, we describe in Section~\ref{sec:estimation_R} how to deduce a polynomial time estimator of the level set  $R^*_{p,h}$.

\begin{theorem}\label{Thm:ErrorBound_classificatoin}
    There exists a universal constant $c>0$ such that the following holds for any $\sigma>0$, $\lambda_0\in [0,\log(nd)]$, $p\in[0,1]$, $h\in [0,1]$, and $M\in \mathbb{C}_{\mathrm{Biso}}$. The estimator $\hat{R}_{p,h}$  defined in~\eqref{eq:estimator:R_ph:Fina} satisfies
    \begin{align*}
\mathbb{E}\left[L_{0,1,\texttt{NA}}[ \hat{R}_{p,h}]\right]\leq{}c(\sigma^2\vee 1) \log^{7/2}(nd)\frac{n\vee d}{\lambda_0 h^2}\enspace ,
    \end{align*}
\end{theorem}

In light of Theorem~\ref{Thm:LowerBound} the polynomial time classification estimator $\hat{R}_{p,h}$ is, up to polylogarithmic terms, minimax optimal for the bi-isotonic  classification problem.

    \section{Description of the ranking algorithm}\label{Sec:Algos}

The procedure that we present in this section is aiming at reconstructing a specific level set of the matrix. A fundamental step in it is to estimate a ranking of the experts/questions. This is done in Algorithm $\ALGO${}, which is our main algorithmic contribution. Then based on the permutations outputted by $\ALGO${}, we re-order the matrix and perform inference of the level set of the reordered matrix. In what follows, we first present the intuition behind $\ALGO${}, then describe in details the procedure itself.

\subsection{Intuition behind $\ALGO${}}\label{sec:intution}

$\ALGO${} outputs an estimator of the true permutation of the experts. The algorithm itself is involved and makes recursive calls to several routines. We describe here the intuition behind this algorithm and behind the optimal error bound $\sigma^2 (n\vee d)/ (h^2\lambda_0)$. 
For the purpose of this subsection, we assume that the matrix $M$ only takes two values, namely $\{p-h,p+h\}$. Also, we focus on the specific problem of estimating the permutation $\pi$ of the set $[n]$ of experts. The intuitions presented here will however remain valid in the general setting. From a broad perspective, Algorithm $\ALGO${} recursively refines  an ordered partition of the experts, and transforms it in the end into a permutation. This refinement is done by a subroutine which, given a set $E\subset [n]$ of the partition, splits it into an ordered partition. 

\begin{figure}[h]
	\centering
	\subfloat[Some set of questions $E\subset \lbrack n\rbrack$ and the corresponding set of questions of interest $Q^*(E)$, defined in Equation~\eqref{eq:definition:Q*}.\label{Fig:Q^*}]{
		\includegraphics{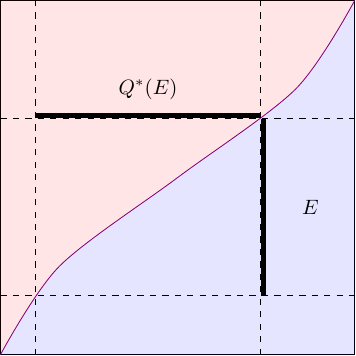}}
 \hspace{1cm}
	\subfloat[After several halving steps, we end up with an ordered partition $(E_i)_{i=1,2,3,4}$ of $\lbrack n \rbrack$ such that the corresponding questions of interest $(Q^*(E_i))_{i=1,2,3,4}$ are small.\label{Fig:Partition}]{
	\includegraphics{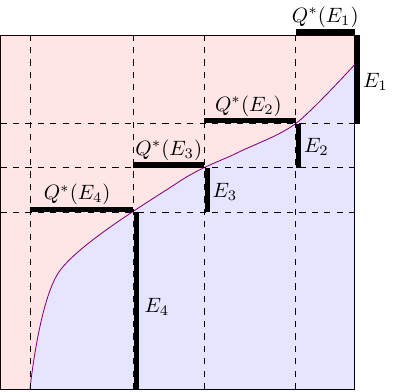}} 
\caption{Illustration of two bi-isotonic matrices $M\in\mathbb{C_{\mathrm{Biso}}}(\mathrm{id}_{\lbrack n\rbrack},\mathrm{id}_{\lbrack d\rbrack})$ so that $\pi=\mathrm{id}_{\lbrack n\rbrack}$ and $\eta= \mathrm{id}_{\lbrack d\rbrack}$. These matrices take two values $p+h$ (red) and $p-h$ (blue).}
\end{figure}

\paragraph{Active set $Q^*(E)$ of questions for a set $E$ of experts.}
For partitioning the set $E$ of experts, one needs to compare the corresponding rows of the matrix $M$ on suitable questions. In our case, the only relevant subset of questions to compare experts in $E$ is
\begin{align}\label{eq:definition:Q*}
    Q^*(E)\coloneqq \{j\in[d]:\ \max_{i\in E}M_{ij}\geq p+h,\ \min_{i\in E}M_{ij}\leq p-h\}\enspace ,
\end{align}
i.e.~the questions where the experts within $E$ differ around the threshold $p$ - see Figure~\ref{Fig:Q^*}. In the absence of noise, one would like to split $E$, using $Q^*(E)$, into two parts, one containing the best half of the experts, and the other containing the worst half. Recursively adding this bisection to the ordered partition, one would ultimately end up with the true permutation $\pi$. Note however that since we observe noisy versions of $M$, we will only be able to estimate - albeit imperfectly - $Q^*(E)$ and subsequently split  $E$ if both $E$ and $Q^*(E)$ are large enough compared to the size of $h$, namely if
\begin{equation}\label{eq:E_Q*E_condition}
|Q^*(E)| \gtrsim  \frac{\sigma^2}{h^2\lambda_0}\enspace .
\end{equation}
Indeed, if $|Q^*(E)|\lesssim \frac{\sigma^2}{h^2\lambda_0}$, then it is impossible to compare any two experts $i$ and $i'$ in $E$, because corresponding rows of $M$ differ by $2h$ on at most  $|Q^*(E)|$ entries. Hence, the signal differs by $2h$ on at most $\sigma^2/h^2$ observations, which makes it impossible to compare $i$ and $i'$ when the noise level is $\sigma$.

\paragraph{Intuition behind the rate $\sigma^2 (n\vee d)/(\lambda_0 h^2)$.}
Consider an idealized situation, where,  for each $E$, an oracle gives us $Q^*(E)$ and, if $Q^*(E)$ satisfies~\eqref{eq:definition:Q*}, the oracle also provides us a perfect bisection of $E$ into two sets $O$ and $I$
such that, for any $i\in O$ and $i'\in I$, we have $\pi(i)< \pi(i')$. 
This idealized procedure would end up in providing an ordered partition of the experts $(E_1,\ldots, E_H)$ such that 
\begin{itemize}
    \item the ordered sets  $E_l$ and $E_{l+1}$ are such that, for $i\in E_l$ and $i'\in E_{l+1}$, we have $\pi(i)< \pi(i')$.
    \item each set $E_l$ has the property that either $|E_l|=1$ or $|Q^*(E_l)| \lesssim \sigma^2/( h^{2}\lambda_0)$.
\end{itemize}
See Figure~\ref{Fig:Partition}. The estimated permutation $\hat{\pi}$ is then taken as being any partition which is consistent with the ordered partition. Note that for a such permutation $\pi'$, the only entries on which it might induce an error with respect to $\mathcal{L}_{p,h}(\hat{\pi}, \eta)$ are those contained in the union of $E_i \times Q^*(E_i)$. Hence, the error $\mathcal{L}_{p,h}(\hat{\pi}, \eta)$ is bounded by
$$\sum_i |E_i| |Q^*(E_i)| \lesssim \sigma^2 \frac{n}{h^2\lambda_0}\enspace .$$
By symmetry, we would get a bound of the order of the order $\sigma^2 \frac{d}{h^2\lambda_0}$ for $\eta$.  Note that this matches the minimax error bound, see Theorem~\ref{Thm:LowerBound}, as well as the error bound of our procedure, see Corollary~\ref{cor:error_bound:permutation:multiple}.

The above description streamlines the structure of Algorithm $\ALGO${}, as well as its analysis. However, since there is noise, the algorithm as well as its analysis are much more complicated, in the three following directions.
\begin{itemize}
    \item We need to estimate $Q^*(E)$ based on noisy data. This will lead to errors on the set of questions we use for comparing experts in $E$.
    \item Even if we were equipped with oracle knowledge of $Q^*(E)$,  we would need to select relevant subsets of  $Q^*(E)$ in order to compare a given pair of experts in $E$. Indeed, two given experts in $E$ may not significantly differ on all the questions $Q^*(E)$, but just on a small subset of it.
     \item Even if we were equipped with an oracle knowledge on which questions are most useful for comparing two given experts in $E$: we would still end up performing a noisy comparison between these two experts. We will therefore not be able to perfectly divide $E$ in two ordered sets, and need to take into account these mistakes in the final partition.
\end{itemize}
Dealing with each of these problems presents a challenge of its own, which we describe in more details in the full description of the algorithm below. We emphasize in particular the second challenge in Subsection~\ref{ss:choiceQp}, which is the one that demanded the most innovative algorithmic and analytical innovations, and is a main highlight of this paper.

\subsection{Definition of the estimators and preliminaries}\label{sec:definition}

In this section, we introduce our new ranking procedure  \ALGO{}  and we describe the estimators $\hat{\pi}_S$  and $\hat{\eta}_S$. Let
\begin{equation}
k_* = 3\lceil \log_2(n\vee d))\rceil \ . 
\end{equation}
The procedure \ALGO{}
takes as inputs as sequence $(p_1,h_1),\ldots, (p_m,h_m)$ of thresholds and tolerance, some quantity $\delta\in (0,1)$ which corresponds to some probability of error, and $k_*$ noisy observations of the matrix $M$. For that purpose, we use a simple subsampling strategy to build these $k_*$ matrices from the sample.

\paragraph{Subsampling and definition of $\hat{\pi}_S$ and $\hat{\eta}_S$.} Recall that we have at our disposal  the data $(N',I,J,Y')$  which stands for a  noisy sample of entries of the matrix $M$  (see Section~\ref{Def:Model}). Using the Poissonization trick, we easily define $k_*$ estimators of $M$ as follows. For $t=1,\dots,N'$, we first sample independently a random variable $U_t$ uniformly in $[k_*]$.  Then, for each $i\in [n]$, $j\in [d]$, $k\in [k_*]$, we define
\begin{equation}\label{eq:definition_N_{ij}_k}
 N_{ij}^{(k)}\coloneqq|\{t\leq N':\ I_t=i,\ J_t=j,\ U_t=k\}|\ ,
\end{equation}
as the number of observation in the $k$-th sample that fall within $(i,j)$. Then, the observation matrix $Y^{(k)}\in \mathbb{R}^{n\times d}$ is simply defined, for any $(i,j)$,  by 
    \begin{align}\label{eq:definition_Y_ijk}
        Y_{ij}^{(k)}&\coloneqq 
            \frac{1}{N_{ij}^{(k)}}\sum_{t=1}^{N'}\mathbbm{1}\{I_t=i,\ J_t=j,\ U_t=k\}\cdot Y'_{t}\enspace , 
    \end{align}
with the convention $0/0=0$. By the Poissonization trick, the matrices $Y^{(k)}$ are i.i.d. and the $N_{ij}^{(k)}$ are distributed as independent Poisson random variables with parameters $\lambda_0^- = \frac{\lambda_0}{3\lceil \log_2(n\vee d)\rceil}$. Besides, if $N_{ij}^{(k)}=0$, then $Y_{ij}^{(k)}=0$. Conditionally to $N_{ij}^{(k)}>0$, we have 
    \begin{align}\label{Eq:Model}
        Y_{ij}^{(k)}=
 M_{ij}+W_{ij}^{(k)}\enspace , 
 \end{align}
where the $ W_{ij}^{(k)}$ are i.i.d.~in $k$, independent in $i,j$, and $\sigma^2$-subGaussian. The probability that $N_{ij}^{(k)}>0$ is  equal to $\lambda_1= 1-e^{-\lambda_0^{-}}$.

Then, the ranking estimator $\hat{\pi}_{S}:=\hat{\pi}_{S}(N',I,J,Y';\delta)$ is defined by
\begin{align}\label{eq:definition_hat_pi}
\hat{\pi}_{S}(N',I,J,Y';\delta) &:= \ALGO \big[(Y^{(1)}, \ldots, Y^{(k_*)});p,h,\lambda_0,\delta\big]\enspace . 
\end{align}
To order the columns of $M$, we simply apply the same procedure to the noisy  observation of $M^{T}$. This leads us to
$\hat{\eta}_{S}(N',I,J,Y') := \ALGO\big[(Y^{(1)T}, \ldots, Y^{(k_*)T});p,h,\lambda_0,\delta\big]$.

\paragraph{Preliminaries.} Given the parameters $p$, $h$, $\sigma^2$, and $\delta$, we define the quantities $\rho$ and $\gamma$.
\begin{align}\label{eq:def:rho}
\rho&:=(1\vee \sigma)e\sqrt{8\log\left(24 nd(n\vee d)^{1/2}\lceil\log_2(n\vee d)\rceil/\delta\right)}\ ; \\ \gamma &:=2\log\left(24 nd(n\vee d)^{1/2}\lceil \log_2(n\vee d)\rceil/\delta\right)/{\lambda_1e^2}\enspace . \label{eq:def:gamma}
\end{align}
where $e=\exp(1)$. Here, we have $\rho\asymp (1\vee \sigma)\log^{1/2}(nd/\delta)$ and $\gamma\asymp \log(nd/\delta)/\lambda_1$. In the sequel, we say that $\mathcal{E}=(E_1,\ldots, E_r)$ form an \emph{ordered partition} of $[n]$ if the collection $(E_1, \ldots, E_r)$ of  subsets is a partition of the set $[n]$ of rows. We say that a permutation $\pi$ is compatible with $\mathcal{E}$, if for any $i\in E_r$ and $j\in E_s$ with $r<s$, we have $\pi(i) < \pi(j)$. 
Given an ordered partition $\mathcal{E}$,  $\texttt{Permutation}(\mathcal{E})$ stands for any permutation $\pi$ (possibly random) that is compatible with $\mathcal{E}$. Note that constructing such $\pi$ from a given $\mathcal{E}$ can be done in linear time.

\subsection{The hierarchical sorting tree}\label{Sec:Tree}

From a broad perspective, the general structure of $\ALGO{}$ (in Algorithm~\ref{Algo:Tree}) is that of iteratively building a hierarchical sorting tree. We start with the trivial ordered partition $\mathcal{E}=([n])$. In the first step, we build in Line~\ref{Line:TreeTrisect} a Trisection  of $[n]$ into three sets $(O,P,I)$ where $O\subset[n]$ (resp. $I$) is, with high-probability, made of rows $i$, such $\pi(i)$ is below (resp. above) the median of $[n]$, that is  $\pi(i)\leq n/2$ (resp. $\pi(i)\geq n/2$) and $P$ contains the remaining rows for which we cannot draw any significant decision. This leads us to $\mathcal{E}=(O,P,I)$. Then, at the later steps, we recursively construct the ordered partition. Together with $\mathcal{E}$, we maintain a vector $v\in \{0,1\}^{|\mathcal{E}|}$ where $v_k=0$ encodes that the set $E_k$ is not to be refined anymore. This is either the case because $E_k$ arises from a set $P$ of indecisive rows in previous subsection, or because $|E_k|$ is small compared $\log(nd)/(h^2\lambda_1)$ 
--see Line~\ref{Line:TreeTrisect2} in Algorithm~\ref{Algo:Tree} or the explanation below~\eqref{eq:E_Q*E_condition}-- or if the sets $Q$ of questions associated to $E$ are too small for $(p_l,h_l)$ --see Line~\ref{line_update_u} and the explanations below~\eqref{eq:E_Q*E_condition}. See Figure~\ref{Fig:Tree} in Appendix~\ref{Sec:AnalysisTree} for a visual representation of the hierarchical sorting tree. 
The algorithm stop after less than $\lceil \log_2(n)\rceil$ iterations when none of the sets are to be refined. At the end (Line~\ref{line:final}), we simply compute any permutation $\pi$ which is compatible with the final ordered partition $\mathcal{E}$ (see Section~\ref{sec:proof:analysis:permutation} for a formal definition). Using hierarchical sorting trees has already been proposed in~\cite{liu2020better,pilliat2022optimal} and, more generally, recursive approaches for ranking problems have already been put forward in~\cite{mao2020towards}.

\begin{algorithm}[ht]
	\caption{\ALGO{}\label{Algo:Tree}}
	\begin{algorithmic}[1]
		\Require samples $\left(Y^{(k)}\right)_{k=1,\dots 3\lceil \log_2(n)\rceil}$, $p=(p_1,\dots,p_m)$, $h=(h_1,\dots,h_m)$, $\lambda_0$, $\delta$.
		\Ensure permutation $\pi$ of $[n]$ and directed graph $G$
		\State initialize $G\gets \mathbf{0}_{n\times n}$, $k\gets 0$, $\mathcal{E}\gets ([n])$, $v\gets (\mathbbm{1}\{\lambda_1n> 4\rho^2/h^2\})$
		\While{$v\neq \mathbf 0$}
		\State $k\gets k+1$, $\tilde{\mathcal{E}} \gets  ()$, $\tilde v \gets ()$ 
		\For{$t$ from $1$ to $|\mathcal{E}|$}
		\State$E\gets t^{\text{th}}$ set in $\mathcal{E}$
		\If{$v_t=0$} \Comment{means that the algorithm decided in an earlier step not to refine $E$}
		\State $\tilde{\mathcal{E}} \gets (\tilde{\mathcal{E}}, E)$, $\tilde v\gets (\tilde v,0)$ 
		\Else
        \State $u\gets 0$
        \For{$l$ from $1$ to $m$}\label{line:beginning_comparison}
		\State $Q\gets$ \texttt{Envelope}($t,\mathcal{E},v,Y^{(3k-2)},p_l,h_l,\lambda_0,\delta$) \Comment{estimates questions of interest w.r.t.  $E$}\label{line:envelope}
		\If{$\lambda_1|Q|> 4\rho^2/h^2$} \Comment{$Q$ large enough, further refinement of $E$ possible}
		\State $u\gets 1$
        \State \texttt{ScanAndUpdate}($Y^{(3k-1)},Y^{(3k)},E,Q, G,\lambda_0,\delta$) \label{line:scanandupdate}
        \Comment{update of $G$}
		\EndIf
		\EndFor \label{line:end_comparison}
        \If{$u=0$}\Comment{all envelopes are small, no further improvement needed}\label{line_update_u}
		\State $\tilde{\mathcal E}\gets (\tilde{\mathcal E},E) $, $\tilde v\gets (\tilde v, 0)$ 
        \Else
        \State
        $(O,P,I)\gets$\texttt{GraphTrisect}$(E,G)$ \label{Line:TreeTrisect}
        \State $\tilde {\mathcal{E}} \gets(\tilde{\mathcal{E}}, O,P,I)$,  $ \tilde v\gets (\tilde v, \mathbbm{1}\{\lambda_1|O|> 4\rho^2/h^2\},0,\mathbbm{1}\{\lambda_1|I|> 4\rho^2/h^2)\}$ \Comment{no further refinement of $P$ or $I$ if those are small enough}\label{Line:TreeTrisect2}
        \EndIf
        \EndIf
        \EndFor
        \State $\mathcal{E}\gets \tilde{\mathcal{E}}$, $v\gets \tilde v$
		\EndWhile
        \State Return (\texttt{Permutation}($\mathcal{E}$))\label{line:final}
	\end{algorithmic}
\end{algorithm}

The procedure \ALGO{} mostly differs from those earlier works~\cite{liu2020better,mao2020towards,mao2018minimax} in three ways. First, aside from the hierarchical sorting tree, \ALGO{} records a comparison graph $G\in \{0,1\}^{n\times n}$.  The procedure initializes with $G=\mathbf{0}_{n\times n}$. Within the subroutine \texttt{ScanAndUpdate}, we update the comparison $G_{ii'}$ to $1$, if we find statistical evidence that $\pi(i)< \pi(i')$. In the main algorithm, the steps in Lines~\ref{line:beginning_comparison}--\ref{line:end_comparison} amounts to find as many statistically significant comparisons as possible within $E$. Then, in Line~\ref{Line:TreeTrisect}, we trisect $E$ into $(O,P,I)$ using the graph $G$ --see Algorithm~\ref{Algo:Trisect}. If, according to $G$, the rank of $i$ is below (resp. above) at least half of the experts $i'$ in $E$ --and is therefore below (resp. above) the median of $E$-- we put $i$  in $O$ (resp. $I$). Then remaining experts are put in the indecisive set $P$. Hence, the statistical efficiency of the procedure mostly depends on the way we construct $G$.
  For that, there are two main ingredients. First, the subroutine \texttt{Envelope} (Line~\ref{line:envelope}) that selects a set $Q$ of questions which, with high probability, contains $Q_{p_l,h_l}^*(E)$, which is defined in~\eqref{eq:definition:Q*}. As explained in Subsection~\ref{sec:intution}, this set is instrumental towards comparing experts in $E$. Second, the subroutine \texttt{ScanAndUpdate} compares experts based on partial row sums on subsets $Q$. This subroutine differs from the idealized examples in Section~\ref{sec:intution} in that, it does not simply compare the row sums restricted to $Q$ on the set $E$, but also selects many relevant subsets of $Q$. These two subroutines are described in the next subsections.

\begin{algorithm}
    \caption{\texttt{GraphTrisect}}\label{Algo:Trisect}
    \begin{algorithmic}
        \Require $E\subseteq [n]$, graph $G\in\{0,1\}^{n\times n}$
        \Ensure trisection $O,P,I$ of $E$
        \State $O\gets \{i\in E:\ \sum_{i'\in E}G_{ii'}>|E|/2\}$ \Comment{$\pi(i)<\pi(i')$ detected for the majority of $i'\in E$}
		\State $I\gets \{i\in E:\ \sum_{i'\in E}G_{i'i}>|E|/2\}$ \Comment{$\pi(i)>\pi(i')$ detected for the majority of $i'\in E$}
		\State $P\gets E\setminus (O\cup I)$ 
    \end{algorithmic}
\end{algorithm}

In the proof of Theorem~\ref{Thm:ErrorBound} --more specifically in that of Theorem~\ref{Thm:ErrorBoundRows}-- we will show that, on an event of probability higher than $1-\delta/2$, at all the steps of the algorithms, the envelope set $Q$ will contain the oracle envelopes $Q_{p_l,h_l}^*(E)$ and that a trisection $(O,P,I)$ computed by \texttt{GraphTrisect} will be such that, for any $i$ in $O$ and $i'\in I$ we have $\pi(i)< \pi(i')$. However, it does not preclude some expert $i'$ in $P$ to have a true rank  below some of $O$ or above some of $I$. In summary, the estimated permutation $\hat{\pi}_S$ makes errors because in the final ordered permutation $\mathcal{E}=(E_1,\ldots, E_r)$, (i) within any set $E_s$, the $E_s$'s are ordered arbitrarily and (ii) because, for some $(i,i')$ with $i\in E_s$ and $i'\in E_{s'}$ with $s<s'$ we have $\pi(i')< \pi(i)$. This second problem arises because of the sets $P$ in the trisection steps. The cornerstone of the proof amounts to proving the loss of arising from these errors is small.

\subsection{The \texttt{ScanAndUpdate}{}  routine}\label{Sec:Trisection}

As indicated above, the purpose of the \texttt{ScanAndUpdate} routine is to update the comparison graph $G$ on the experts in $E$. \texttt{ScanAndUpdate} takes as a parameter a subset $Q$ of questions --that has been computed using the \texttt{Envelope} routine-- and builds many possibly relevant subsets $Q'$ of questions (to be discussed below). Given such a  subset $Q'$ of questions, and some matrix $\tilde{Y}^{(b)}$ sampled according to~\eqref{Eq:Model}, it calls the routine \texttt{UpdateGraph} (Algorithm~\ref{Algo:Graph}) that  compares in Line~\ref{Line:GraphThreshold}, for any $i$, $i'\in E$, the partial row sums of $\tilde{Y}^{(b)}$ restricted to the columns in $Q'$. Since $\tilde{Y}^{(b)}$ is close to $M$ and since $M\in \mathbb{C}_{\mathrm{biso}}$, when the partial row sum of $i$ is significantly larger than that of $i'$, we can conclude that $\pi(i)< \pi(i')$ and we 
update  the comparison graph to $G_{ii'}=1$. The thresholds $2\rho\sqrt{\lambda_1/Q'}$ in Line~\ref{Line:GraphThreshold} has been chosen in such a way that we draw any false conclusion with extremely small probability. \texttt{ScanAndUpdate} performs all possible $|E|(|E|-1)$ comparisons. Note that the partial row sums are computed only if $Q'$ is large enough (Condition in Line \ref{Line:GraphSize}), otherwise we may not have enough observations on $\tilde{Y}^{(b)}$ to draw statistically significant conclusions.

Let us now further discuss the construction of all the sets $Q'$ in Lines~\ref{Line:TrisectInBetween},~\ref{Line:TrisectSmaller}, and \ref{Line:TrisectLarger}. We do not simply use $Q$ to compare experts in $E$, as $Q$ - which estimates $Q^*(E)$ - aims at selecting any questions on which two experts in $E$ disagree. But in order to compare a given expert $i\in E$ with other experts in $E$, it is important to refine $Q$ and adapt it to questions that are most relevant to $i$. The choice of sets $Q'$ reflect this - see Subsection~\ref{ss:choiceQp} for a more in depth explanation. 

For each $j\in Q$, the algorithm uses partial columns sums of $\tilde Y^{(a)}$ with respect to $E$ to construct sets $Q'\subseteq Q$ that contain $j'\in Q$ for which we can detect based on data that either $\eta(j')<\eta(j)$ or $\eta(j')>\eta(j)$; or for which such a detection based on these partial columns sums is not possible.  In our construction, we restrict the deviation of the partial columns sums by different choices of a parameter $c$. This gives us an implicit control of the size of the $Q'$.

\begin{algorithm}
	\caption{\texttt{ScanAndUpdate}\label{algo:scanupdate}}
	\begin{algorithmic}[1]
		\Require $E\subseteq [n]$, $Q\subseteq [d]$, $\tilde Y^{(a)}, \tilde Y^{(b)}\in \mathbb{R}^{n\times d}$, graph $G\in\{0,1\}^{n\times d}$, $\lambda_0$, $\delta$
		\Ensure Update of $G$
		\For{$j\in Q$}
		\State $Q'\gets\{j'\in Q: \ |\frac{1}{|E|}\sum_{i\in E}\tilde Y^{(a)}_{ij}-\tilde Y^{(a)}_{ij'}|\leq 2\rho\sqrt{\lambda_1/|E|}\}$\label{Line:TrisectInBetween} \Comment{questions $j'$ with $\eta(j')$ close to $\eta(j)$}
		\State \texttt{UpdateGraph}($G$, $E$, $Q'$, $\tilde Y^{(b)},\lambda_0,\delta$)
		\For{$c=2,3,\dots, \lceil\sqrt{\lambda_1|E|}/(2\rho)+1\rceil$} \Comment{try different margins}
		\State $Q'\gets\{j'\in Q:\ 2\rho\sqrt{\lambda_1/|E|}\leq \frac{1}{|E|}\sum_{i\in E}\Big[\tilde Y_{ij'}^{(a)}-\tilde Y_{ij}^{(a)}\Big]\leq 2c\rho\sqrt{\lambda_1/|E|}\}$\Comment{questions $j'$ with $\eta(j')<\eta(j)$ within the given margin} \label{Line:TrisectSmaller}
		\State \texttt{UpdateGraph}($G$, $E$, $Q'$, $\tilde Y^{(b)},\lambda_0,\delta$)
		\State $Q'\gets\{j'\in Q:\ 2\rho\sqrt{\lambda_1/|E|}\leq \frac{1}{|E|}\sum_{i\in E}\Big[\tilde Y_{ij}^{(a)}-\tilde Y_{ij'}^{(a)}\Big]\leq 2c\rho\sqrt{\lambda_1/|E|}\}$\Comment{questions $j'$ with $\eta(j')>\eta(j)$ within the given margin}\label{Line:TrisectLarger}
		\State\texttt{UpdateGraph}($G$, $E$, $Q'$, $\tilde Y^{(b)}$, $\lambda_0$, $\delta$)
		\EndFor
		\EndFor
	\end{algorithmic}
\end{algorithm}

\begin{algorithm}[ht]
		\caption{\texttt{UpdateGraph}}\label{Algo:Graph}
	\begin{algorithmic}[1]
		\Require graph $G\in\{0,1\}^{n\times n}$, $E\subseteq [n]$, $Q'\subseteq [d]$, $\tilde Y^{(b)}\in \mathbb{R}^{n\times d}$, $\lambda_0$, $\delta$
		\Ensure update of $G$
				\If{$|Q'|>\gamma$}\label{Line:GraphSize} \Comment{required for detection based on partial row sums}
		\For{$i,i'\in E$} 
		\If{$\frac{1}{|Q'|}\sum_{j\in Q'}\Big[\tilde Y^{(b)}_{ij}-\tilde Y^{(b)}_{i'j} \Big]>2\rho\sqrt{\lambda_1/|Q'|}$} \label{Line:GraphThreshold}\Comment{criterion for detecting $\pi(i)<\pi(i')$}
		\State $G_{ii'}\gets 1$
		\EndIf
		\EndFor
		\EndIf
	\end{algorithmic}
\end{algorithm}

 \subsection{Selecting relevant questions with \texttt{Envelope}}\label{Sec:Questions}

We emphasized in Subsection~\ref{sec:intution} the importance of reducing the set of questions from $[d]$ to $Q_{p,h}^*(E)$ defined in~\eqref{eq:definition:Q*}. 
The purpose of $\texttt{Envelope}$ is to build a set $Q$, whose size is as small as possible, and that contains $Q_{p,h}^{*}(E)$ with high probability. We can decompose $Q_{p,h}^*(E)= \overline{Q}_{p,h}^*(E)\cap \underline{Q}_{p,h}^*(E)$. 
\[
	\underline{Q}_{p,h}^*(E)=\{j\in[d]:\ \max_{i\in E}M_{ij}\geq p+h\}\ , \quad \quad \overline{Q}_{p,h}^*(E)=\{j\in[d]:\ \min_{i\in E}M_{ij}\leq p-h\}\enspace . 
\]
In Algorithm~\ref{Algo:Envelope}, the goal is to build a superset $\overline{Q}$ of $\overline{Q}_{p,h}^*(E_s)$ and a superset $\underline{Q}$ of $\underline{Q}_{p,h}^*(E_s)$. We focus on the latter. In Line~\ref{line:def:s}, we define $\underline{s}$ as largest index $t<s$ such that $v_t=1$. The reason why we take $\underline{s}$ and note simply $s-1$ is that, with high probability\footnote{see Appendix~\ref{Sec:AnalysisQuestions} for precise statements and proofs}, (i) the collection $|E_{\underline{s}}|$ is large enough  so that partial column averages on $E_{\underline{s}}$ are statistically significant and (ii) it holds that all $i'$ in $E_{\underline{s}}$ are below those in $E_{s}$ --note that this is not necessarily the case for $E_{s-1}$. Under (ii), for all $j\in [d]$, $\max_{i'\in E_{\underline{s}}}M_{ij} \leq \min_{i\in E_S}M_{ij}$ so that $\underline{Q}_{p,h}^*(E)\subset \{j\in[d]:\ \frac{1}{|E_{\underline{s}}|}\sum_{i\in E_{\underline{s}}}M_{ij}\leq p-h\}$. In Line~\ref{eq:definition:underlineQ}, we build $\underline{Q}$ which, with high probability,  is a superset of the latter by accounting for the noise of the observations. Finally, the routine \texttt{Envelope} returns the intersection of $\underline{Q}$ and $\overline{Q}$.

\begin{algorithm}[H]
\caption{\texttt{Envelope}}\label{Algo:Envelope}
\begin{algorithmic}[1]
	\Require{index $s\in[r]$, $\mathcal{E}=(E_1,E_2,\dots,E_r)$, $v\in\{0,1\}^r$, $\tilde Y\in\mathbb{R}^{n\times d}$, $p$, $h$, $\lambda_0$, $\delta$}
	\Ensure{$Q\subseteq [d]$}
	\If{$s=\min\{t=1,2,\dots, r:\ v_t=1\}$} 
	\State $\underline Q \gets [d]$ \Comment{use all questions as ``left envelope''}
	\Else 
	\State \label{Line:EnvelopeUnderS}$\underline s \gets \max\{t<s:\ v_t=1\}$  \label{line:def:s}
	\State\label{eq:definition:underlineQ} $\underline Q= \left\{j\in[d]:\ \frac{1}{|E_{\underline{s}}|}\sum_{i\in E_{\underline{s}}}\tilde Y_{ij}\geq \lambda_1(p+h)-\rho\sqrt{\lambda_1/|E_{\underline s}|}\right\}$ \Comment{use partial column sums on $E_{\underline s}$ to detect questions uniformly $\geq p+h$ on $E_{\underline s}$, detected questions $j$ as ``left envelope''}
	\EndIf
	\If{$s=\max\{t=1,2,\dots, r:\ v_t=1\}$}
	\State $\overline Q \gets [d]$ \Comment{use all questions as ``right envelope''}
	\Else 
	\State \label{Line:EnvelopeOverS}$\overline s \gets \max\{t>s:\ v_t=1\}$
	\State $\overline Q= \left\{j\in[d]:\ \frac{1}{|E_{\overline{s}}|}\sum_{i\in E_{\overline{s}}}\tilde Y_{ij}\leq \lambda_1(p-h)+\rho\sqrt{\lambda_1/|E_{\underline s}|}\right\}$ \Comment{use partial column sums on $E_{\overline s}$ to detect questions uniformly $\leq p-h$ on $E_{\overline s}$, detected questions $j$ as ``right envelope''}
	\EndIf
	\State $Q\gets \underline{Q}\cap\overline{Q}$
\end{algorithmic}
\end{algorithm}

\paragraph{Computational complexity of \ALGO{}.}  With a suitable implementation that keeps in memory values of partial row sums, the total computational complexity is (up to polylogarithmic terms) of the order of $m nd(d+n)$, where we recall that $m$ in the number threshold/tolerance under consideration. Thus, \ALGO{} is less than quadratic with respect to the size of the data set.

\section{Classification matrix estimation}\label{sec:estimation_R}

In this section, we explain how to deduce an estimator $\hat{R}_{p,h}$ of the 
classification matrix $R^*_{p,h}$ from some permutation estimator.

Given suitable estimators $\hat{\pi}$ and $\hat{\eta}$, and an independent sample $(N',I,J,Y')$ it is easy to build an estimator $\hat{R}_{p,h}$ of $R^*_{p,h}$ using a block averaging strategy. Define the window sizes
\begin{equation}\label{eq:k_h}
k_h:= \Big\lceil(\sigma\vee 1)\sqrt{\frac{512\log(nd)n}{d\lambda_0 h^2}} \Big\rceil \ , \quad \quad l_h:= \Big\lceil(\sigma\vee 1)\sqrt{\frac{512\log(nd)d}{n\lambda_0 h^2}} \Big\rceil \enspace . 
\end{equation}
If $k_h\geq n$ or $l_h\geq d$, we simply define $\hat{R}_{p,h}$ as the constant matrix equal to $1$.
Else, we define $\mathcal{B}$ as the regular grid of $[n]\times [d]$ with blocks of size $k_h\times l_h$. More specifically, for all $(r,s) \in [(\lfloor n/k_h\rfloor)\vee 1]\times [(\lfloor n/l_h\rfloor)\vee 1]$, we define the blocks $B_{r,s}:= [(r-1)k_h+1, rk_h]\times [(s-1)l_h+1, sl_h]$, except for the blocks such that $r=\lfloor n/k_h\rfloor$ or $s=\lfloor d/l_h\rfloor$ for which the blocks go up to $n$ or $d$. 

Then, the block constant matrix $\overline{Y}^{\hat{\pi},\hat{\eta}}_{\mathcal{B}}$ is defined by first ordering the data according to $\hat{\pi}$ and $\hat{\eta}$ and averaging over all blocks in $\mathcal{B}$. More precisely, for any block $B\in \mathcal{B}$ and any $(i,j)\in B$, we have
\begin{align*}
[\overline{Y}^{\hat{\pi},\hat{\eta}}_{\mathcal{B}}]_{ij} := \frac{1}{N_B^{\hat{\pi},\hat{\eta}}}\sum_{t=1}^{N'}\sum_{(k,l)\in B}\mathbbm{1}\{((I_t,J_t)=(\hat{\pi}^{-1}(k),\hat{\eta}^{-1}(l))\}Y'_t\enspace , 
\end{align*}
with the convention 0/0=0 and where $N_B^{\hat{\pi},\hat{\eta}}:= \sum_{t=1}^{N'}\sum_{(k,l)\in B}\mathbbm{1}\{((I_t,J_t)=(\hat{\pi}^{-1}(k),\hat{\eta}^{-1}(l))\}$ is the number of observations in $B$.  Finally, we define the estimator  $\hat{R}_{p,h}$ by 
\begin{equation}\label{eq:estimator:R_ph}
[\hat{R}_{p,h}]_{ij} = \mathbbm{1}\{[\overline{Y}^{\hat{\pi},\hat{\eta}}_{\mathcal{B}}]_{\hat{\pi}(i),\hat{\eta}(j)} \geq p \}
\enspace . 
\end{equation}

\begin{proposition}\label{prp:reconstruction}
For any $\pi$, $\eta$, any matrix  $M\in \mathbb{C}_{\mathrm{Biso}}(\pi,\eta)$, and any estimator $(\hat{\pi},\hat{\eta})$ independent of the sample $(N',I,J,Y')$ of intensity $\lambda_0$, we have
\begin{align}\label{eq:proposition:reconstruction}
\mathbb{E}\left[L_{0,1,\texttt{NA}}[ \hat{R}_{p,h}]\Big| \hat{\pi}, \hat{\eta}\right]&  \leq c\Big[ \sum_{s=1}^{\lceil \log_2(1/h)\rceil}\sum_{t\in \{-1,1\}}2^s\mathcal{L}_{p+ t h2^{s-2},h2^{s-2}}(\hat{\pi},\hat{\eta})  + \sum_{s=2}^{3}\sum_{t\in \{-1,1\}} \mathcal{L}_{p+3th2^{-s},h2^{-s}}(\hat{\pi},\hat{\eta})\notag\\ &  + 1+  (\sigma^2\vee 1) \frac{n\vee d}{\lambda_0 h^2}\log^2(nd)\Big]\wedge nd \enspace . \noindent 
\end{align}
\end{proposition}

In light of the above result, it suffices to split the data into two independent samples and use the first subsamples to estimate $\pi$ and $\eta$ in such a way that a polynomial number of  ranking losses of the form $\mathcal{L}_{p,h}$ are small and use the small sample to control estimate the matrix $M$ to have a small classification risk. More specifically, we divide the full sample $(N',I,J,Y')$ into two subsamples of intensity $\lambda_0/2$. We use the first subsample to build the permutations $(\hat{\pi}_S,\hat{\eta}_{S})$ defined in~\eqref{eq:definition_hat_pi} with the vector $(\underline{p},\underline{h})$ of (thresholds, tolerance) arising in~\eqref{eq:proposition:reconstruction}, and $\delta= (nd)^{-2}$. Then, we use the second subsample to build the block constant matrix $\overline{Y}^{\hat{\pi}_S,\hat{\eta}_{S}}_{\mathcal{B}}$. Thus, we obtain the final estimator
\begin{equation}\label{eq:estimator:R_ph:Fina}
	[\tilde{R}_{p,h}]_{ij} = \mathbbm{1}\{[\overline{Y}^{\hat{\pi}_S,\hat{\eta}_{S}}_{\mathcal{B}}]_{\hat{\pi}_S^{-1}(i),\hat{\eta}_{S}^{-1}(j)} \geq p \}
	\enspace . 
	\end{equation}
In Section~\ref{sec:main_results}, we have stated that this simple estimator is minimax optimal.

    \section{Extensions and discussion}\label{sec:Discussion}

\subsection{Connection to the noisy-sorting literature}

The \emph{noisy sorting} model~\cite{braverman2008noisy} is a specific case of SST tournament model. In its most specific form, it subsumes the existence of permutation $\pi$ of $[n]$ such that $M_{ij}= 1/2+h$ if $\pi(i)< \pi(j)$ and $M_{ij}= 1/2-h$ if $\pi(i)> \pi(j)$, the diagonal of $M$ being equal to $1/2$. We write $\mathbb{C}_{Noisy,h}$ for the collection of such matrices with a given $h>0$. In its most general form, it is assumed that the matrix $M$ satisfies $M_{ij}\geq 1/2+h$ if $\pi(i)< \pi(j)$ and $M_{ij}\leq 1/2-h$ if $\pi(i)> \pi(j)$ and write $\mathbb{C}'_{Noisy,h}$ for the collection of such matrices. 

With these models, authors generally express the quality of their ranking procedure using a distance on the space of the permutations such as  the Kendall tau distance defined by 
\begin{align*}
    d_{KT}(\pi',\pi)=\sum_{i=1}^n\sum_{j=1}^n\mathbbm{1}\{\pi(i)>\pi(j),\ \pi'(i)<\pi'(j)\} \ , 
\end{align*}
which, is some way, measures the level of disagreement when comparing any two players with respect to $\pi$ and $\pi'$. Interestingly, one readily checks that, for any matrix $M\in \mathbb{C}'_{Noisy,h}$,
\begin{equation}\label{eq:equivalence_loss}
\mathcal{L}_{1/2,h}(\pi',\pi')= 2d_{KT}(\pi',\pi)  \ ,
\end{equation}
where we recall that $\mathcal{L}_{p,h}$ is defined in~\eqref{eq:loss}. 
Generalizing the work of Braverman and Mossel~\cite{braverman2008noisy}, Mao et al.~\cite{mao2018minimax} have shown that the minimax optimal KT error over $\mathbb{C}'_{Noisy,h}$ is of the order of $n/(\lambda_0 h^2)$.  
Conversely, they  introduced a simple multi-sorting estimator $\hat{\pi}_{MWR}$ which, up to polylogarithmic terms, achieves the optimal rate, but only in the regime where $h$ is bounded away from zero and for restricted noisy matrices. More specifically, their estimator achieves
\[
\sup_{M\in \mathbb{C}_{Noisy,h}}\mathbb{E}\left[d_{KT}(\hat{\pi}_{MWR},\pi)\right]\leq c \frac{n}{\lambda_0}\wedge n^2\ , \text{ for }h\in [c',1/2)\enspace ,
\]
where $c$ and $c'$ are universal constants. Braverman and Mossel~\cite{braverman2008noisy} --see also Shah et al.~\cite{shah2016stochastically}-- have focused on the full observation model which roughly corresponds to $\lambda_0=1$ in our context. Relying on the disagreement-minimizing permutation $\hat{\pi}_{FAS}$, which can be efficiently computed, they manage to handle the generalized noisy sorting model, but still with a margin $h$ which is bounded away from zero. More specifically, they obtain 
\[
\sup_{M\in \mathbb{C}'_{Noisy,h}}\mathbb{E}\left[d_{KT}(\hat{\pi}_{FAS},\pi)\right]\leq c_h n\log(n) \text{ for }h\in (0,1/2)\enspace ,
\]
where the constant $c_h$ depends in a non-explicit way on the value of $h$. Besides, the computational complexity of their algorithm is exponential with respect to $1/h$.

In light of~\eqref{eq:equivalence_loss}, we can build\footnote{As explained earlier, a SST matrix $M$ is bi-isotonic up to the permutations $\pi$ and $\pi^{-}$ and  where $\pi^{-}$ is defined by $\pi^{-}(i)=n+1-\pi(i)$.} an estimator $\hat{\pi}_{S}$ from \ALGO{} 
that achieves
\begin{align*}
\sup_{M\in \mathbb{C}'_{Noisy,h}} \mathbb{E}[d_{KT}(\tilde{\pi}_S,\hat\pi)]\leq c' \log^{3/2}(n) \frac{n}{\lambda_0 h^2}\wedge n^2\enspace , 
\end{align*}
according to Theorem~\ref{Thm:ErrorBound}. Here, the constant $c'$ is universal. Hence,  our procedure, is up to our knowledge, the first one to achieve, in polynomial time, 
the optimal convergence rate with respect to the margin $h$ and the sampling rate $\lambda_0$, this on the whole class of generalized sorting models. We thereby answer an open question in~\cite{mao2018minimax}. Although the construction of $\hat{\pi}_S$ requires the knowledge of $h$, is also possible to be adaptive with respect to an unknown value $h$ by relying on multiple tolerance in \ALGO{}
--see Remark~\ref{remark:multiple:ph} for a more detailed discussion.

We would like to emphasize that the noisy sorting problem with margin $h$ has unique features which make it simpler than other bi-isotonic models. In particular, recovering the permutation $\pi$ is equivalent to recovering the level set $\{(i,j), M_{ij}>1/2\}$, this level being (up to the permutation $\pi$) completely triangular. It is much simpler to recover such a level set because of all the symmetries in the model: a procedure based on trisections like ours would not need to pay attention to potential asymmetries and imbalances around the level set $p$, which make in particular the choice of questions based on which one trisects much more complicated - see Subsection~\ref{ss:choiceQp} for a discussion on this.

In our work, we have introduced a polynomial time estimator, for recovering any level set $(p,h)$ in both the tournament and crowdsourcing model. In particular, our procedure does not require symmetries - as in the classical noisy sorting model - nor the existence of a margin in the matrix $M$ as for noisy sorting problems.

\subsection{Discussion of the choices of the set $Q'$ of questions in \texttt{ScanAndUpdate}}\label{ss:choiceQp}

We now further discuss the \texttt{ScanAndUpdate} routine, as the algorithmic idea behind as well as its analysis is one of the most innovative one in comparison to the available literature. 
Recall that we construct iteratively a partition $\tilde{\mathcal E}$ of the experts $[n]$, and we update iteratively a graph $G$, which both encode order information on the experts. In order to do that, we work at each step on all sets $E\in \tilde{\mathcal E}$, and we construct several sets of questions $Q'$ that we use in order to compare the experts on these questions - see mostly \texttt{ScanAndUpdate} in Algorithm~\ref{algo:scanupdate}, and also \texttt{Envelope} in Algorithm~\ref{Algo:Envelope} for the creation of the envelope set $Q$ based on which the $Q'$ are constructed. Note first that, for any $E\in \tilde{\mathcal E}$,  a natural oracle choice of the set $Q'$ would be to consider $Q^*(E)$, which contains precisely questions on which experts in $E$ display substantial performance gaps around the threshold $p$, see Equation~\eqref{eq:definition:Q*}. However, two distinct challenges occur:
\begin{itemize}
    \item(i) first, $Q^{*}(E)$ is not available and needs to be estimated, and 
    \item(ii) second, while $Q^*(E)$ is the set of questions that is most effective for distinguishing the best expert in $E$ from the worst expert in $E$, many questions inside it might be irrelevant for distinguishing between two given experts $i,i'$ of $E$, e.g.~if they lie more toward the median of $E$ - see below for a detailed discussion and an illustration of this important point. So that in some situations there exist subsets $Q_0$ of $Q^*(E)$ that are much more informative to distinguish between experts $i$ and $i'$.
\end{itemize}

\begin{figure}[H]
	\centering
	\subfloat[We estimate $Q^*(E)$ using some experts $i'< i\ \forall i\in E$ and some experts $i'>i\ \forall i\in E$ (depicted as red  and blue rows, resp., $E$ corresponds to the orrange rows and $Q^*(E)$ to the gray columns). In Section~\ref{Sec:Questions}, we explain why the corresponding dashed areas must contain $Q^*(E)$ which yields $Q$ as intersection to be the output of the \texttt{Envelope} procedure. \label{Subfig:Envelopes}]{
		\includegraphics{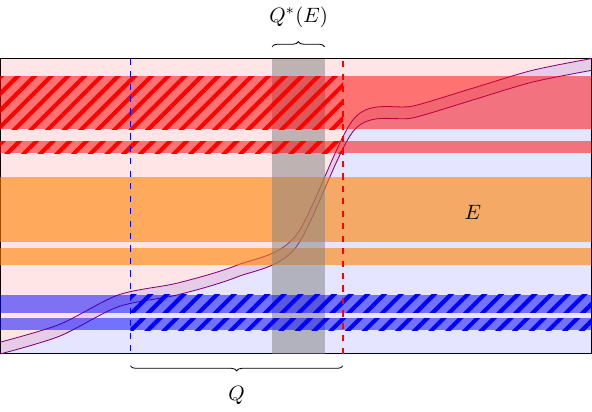}}
	
	\subfloat[We explain in Section~\ref{Sec:Trisection}, how, for each $j$, nested intervals left and right of $j$ are constructed as candidate sets $Q'$. These intervals are depicted in gray with different shades.\label{Subfig:Scan}]{
		\includegraphics{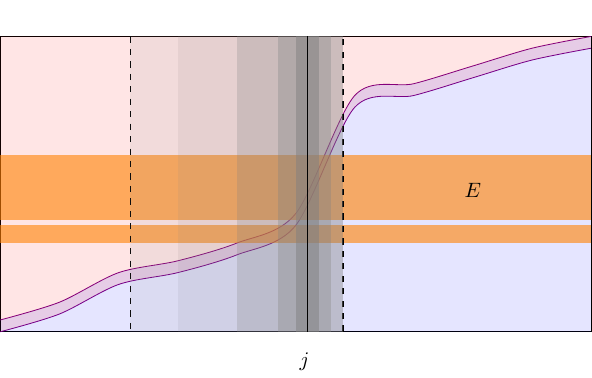}}
\caption{Illustration of some bi-isotonic matrix $M\in \mathbb{C}_{\mathrm{Biso}}(\mathrm{id}_{\lbrack n\rbrack},\mathrm{id}_{\lbrack d\rbrack})$ so that $\pi=\mathrm{id}_{\lbrack n\rbrack}$ and $\eta= \mathrm{id}_{\lbrack d\rbrack}$. The two purple curves divide the matrices into three areas: values, that are at least $p+h$ (light red background); values, that are at most $p-h$ (light blue background); and values that are between $p+h$ and $p-h$ (violet background). We construct sets $Q'$, first using \texttt{Envelope} (\ref{Subfig:Envelopes}) and then in \texttt{ScanAndUpdate} (\ref{Subfig:Scan}). \label{Fig:Overview}}
\end{figure}

To handle this, we have to estimate a collection of proxies $Q'$ of $Q_0$ in such a way that: 
\begin{itemize}
\vspace{-0.5cm}    \item[(a)] at least one proxy  $Q'$ contains a "significant" number of elements of $Q_0$, while its cardinality should not be too large (compared to $|Q_0|$),\vspace{-0.2cm}
    \item[(b)] the number of elements in this collection of proxies is not too large (i.e.~polynomial with $n,d$).
\end{itemize}\vspace{-0.5cm}
The purpose of Property (a) is that comparing experts based on such $Q'$ will be as effective as a comparison based on $Q_0$. Properties (b) also enforces that the computational complexity of the procedure remains controlled.

We would like to emphasize that in classical noisy sorting - tournament setting where $n=d$, where $p=1/2$ and where the matrix $M$ takes values in $\{1/2-h, 1/2+h\}$ around its diagonal -  we have $Q^*(E)=E$ for any set $E$ of experts, by symmetry. So that the problem of estimating $Q^*(E)$ does not exist in classical noisy sorting.  Besides, in this noisy sorting model, computing averages based on $Q^*(E)$, is mostly relevant for comparisons of experts in $E$, because of  symmetries mentioned earlier. In our setting however - i.e.~either in a tournament setting where we possibly consider a different threshold than $p=1/2$, or in a ranking setting - we do not have the same symmetries as in noisy sorting. So that $Q^*(E)$ cannot be straightforwardly constructed based on $E$ anymore and some subsets $Q_0$ of $Q^*(E)$ are perhaps more informative.

\textit{In our setting, our main algorithmic and conceptual innovation - compared to noisy sorting~\cite{mao2018minimax} and other ranking algorithms~\cite{mao2020towards,liu2019minimax,pilliat2022optimal,pilliat2024optimal} - is the construction of the sets $Q'$ and their analysis, which turned out to be surprisingly challenging.} We already explained the intuition behind \texttt{Envelope} in Section~\ref{Sec:Questions}.
This is however insufficient in very asymmetric situations. Indeed, even in a noiseless case where our estimation of the questions in $Q$ would be perfect, one might have that only a small subset of $Q$ obtained \texttt{Envelope} is actually composing $Q^*(E)$ - see Figure~\ref{Subfig:Envelopes} and even worse, only a small subset of $Q^*(E)$ is more informative. For this reason, we go one step further in the refinement of our question sets and construct several subsets $Q'$ of $Q$ which are candidate proxies for $Q^*(E)$ - see \texttt{ScanAndUpdate} in Algorithm~\ref{algo:scanupdate}. The idea is to consider each question $j\in Q$. Then, based on the averages of the observations on $E$, we either select the subset of question $j'$ whose averages are close to that of $j$ - Line~\ref{Line:TrisectInBetween}  - or questions $j'$ whose averages belong to some nested collection of bins below that of $j$ - Line~\ref{Line:TrisectSmaller} - or  questions $j'$ whose averages belong to some bins above that of $j$ - Line~\ref{Line:TrisectLarger}. The size of these bins has been chosen in such a way that the error due to the noise is small compared to the respective size of these bins.  See Figure~\ref{Subfig:Scan} for an example. In this way, we target the actual set of questions where there are most variations among experts within $E$, ensuring that we are not impacted by the case where this set  is actually smaller than $Q$ - see Figure~\ref{Subfig:Envelopes}. In the theoretical analysis of our algorithm, we prove that whenever the variation within $E$ is large enough, one of these sets $Q'$ will be a relevant set of questions, namely that it will be of small enough size, yet contain sufficiently many questions where sufficiently many of the experts disagree so that it can be used to effectively discriminate between the experts. This is done by carefully crafted volumetric arguments.

\subsection{Two (or several) values problem}\label{ss:2va}

To simplify the discussion, we assume in this subsection that $n\geq d$. Let us first focus on a specific, yet emblematic case, where the matrix $M$ only takes two values, say $p-h$ and $p+h$ - note though that here we do not necessarily consider the noisy sorting setting. In that setting, our estimators $\hat{\pi}_S$ and $\hat{\eta}_S$~\eqref{eq:definition_hat_pi} satisfy the optimal risk bound from Theorem~\ref{Thm:ErrorBound}
\[
\mathbb{E}\left[\mathcal{L}_{p,h}(\hat{\pi}_S,\hat{\eta}_S)\right]\lesssim \log^{5/2}(n)(\sigma^2\vee 1)\frac{n}{\lambda_0h^2 } \enspace . 
\]
This bound interprets as the fact that, on average, the estimated permutation $\hat{\pi}_S$ wrongly ranks rows $i$ and $j$ only if those differ by less $1/(\lambda_0 h^2)$ entries. Regarding the matrix $M$, in this setting,  it is equivalent, up to a factor $h^2$, to reconstruct the matrix $M$ in square Frobenius norm and to estimate the level set $R^{*}_{p,h}$. Indeed, upon defining $\hat{M}= p-h \mathbf{1}_{n\times d} + 2h\tilde{R}_{p,h}$ where $\tilde{R}_{p,h}$ is defined in Equation~\ref{eq:estimator:R_ph}, we deduce from Theorem~\ref{Thm:ErrorBound_classificatoin} that 
\begin{equation}\label{eq:risk:two_values}
\mathbb{E}\left[\|\hat{M}-M\|_F^2\right]\lesssim \log^{7/2}(n)(\sigma^2\vee 1)\frac{n}{\lambda_0}\enspace . 
\end{equation}
We know e.g. from~\cite{mao2020towards} that this Frobenius risk bound is, up to logarithmic term,  rate-optimal.

\medskip

In fact, the above results easily extend to the case where the matrix $M$ takes $K\geq 2$ values. 
Now assume that the square matrix $M$ only takes $K$ values $(p_1,\ldots, p_K)$ with $p_1< p_2 < \ldots p_K$. Using \ALGOS, we can construct $(\hat{\pi}_S,\hat{\eta}_{S})$ and using block-constant estimation of the matrix in the spirit of Equation~\ref{eq:estimator:R_ph:Fina} we can build an estimator $\hat{M}$ achieving
\begin{equation}\label{eq:risk:k_values}
\mathbb{E}\left[\|\hat{M}-M\|_F^2\right]\lesssim \log^{7/2}(n) K (\sigma^2\vee 1)\frac{n}{\lambda_0}\enspace ,
\end{equation}
which only differs from~\eqref{eq:risk:k_values} by the multiplicative factor $K$. In summary, our polynomial time procedure, achieves the optimal reconstruction risk when the number of values is considered as constant.

\subsection{Narrowing the computational barrier for bi-isotonic problems}\label{ss:compbar}

For a sake of clarity, we focus in this subsection on the square case where $n=d$, $\lambda_0=1$, and $\sigma\leq 1$. As explained in the introduction, it is known~\cite{shah2015estimation,shah2016stochastically,mao2018minimax} that the minimax risk for estimating $M$ in Frobenius norm is of the order $n$, i.e.
\[
cn \leq \inf_{\hat{M}} \sup_{M\in \mathbb{C}_{\mathrm{biso}}}\mathbb{E}_{M}[\|\hat{M}-M\|_F^2]\leq c'\log^2(n)n \enspace .
\]
However,  the best available polynomial time procedure only achieves, up to polylog terms, the risk $n^{7/6}$~\cite{liu2020better,pilliat2024optimal} despite a long line of research on this topic~\cite{shah2016stochastically,pananjady2022isotonic,shah2015estimation}, which leads to the conjecture of a computational barrier for this problem~\cite{mao2020towards}. As already pointed out in~\cite{mao2020towards}, the problem of estimating $M$ mostly boils down to that of estimating the permutation $\pi$ and $\eta$. More specifically, if we have at hand estimators $\hat{\pi}$ and $\hat{\eta}$, we can simply estimate $M$ using the least-square estimator on the space of bi-isotonic matrices based on a permuted matrix of observations $Y_{\hat{\pi}^{-1},\hat{\eta}^{-1}}$ - where we do sample splitting and consider a sample $Y$ independent of $\hat{\pi}$ and $\hat{\eta}$. Then, it is proved in~\cite{mao2020towards} that 
\[
\mathbb{E}[\|\hat{M}-M\|_F^2]\leq \mathbb{E}[\|M_{\hat{\pi}^{-1},\hat{\eta}^{-1}}-M_{\pi^{-1},\eta^{-1}}\|_F^2] + cn\log^2(n) \enspace .
\]
A similar result holds if we restrict our attention to SST matrices. As a consequence, the existence of a computational barrier for Frobenius estimation of bi-isotonic matrices up to permutations is equivalent to the existence of the barrier for the problem of permutation estimation - namely estimation of $\pi$ and $\eta$ - with respect to the Frobenius-type losses $\mathcal{L}_F$ defined in~\eqref{eq:definition:LF}. 

In this work and more specifically in Remark~\ref{remark:multiple:ph}, we have introduced a polynomial time-estimator $(\hat{\pi}_S,\hat{\eta}_{S})$ achieving, for all $(p,h)\in (0,1)$, the risk bound 
\begin{equation}\label{eq:upper_risk_multiple}
\mathbb{E}\left[\mathcal{L}_{p,h}(\hat{\pi}_S,\hat{\eta}_S)\right]\leq c'\log^{5/2}(n) \frac{n}{h^2} \ , 
\end{equation}
thereby showing the absence of computation-information gap for simultaneous inference with respect to all the weaker losses $\mathcal{L}_{p,h}$. This has an important consequence. As underlined in Subsection~\ref{ss:2va}, the statistical-computational gap disappears when the matrix $M$ only takes a (small) finite  number of values. This implies that the conjectured computation barrier for estimation of SST and, up to permutations, bi-isotonic matrices can only arise for multi-scale matrices taking a polynomial number of different values. A reason for the difference between the general setting, and with the finite number of values for setting, is as follows. In general, one should not consider the reconstruction of each level set separately - as we do in our procedure \ALGO{}. This is sufficient if $M$ takes a small number of different values. In general, one should aggregate information contained in several relevant level sets. In order to obtain the optimal rate of reconstruction. Doing this is very challenging, and we believe that it is not possible in polynomial time, leading back to the conjectured statistical-computational gap in the general case. 

\section*{Acknowledgements}

The work of N. Verzelen has been partially supported by grant ANR-21-CE23-0035 (ASCAI,ANR). The work of A. Carpentier is partially supported by the Deutsche Forschungsgemeinschaft (DFG)- Project-ID 318763901 - SFB1294 "Data Assimilation", Project A03. The work of A.~Carpentier and M.~Graf is also partially supported by the DFG on the Forschungsgruppe FOR5381 "Mathematical Statistics in the Information Age - Statistical Efficiency and Computational Tractability", Project TP 02 (Project-ID 460867398), by the Agence Nationale de la Recherche (ANR) and the DFG on the French-German PRCI ANR-DFG ASCAI CA1488/4-1 "Aktive und Batch-Segmentierung, Clustering und Seriation: Grundlagen der KI" (Project-ID 490860858).
    
    \printbibliography
    \newpage
    \appendix 
    \section{Reduction to rows and columns sorting}\label{Sec:RowsAndCols}

Recall that our goal is the estimation of unknown permutations $\pi\in\mathcal S_n$ and $\eta\in \mathcal{S}_n$. For estimators $\hat{\pi}$ and $\hat{\eta}$, we have defined a loss function $\mathcal{L}_{p,h}(\hat\pi,\hat\eta)$ which depends on $\pi$, $\eta$ and the unknown matrix $M\in \mathbb{C}_{\mathrm{Biso}}(\pi,\eta)$. Let us define analog losses that consider only consider row and column permutations respectively, namely
\begin{align*}
	\mathcal{R}_{p,h}(\hat \pi)\coloneqq& \left|\left\{(i,j)\in [n]\times [d]:\ M_{\pi^{-1}(i)j}\leq p-h,\ M_{\hat\pi^{-1}(i)j}\geq p+h\right\}\right|\\
	&+\left|\left\{(i,j)\in [n]\times [d]:\ M_{\pi^{-1}(i)j}\geq p+h,\ M_{\hat\pi^{-1}(i)j}\leq p-h\right\}\right|\enspace ,\\ \intertext{and}
     \mathcal{C}_{p,h}(\hat \eta )\coloneqq& \left|\left\{(i,j)\in [n]\times [d]:\ M_{i\eta^{-1}(j)}\leq p-h,\ M_{i\hat \eta^{-1}(j)}\geq p+h\right\}\right|\\
	&+\left|\left\{(i,j)\in [n]\times [d]:\ M_{i\eta^{-1}(j)}\geq p+h,\ M_{i\hat \eta^{-1}(j)}\leq p-h\right\}\right|\enspace.
\end{align*}

The following lemma, proved in section~\ref{sec:proof:lemma:reduction}, states that the global loss $\mathcal{L}$ can be bounded by shifts of the losses $\mathcal{R}, \mathcal{C}$ for respectively $\hat \pi$ and $\hat \eta$.
\begin{lemma}\label{Lemma:LossBound}
    It holds
    \begin{align*}
        \mathcal{L}_{p,2h}(\hat{\pi},\hat{\eta})\leq 2\left(\mathcal{R}_{p-h,h}(\hat{\pi})+\mathcal{C}_{p+h,h}(\hat{\eta})\right) \wedge 2\left(\mathcal{C}_{p-h,h}(\hat{\eta})+\mathcal{R}_{p+h,h}(\hat{\pi})\right)
    \end{align*}
    and 
    \begin{align*}
        \mathcal{L}_{p,h}(\hat{\pi},\hat{\eta})\geq \mathcal{R}_{p,h}(\hat{\pi})\vee \mathcal{C}_{p,h}(\hat{\eta})\enspace .
    \end{align*}
\end{lemma}
In what follows, we will therefore mostly focus on the estimation of $\pi$ through $\hat{\pi}$ and the loss $\mathcal{R}_{p,h}(\hat{\pi})$ - which by symmetry will allow us to bound $\mathcal{C}_{p,h}(\hat{\eta})$ and then $\mathcal{L}_{p,h}(\hat{\pi},\hat\eta)$.

\begin{theorem}\label{Thm:ErrorBoundRows}
Let $\pi\in\mathcal{S}_n$, $\eta \in \mathcal{S}_d$, $M\in \mathcal{C}_{\mathrm{Biso}}(\pi,\eta)$ and $(N',I,J,Y')$ an observation that satisfies Definition~\ref{Def:Model} under some distribution $\mathbb{P}$ with subGaussian constant $\sigma^2$. Given $p,h\in[0,1]$, there exists a universal constant $c>0$ such that for the permutation estimator $\hat \pi_S$ defined in~\eqref{eq:definition_hat_pi} holds with probability at least $1-\delta/2$
\begin{align*}
		\mathcal{R}_{p,h}(\hat \pi_S)\leq c(\sigma^2\vee 1) \log(nd/\delta)^{5/2}\frac{n\vee d}{\lambda_0 h^2}\wedge nd\enspace .\end{align*}
\end{theorem}
The main challenge in this paper will be to establish this theorem and this is the purpose of the two next subsections --see in particular Section~\ref{sec:proof:analysis:permutation}. Before this, we show how Theorem~\ref{Thm:ErrorBound} is easily obtained from this last result.

\begin{proof}[Proof of Theorem~\ref{Thm:ErrorBound}]
    Consider $\hat{\pi}_S$ --defined in~\eqref{eq:definition_hat_pi}-- for the threshold $p-h/2$, tolerance $h/2$ and tuning parameter $\delta/2$. By exchanging the roles of rows and columns, we also consider $\hat \eta_S$ for the threshold $p+h/2$, tolerance $h/2$ and tuning parameter $\delta/2$. By Theorem~\ref{Thm:ErrorBoundRows} we know for some universal constant $c>0$ that 
    \begin{align*}
        \mathcal{R}_{p-h/2,h/2}(\hat \pi_S)\leq c(\sigma^2\vee 1) \log(nd/\delta)^{5/2}\frac{n\vee d}{\lambda_0 h^2}\wedge nd\enspace ,\\
        \mathcal{C}_{p+h/2,h/2}(\hat \eta_S)\leq c(\sigma^2\vee 1) \log(nd/\delta)^{5/2}\frac{n\vee d}{\lambda_0 h^2}\wedge nd\enspace ,
    \end{align*}
    both with probability at least $1-\delta/2$. 
    Together with the first part of Lemma~\ref{Lemma:LossBound}, a union bound yields 
  \begin{align*}
      \mathcal{L}_{p,h}(\hat\pi_S,\hat\eta_S)&\leq 2 \mathcal{R}_{p-h/2,h/2}(\hat\pi_S)+2\mathcal{C}_{p+h/2,h/2}(\hat\eta_S)\\
      &\leq 4c(\sigma^2\vee 1) \log(nd/\delta)^{5/2}\frac{n\vee d}{\lambda_0 h^2}\wedge nd\enspace ,
  \end{align*} 
  with probability at least $1-\delta$.
\end{proof}

\section{General Analysis of \ALGO{} (Proof of Theorem~\ref{Thm:ErrorBoundRows})}\label{Sec:AnalysisAlgos}

In this section, we will give an overview of the analysis of each of the building blocks of the proof of Theorem~\ref{Thm:ErrorBoundRows}, where the row loss $\mathcal{R}_{p,h}$ is considered for the output $\hat \pi_S$ of \ALGO{} with only one threshold $p$ and tolerance parameter $h$ as input. Respective proofs are postponed to later sections of this appendix. 

In Section~\ref{sec:proof:trisect:valid}, we show, that with high probability, all the trisection of a set $E$ into three sets $(O,P,I)$ satisfy suitable properties with respect to the ordering. Then, Section~\ref{sec:proof:trisection}, which is the keystone part of our proof, controls the permutation error which is dues to indecisive sets $P$ --in fact we look with a superset $\overline{P}$ of $P$. Section~\ref{sec:proof:enveloppe} is dedicated to the analysis of the \texttt{Envelope} routine. In Section~\ref{sec:proof:tree}, we gather all these results to study the properties of the sorting tree. Finally, in Section~\ref{sec:proof:analysis:permutation}, we control the error $\mathcal{R}_{p,h}(\hat{\pi}_S)$ of the ordering and thereby we prove Theorem~\ref{Thm:ErrorBoundRows}

\subsection{Validity of the trisection procedure}\label{sec:proof:trisect:valid} We start by analyzing the procedures \texttt{GraphTrisect}, \texttt{UpdateGraph} and \texttt{ScanAndUpdate} (Algorithms~{\ref{Algo:Trisect}, \ref{algo:scanupdate} and \ref{Algo:Graph}}) for a generic input first. More precisely,
we assume here that these procedures are fed with data  $\tilde{Y}^{(a)},\tilde{Y}^{(b)}$, set $E$ of experts, set $Q$ of questions and comparison graph $G$ that satisfy the following property.
\renewcommand\theassumption{\arabic{assumption} for $\tilde Y^{(a)}$, $\tilde Y^{(b)}$, $E$, $Q$, and $G$}
\begin{assumption}
 \renewcommand\theassumption{\arabic{assumption}}
 \addtocounter{assumption}{-1}
 \refstepcounter{assumption}
	\phantom{a} \vspace{-0.5cm}
	\begin{itemize}
 		\item $\tilde Y^{(a)}$ and $\tilde Y^{(b)}$ are independent and distributed as described in Equation~\eqref{Eq:Model},
		\item $Q^*(E)= \{j\in[d]:\ \max_{i\in E}M_{ij}\geq p+h,\ \min_{i\in E}M_{ij}\leq p-h\}\subseteq Q$,
		\item $\lambda_1(|E|\wedge|Q|)\geq 4\rho^2/h^2$,
		\item For all $i,i'\in E$, $G_{ii'}=1\Rightarrow \pi(i)<\pi(i')$.
	\end{itemize}
  \label{Assum1}
 \end{assumption}

We will later show that, outside an event of small probability, this property is satisfied each time, we apply \texttt{GraphTrisect}, \texttt{UpdateGraph} or \texttt{ScanAndUpdate}. The last point of Property~\ref{Assum1} states that $G$ only contains correct information  about $\pi$. 

The next lemma states that, when \texttt{ScanAndUpdate} is fed with $\tilde Y^{(a)}$, $\tilde Y^{(b)}$, $E$, $Q$, and $G$ that satisfy Property~\ref{Assum1}, then the updated graph $\tilde{G}$ still satisfies the last point of Property~\ref{Assum1}.

\begin{lemma}\label{Lemma:Graph}
	Under Property~\ref{Assum1} for $\tilde Y^{(a)}$, $\tilde Y^{(b)}$, $E$, $Q$ and $G$, on an event of probability at least $1-\frac{|Q|\delta}{12\lceil \log_2(n\vee d)\rceil d}$, the updated graph $\tilde{G}$ obtained from~\texttt{ScanAndUpdate} satisfies
	\begin{align*}
		\tilde G_{ii'}=1\Rightarrow \pi(i)<\pi(i') 
	\end{align*}
	for all $i,i'\in E$ on the updated graph $\tilde G$. 
\end{lemma}

With the updated graph $\tilde G$, \texttt{GraphTrisect} produces what we call \emph{trisection} $(O,P,I)$ of $E$ in the following manner:
\begin{align}
    O&\coloneqq \{i\in E:\ \sum_{i'\in E} \tilde G_{ii'}>|E|/2\}  \ ;\quad \quad  
    I\coloneqq \{i\in E:\ \sum_{i'\in E} \tilde G_{i'i}>|E|/2\} \ ; \quad \label{Eq:DefTrisection} 
    P\coloneqq E\setminus (O\cup I)\enspace . 
\end{align}
We further define $\overline i$ --the median of $E$ with respect to $\pi$-- by the property
\begin{align}
	|\{i'\in E:\ \pi(i')\leq \pi(\overline i)\}|=\lceil |E|/2\rceil\enspace , \label{Eq:DefMedian}
\end{align}
and prove the following intuitive statement about the trisection of $E$.

\begin{lemma}\label{Lemma:PropertiesTrisection}	Under Property~\ref{Assum1} for $\tilde Y^{(a)}$, $\tilde Y^{(b)}$, $E$, $Q$ and $G$, it holds with a probability of at least $1-\frac{|Q|\delta}{12\lceil \log_2(n\vee d)\rceil d}$ that
	\begin{align}
		\pi(i_O)<\pi(\overline{i})<\pi(i_I) \hspace{.5cm} \forall i_O\in O, \ i_I\in I\enspace . \label{Eq:RelationTrisectionMedian}
	\end{align}
\end{lemma}

 Inequality~\eqref{Eq:RelationTrisectionMedian} shows, that $(O,P,I)$ is a partition of $E$ with $\overline i\in P$, all the elements $i_O$ of $O$ satisfying $\pi(i_O)<\pi(\overline{i})$, and all elements $i_I$ of $I$ satisfying $\pi(i_I)>\pi(\overline{i})$. This implies that elements of $O$ are well ordered in comparison to $I$.

 \subsection{An error bound for the trisection\label{sec:proof:trisection} } When we trisect a set $E$ into $(O,P,I)$, the sets $O$ and $I$ are well separated in the sense of Lemma~\ref{Lemma:PropertiesTrisection}. $P$ can be seen as a set of experts that are indistinguishable from $\overline i$, and it is possible that $i\in P$ exists, such that either $\pi(i)<\pi(i_O)$ for some $i_O\in O$ or $\pi(i)>\pi(i_I)$ for some $i_I\in I$. This overlap is an important error  source, since we want to build a sorting tree based on trisections. We describe here briefly how we control this error and postpone the proofs of these results to Section~\ref{Sec:TrisectErrorBound}.

 \texttt{ScanAndUpdate} not only compares experts by considering partial row sums with respect to $Q$, but also by looking at suitable subsets $Q'\subseteq Q$. We construct them in Section~\ref{Sec:AnalysisTrisection}, more precisely in \eqref{Eq:DefLeftRight}. The collection of the sets $Q'$ on that we compare experts in $E$ during \texttt{ScanAndUpdate} is defined as $\mathcal{Q}$. Instead of the unobservable median $\overline i$ defined in \eqref{Eq:DefMedian}, let us consider some empirical counterparts: For $Q'\in \mathcal Q$  we define the empirical median $\med{Q'}$ via the property
	\begin{align*}
	|\{i\in E:\ \yb{i}{Q'}> \yb{\med{Q'}}{Q'}\}|< |E|/2 \text{ and } |\{i\in E:\ \yb{i}{Q'}< \yb{\med{Q'}}{Q'}\}|\leq |E|/2\enspace .
\end{align*}
Note that for the definition of the empirical median $\med{Q'}$, multiple choices might be possible. In that case, one can choose an arbitrary expert with the defining properties without effecting any proofs.

Our aim is to control the error inside an extension of $P$ (see Lemma~\ref{Cor:SubsetOfTilde}), namely
\begin{align}
		\overline{P}\coloneqq \left\{i\in E:\ |\yb{i}{Q'}-\yb{\med{Q'}}{Q'}|\leq8\rho\sqrt{\lambda_1/|Q'|}\ \forall Q'\in\mathcal{Q} \right\}\enspace . \label{Eq:DefConservativeTrisection}
\end{align}
For doing that, we will consider a slightly different error metric. For given $p,h>0$, $E'\subseteq [n]$ and $Q'\subseteq [d]$, it relates to the maximal number of entries in $E'\times Q'$ above $p+h$ that can be confused with entries below $p-h$ by exchanging experts in $E'$ and is given as
\begin{align}\label{eq:definition:loss:Rtilde}
    \tilde{\mathcal{R}}_{p,h,E',Q'}&\coloneqq \sum_{j\in Q'}|\{i\in E':\ M_{ij}\leq p-h\}|\wedge |\{i\in E':\ M_{ij}\geq p+h\}|\enspace .
\end{align}
We can prove the following bound.
\begin{lemma}\label{Lemma:Blocks}
	Under Property~\ref{Assum1} for $\tilde Y^{(a)}$, $\tilde Y^{(b)}$, $E$, $Q$ and $G$, with a probability of at least \linebreak$1-\frac{|Q|\delta}{12\lceil \log_2(n\vee d)\rceil d}$, it holds
	\begin{align*}
		\tilde{\mathcal{R}}_{p,h,\overline P,[d]}&= \sum_{j\in [d]}|\{i\in \overline{P}:\ M_{ij}\leq p-h\}|\wedge |\{i\in \overline{P}:\ M_{ij}\geq p+h\}|\\
		&=\sum_{j\in Q}|\{i\in \overline{P}:\ M_{ij}\leq p-h\}|\wedge |\{i\in \overline{P}:\ M_{ij}\geq p+h\}|\leq 3744\rho(|E|\vee |Q|)/\lambda_1h^2\enspace .
	\end{align*}
\end{lemma}
This lemma is one of the key parts of our proof. While the above bound would have been quite simple to prove e.g.in a toy models where $Q=Q^*(E)$, $p=1/2$, $M$ corresponds to a noisy sorting model with values $1/2-h$ and $1/2+h$, this turns out to much more challenging in the general case. 
The proof of Lemma~\ref{Lemma:Blocks} can be found in Section~\ref{Sec:TrisectErrorBound} --see in particular Figure~\ref{Fig:Trisection}.

\subsection{Analysis of the envelope procedure}\label{sec:proof:enveloppe}

As in the analysis of the trisection procedure, we will first analyze the procedure \texttt{Envelope} (Algorithm~\ref{Algo:Envelope}) for some generic input. We consider $\tilde Y$, $\mathcal{E}=(E_1,E_2,\dots,E_r)$ and $v=(v_1,v_2,\dots,v_r)$ that satisfies the following:
 \renewcommand\theassumption{\arabic{assumption} for $\tilde Y$, $\mathcal E$ and $G$}
 \begin{assumption}
 \renewcommand\theassumption{\arabic{assumption}}
 \addtocounter{assumption}{-1}
 \refstepcounter{assumption}
	\phantom{} \vspace{-0.5cm}
	\begin{itemize}
		\item $\tilde Y$ is distributed as described in Equation~\eqref{Eq:Model},
		\item $\mathcal{E}$ is a partition of $[n]$,
		\item $v_s=1 \Rightarrow \lambda_1|E_s|>4\rho^2/h^2$,
		\item $v_s=v_{s'}=1$ for $s<s'$ implies $\pi(i)<\pi(i')$ for $i\in E_s$ and $i'\in E_{s'}$.
	\end{itemize}
    \label{Assum2}
\end{assumption}
The two last points mean that $v$ indicates whether groups of experts are large enough and that these large groups are ordered with respect to $\pi$. If we obtain the set $Q_s$ as an output of \texttt{Envelope}($s$, $\mathcal{E}$, $v$, $\tilde Y$, $p$, $h$, $\lambda_0$, $\delta$) for all $s$ with $v_s=1$, we show the following.
\begin{lemma}\label{Lemma:Envelopes}
	Under Property~\ref{Assum2} for $\tilde Y$, $\mathcal{E}$ and $v$, with probability at least $1-\frac{\delta}{12\lceil \log_2(n\vee d)\rceil (n\vee d)^{1/2}}$, it holds that
	\begin{align*}
\sum_{s=1,\ v_s=1}^r |Q_s|\leq 3d\ , \text{ and }Q^*(E_s)\subseteq Q_s \hspace{.5cm}\forall s=1,2,\dots,r \text{ with } v_s =1\enspace .
	\end{align*}
\end{lemma}

This enforces that the sets $Q_s$ obtained through \texttt{Envelope} all contain the relevant sets of questions $Q^*(E_s)$ and that the total size of the estimated questions,  $\sum_{s=1,\ v_s=1}^r |Q_s|$,  is controlled. 

\subsection{Analysis of the sorting tree}\label{sec:proof:tree}

Equipped with the previous results, we establish some properties of the sorting tree that we have built. In order so state these properties, we need to introduce further notation to formalize of the intermediary objects built by the algorithm.  We briefly described \ALGO{} (Algorithm \ref{Algo:Tree}) in Section \ref{Sec:Tree}. Let us now define
\begin{align*}
	\mathcal{E}^0\coloneqq ([n]),\ v^0\coloneqq (\mathbbm{1}\{\lambda_1 n> 4\rho^2/h^2\}),\ r_0=1 \text{ and } G^0=\mathbf{0}_{n\times n}\enspace ,
\end{align*} and assume we are given independent observations $Y^{(1)},Y^{(2)},\dots,Y^{(3\lceil \log_2(n)\rceil)}$ from the Poisson observation model \eqref{Eq:Model}, as described in Section \ref{Sec:Algos}.

\ALGO{} operates in main rounds, indexed by $k$, starting with $k=1$, which correspond to successive refinement of the partition - and sometimes also in sub-rounds indexed by $t$, which correspond to successive trisections. In Algorithm \ref{Algo:Tree}, notations corresponding to the rounds are not indexed by the rounds in order to alleviate notation. We introduce below the sequential notations corresponding to all rounds, in order to be able to analyze the algorithm. For $k\geq 1$, given an ordered partition $\mathcal{E}^{k-1}$, $v^{k-1}\in\{0,1\}^{r_{k-1}}$, $r_{k-1}$ and a graph $G^{k-1}\in\{0,1\}^{n\times n}$, let us define:
\begin{itemize}
	\item $Q_{t}^{k-1}$ for $v_t^{k-1}=1$; output of the routine \texttt{Envelope} (Algorithm \ref{Algo:Envelope}) with parameters $t$, $\mathcal{E}^{k-1}$,  $v^{k-1}$, $ Y^{(3k-2)}$, computed in line \ref{line:envelope} of Algorithm \ref{Algo:Tree}
	\item $w_t^{k-1}\in\{0,1\}^{r_{k-1}}$ defined by 
	\begin{align*}
		w_t^{k-1}=
		\begin{cases}
			\mathbbm{1}\{\lambda_1|Q_t^{k-1}|>4\rho^2/h^2\},& \text{if }v_t^{k-1}=1\enspace , \\
			0 &\text{if }v_t^{k-1}=0\enspace ,
		\end{cases}
	\end{align*}
 corresponds to the indexes of the sets of the partition $\mathcal{E}^{k-1}$ to be further refined. 
	\item $r_k\coloneqq |\{t=1,2,\dots,r_{k-1}:\ w_t^{k-1}=0\}|+3\cdot |\{t=1,2,\dots,r_{k-1}:\ w_t^{k-1}=1\}|$.  Here, $r_k$ will stand for the size of the partition $\mathcal{E}^{k}$. 
	\item $G_0^k\coloneqq G^{k-1}$.
\end{itemize}
Consider for $t=1,2,\dots,r_{k-1}$ the index $s=|\{t'<t:w_{t'}^{k-1}=0\}| +3\cdot|\{t'<t:w_{t'}^{k-1}=1\}|+1$. Then, if $w_t^{k-1}=0$, we do not refine $E_t^{k-1}$ anymore and we define $E_s^k\coloneqq E_t^{k-1}$, $v_s^k\coloneqq 0$, and $G_t^k\coloneqq  G_{t-1}^k$. If $w_{t}^{k-1}=1$, we denote 
\begin{itemize}
	\item $(O(E_t^{k-1}),P(E_t^{k-1}),I(E_t^{k-1}))$; output of \texttt{GraphTrisect} (Algorithm  \ref{Algo:Trisect}) with parameters \linebreak $Y^{(3k-1)}$, $Y^{(3k)}$, $E_t^{k-1}$, $Q_t^{k-1}$ and $G_{t-1}^k$, computed in line \ref{Line:TreeTrisect} of Algorithm \ref{Algo:Tree}
	\item $E_s^k\coloneqq O(E_t^{k-1})$, $v_s^k\coloneqq \mathbbm{1}\{\lambda_1|E_{s}^k|>4\rho^2/h^2\}$,
	\item $E_{s+1}^k\coloneqq P(E_t^{k-1})$, $v_{s+1}^k\coloneqq 0$,
	\item $E_{s+2}^k\coloneqq I(E_t^{k-1})$, $v_{s+2}^k\coloneqq \mathbbm{1}\{\lambda_1|E_{s+2}^k|>4\rho^2/h^2\}$,
	\item $G_t^k$; update of $G_{t-1}^k$; every time we run \texttt{ScanAndUpdate} (Algorithm~\ref{algo:scanupdate}) in line \ref{line:scanandupdate}, the routine \texttt{UpdateGraph} (Algorithm~\ref{Algo:Graph}) with parameters $G_{t-1}^k$, $E_t^{k-1}$, $Q_t^{k-1}$, $Y^{(3k)}$ leads to a new directed graph.
 \item $\mathcal{Q}_t^{k-1}$; defined as in \eqref{Eq:DefLeftRight}, collection of subsets $Q'\subseteq Q_t^{k-1} $ on which we compare experts in $E_t^{k-1}$ during \texttt{ScanAndUpdate}.
\end{itemize}
At the end of the $r_{k-1}$ iterations, one obtains $\mathcal{E}^k=(E_1^k,E_2^k,\dots,E_{r_k}^k)$, $v^k=(v_1^k,v_2^k,\dots, v_{r_k}^k)$ and $G^k=G_{r_{k-1}}^k$.

The algorithm goes through every $E_t^{k-1}\in\mathcal{E}^k$. If $v_t^{k-1}=0$, we say that $E_t^{k-1}$ is a \emph{passive set}. We account this by setting $w_t^{k-1}=0$. Passive sets are carried over to $\mathcal{E}^k$, and we also set the corresponding entry of $v^k$ to $0$.
If $v_t^{k-1}=1$, we first compute $Q_t^{k-1}$. If this set is small, we will change the status of $E_t^{k-1}$ to be a passive set. So again, $w_t^k=0$, $E_t^{k-1}$ is carried over to $\mathcal{E}^k$ and the corresponding entry of $v^k$ is set to $0$.
If $v_t^{k-1}=1$ and $Q_t^{k-1}$ is large enough, we call $E_t^{k-1}$ an \emph{active set}, so $w_t^k=1$. We will apply \texttt{GraphTrisect} and obtain a trisection of $E_t^{k-1}$ of the form $(O,P,I)$. This trisection now replaces $E_t^{k-1}$ in the new $\mathcal{E}^k$. $O$ and $I$ will be considered as passive, if they are too small, which is taken into account in the definition of the corresponding entries of $v^k$. $P$ is always considered a passive set, so the corresponding entry of $v^k$ is set to $0$. 
We refer to Figure \ref{Fig:Tree} for a graphical illustration.

\begin{figure}[H]
	\begin{center}
	\includegraphics{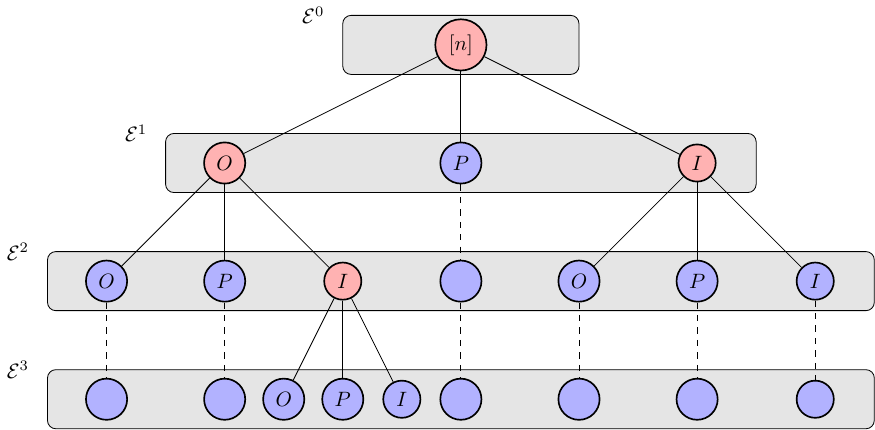}
	\end{center}
	\caption{Illustration of a sorting tree with $K=3$ iterations. The red nodes correspond to active sets, on which we apply Algorithm \texttt{GraphTrisect}, the trisection. The blue nodes correspond to passive sets, which are carried over unchanged into all further levels of the tree.}\label{Fig:Tree}
\end{figure}

In Section~\ref{Sec:AnalysisTree}, we will prove the following statement which is mainly a consequence of the results above for generic input arguments. In the next theorem, the first point states that each of the envelopes contains the relevant set of questions and that their sizes is bounded. The second and the third points state that each of the trisections satisfies the properties outlined in the previous subsections. Finally, the two last points emphasize that so-called active sets are perfectly ordered with each other.  

\begin{theorem}\label{Thm:Properties}
	There exist events $\xi^1\supset \xi^2\supset\dots\supset \xi^K$ with
	\begin{align*}
		\mathbb{P}(\xi^k)\geq 1- \frac{k\delta}{2\lceil \log_2(n\vee d)\rceil}\hspace{.5cm}\text{for }\hspace{.5cm}k=1,2,\dots,K\enspace ,
	\end{align*}
	such that on $\xi^k$, the following holds true after the $k^{\mathrm{th}}$ iteration step of Algorithm \ref{Algo:Tree}:
		\begin{enumerate}[label = \arabic*., ref=\ref{Thm:Properties}.\arabic*]
		\item\label{ThmItem:Envelope} For $t\in\{1,2,\dots,r_{k-1}\}$ with $v_t^{k-1}=1$ we have
		\begin{align*}
			Q_t^{k-1}\supseteq Q^*(E_t^{k-1})
		\end{align*}
		and moreover it holds that
		\begin{align*}
			\sum_{t\geq 1,\ w_t^{k-1}=1}|Q_t^{k-1}|\leq 3d\enspace .
		\end{align*}
		\item\label{ThmItem:Active} Let $t\in\{1,2,\dots,r_{k-1}\}$, $w_t^{k-1}=1$ and $E_s^k=O(E_t^{k-1})$, $E_{s+1}^k=P(E_t^{k-1})$, $E_{s+2}^k=I(E_t^{k-1})$.
		\begin{enumerate}[label=2.\alph*.,ref=\ref{ThmItem:Active}.\alph*]
			\item\label{ThmItem:TrisectIsPartition} $(E_s^k,E_{s+1}^k,E_{s+2}^k)$ is a partition of $E_t^{k-1}$,
			\item\label{ThmItem:MedianSeparates} If $\overline{i}_t^{k-1}$ is the median of $E_t^{k-1}$ in the sense of Equation~\eqref{Eq:DefMedian}, it holds $\overline{i}_t^{k-1}\in E_{s+1}^k$ and 
			\begin{align*}
				\pi(i_O)<\pi(\overline{i}_t^{k-1})<\pi(i_I)\hspace{.5cm}\forall i_O\in E_s^k,\ i_I\in E_{s+2}^k\enspace .
			\end{align*}
			\item\label{ThmItem:SetsShrink} $|E_s^k|\vee |E_{s+2}^k|\leq 2^{-k}n$.
		\end{enumerate}
		\item\label{ThmItem:Graph} $G_{i,i'}^k=1\Rightarrow\pi(i)<\pi(i')\ \forall i,i'\in[n]$.
        \item\label{ThmItem:Order} For $i\in E_s^k$ and $i'\in  E_{s'}^k$ with $v_s^k=v_{s'}^k=1$ holds
		\begin{align*}
			s<s'\ \Rightarrow \ \pi(i)<\pi(i')\enspace  ,
		\end{align*}
		\item\label{ThmItem:Partition} $\mathcal{E}^k$ is a partition of $[n]$.
	\end{enumerate}
Moreover, with a probability of at least $1- \delta/2$, Algorithm \ref{Algo:Tree} terminates after at most $\lceil \log_2(n)\rceil$ iterations.
\end{theorem}

While the above theorem states that the partition is in some way coherent with the true permutation $\pi$ this will not be sufficient to control the error of the estimated permutation $\hat{\pi}$. For that we need to define auxiliary objects. First, for any $E_t^{k-1}$, we define $\tilde{O}(E_t^{k-1})$, $\tilde{I}(E_t^{k-1})$, and $\tilde{P}(E_t^{k-1})$ as some completions of the sets $O(E_t^{k-1})$, $I(E_t^{k-1})$,  and $P(E_t^{k-1})$.
\begin{align}
	\tilde{O}(E_t^{k-1})\coloneqq \left\{i\in[n]:\ \min_{i'\in O(E_t^{k-1})}\pi(i')\leq \pi(i)\leq \max_{i'\in O(E_t^{k-1})}\pi(i')\right\}\enspace ,\notag\\
	\tilde{P}(E_t^{k-1})\coloneqq \left\{i\in[n]:\ \min_{i'\in P(E_t^{k-1})}\pi(i')\leq \pi(i)\leq \max_{i'\in P(E_t^{k-1})}\pi(i')\right\}\enspace ,\label{Eq:TildeSetsTrisection}\\
	\tilde{I}(E_t^{k-1})\coloneqq \left\{i\in[n]:\ \min_{i'\in I(E_t^{k-1})}\pi(i')\leq \pi(i)\leq \max_{i'\in I(E_t^{k-1})}\pi(i')\right\}\enspace .
\end{align}
In other words $\tilde{O}(E_t^{k-1})\supset O(E_t^{k-1})$ contains all the experts in $[n]$ whose position $\pi(i)$ is sandwiched by two experts in $O(E_t^{k-1})$. 

Recall also that we defined median $\overline{i}_t^{k-1}\in E_t^{k-1}$ in Equation~\eqref{Eq:DefMedian} by the property	     
\[\left|\left\{i'\in E_t^{k-1}:\ \pi(i')\leq \pi(\overline{ i}_t^{k-1})\right\}\right|=\lceil |E_t^{k-1}|/2\rceil \ . \]
Given a subset $Q'\subset [d]$ of question, we now define its empirical counterpart $\mmed{t}{k-1}{Q'}$ by the property 
\begin{align*}
	\left|\left\{i\in E_t^{k-1}:\ \frac{1}{|Q'|}\sum_{j\in Q'}Y^{(3k)}_{ij}> \frac{1}{|Q'|}\sum_{j\in Q'}Y^{(3k)}_{\mmed{t}{k-1}{Q'}j}\right\}\right|< |E_t^{k-1}|/2\enspace ,\\ 
	\left|\left\{i\in E_t^{k-1}:\ \frac{1}{|Q'|}\sum_{j\in Q'}Y^{(3k)}_{ij}< \frac{1}{|Q'|}\sum_{j\in Q'}Y^{(3k)}_{\mmed{t}{k-1}{Q'}j}\right\}\right|\leq |E_t^{k-1}|/2\enspace .
\end{align*}
In other words, $\mmed{t}{k-1}{Q'}$ corresponds to the median of $E_t^{k-1}$ according to the empirical average of the data on the set $Q'$ of questions. 

Equipped with this new notation and recalling the definition of $\mathcal{Q}_t^{k-1}$ above, we define
\begin{equation}
	\overline{P}(E_t^{k-1})=\left\{i\in E_t^{k-1}:\ \frac{1}{|Q'|}\left|\sum_{j\in Q'} Y_{ij}^{(3k)}-Y_{\mmed{t}{k-1}{Q'}j}^{(3k)}\right|\leq 8\rho\sqrt{\lambda_1/|Q'|}\hspace{.5cm }\forall Q'\in \mathcal{Q}_t^{k-1}\right\}\enspace .
\end{equation}
In fact, under an event of large probability, the set $\overline{P}(E_t^{k-1})$ contains $P(E_t^{k-1})$; this set will be important to control the permutation error in the next subsection. We would like to point out a small abuse of notation in the definition of $\overline{P}(E_t^{k-1})$, which not only depends on the set $E_t^{k-1}$, but also on the indices $t$ and $k-1$, since we consider sets $Q'\in\mathcal{Q}_t^{k-1}$.

With this notation, the following result, whose proof is an immediate consequence of that of Theorem~\ref{Thm:Properties}. 
\begin{corollary}\label{Cor:PropertiesTree}
	On the event $\xi^K$, it holds that
	\begin{enumerate}[label = \arabic*., ref=\ref{Cor:PropertiesTree}.\arabic*]
		\item\label{CorItem:InO} $\Big(\exists Q'\in\mathcal{Q}_t^{k-1}$ s.t. $\frac{1}{|Q'|}\sum_{j\in Q'}Y^{(3k)}_{ij}-Y^{(3k)}_{\mmed{t}{k-1}{Q'}j}>2\rho\sqrt{\lambda_1/|Q'|}\Big)$ implies $i\in O(E_t^{k-1})$,
		\item\label{CorItem:InI} $\Big(\exists Q'\in\mathcal{Q}_t^{k-1}$ s.t. $\frac{1}{|Q'|}\sum_{j\in Q'}Y^{(3k)}_{\mmed{t}{k-1}{Q'}j}-Y^{(3k)}_{ij}>2\rho\sqrt{\lambda_1/|Q'|}\Big)$ implies $i\in I(E_t^{k-1})$,
		\item\label{CorItem:InP} $i\in P(E_t^{k-1})$ implies $\frac{1}{|Q'|}\left|\sum_{j\in Q'}Y^{(3k)}_{ij}-Y^{(3k)}_{\mmed{t}{k-1}{Q'}j}\right|\leq2\rho\sqrt{\lambda_1/|Q'|}\ \forall Q'\in\mathcal{Q}_t^{k-1}$,
		\item\label{CorItem:TrisectionSubsets} $O(E_t^{k-1})\subseteq \tilde{O}(E_t^{k-1})$, $P(E_t^{k-1})\subseteq \tilde{P}(E_t^{k-1})$ and $I(E_t^{k-1})\subseteq \tilde{I}(E_t^{k-1})$,
		\item\label{CorItem:PTildeSubset} $E_t^{k-1}\cap \tilde{P}(E_t^{k-1})\subseteq \overline{P}(E_t^{k-1})$ and
		\item\label{CorItem:ErrorBound} $\tilde{\mathcal{R}}_{p,h,\overline{ P}(E_t^{k-1}),[d]}\leq 3744\rho(|E_t^{k-1}|\vee |Q_t^{k-1}|)/\lambda_1h^2.$
	\end{enumerate}
\end{corollary}
The three first points are consequence of the definition of the procedure. The fourth point states the $\tilde{O}(E_t^{k-1})$ is a superset of $O(E_t^{k-1})$, the fifth point $\overline{P}(E_t^{k-1})$ is large. Finally, the  last point controls the modified error $\tilde{\mathcal{R}}$ defined in~\eqref{eq:definition:loss:Rtilde} for the larger set $\overline{ P}(E_t^{k-1})$ --this bound corresponds to Lemma~\ref{Lemma:Blocks}.

\subsection{Analysis of the permutation estimator}\label{sec:proof:analysis:permutation}

We have seen in Theorem~\ref{Thm:Properties} that Algorithm \ref{Algo:Tree} terminates with a probability of at least $1-\delta/2$ after $K\leq \lceil \log_2(n)\rceil$ iterations and returns a partition $\mathcal{E}^K=(E_1^K,E_2^K,\dots,E_{r_k}^K)$ of $[n]$. Recall that we define the estimator $\hat{\pi}$ as permutation drawn uniformly among all  permutations $\pi'\in\mathcal{S}_n$ with the property 
\begin{align*}
	i\in E_s^K\ \Leftrightarrow \ \sum_{s'<s}|E_{s'}^K|<\pi'(i)\leq \sum_{s'\leq s}|E_{s'}^K|\enspace .
\end{align*}
Throughout this subsection write $\hat \pi$ for $\hat \pi_S$ to simplify the notation. 
On $\xi^K$, the property that $\hat \pi$ is coherent with partition $\mathcal{E}^K$ implies that it is also coherent with all less refined partitions $\mathcal{E}^k$, $k \leq K$ as stated in the following lemma.
\begin{lemma}\label{Lemma:PermutationInInterval}
	On $\xi^K$, it holds for all $k\in\{0,1,\dots,K\}$ that
	\begin{align*}
		i\in E_s^k \  \Leftrightarrow \ \sum_{s'<s}|E_{s'}^k|<\hat{\pi}(i)\leq \sum_{s'\leq s}|E_{s'}^k|\enspace .
	\end{align*}
\end{lemma}
The proof can be found in Section~\ref{Sec:PermutationError}.

\subsubsection{Simple error bound in an idealized situation}
At first, consider the following idealized situation. If we had at each step, for sets $E_t^{k-1}$ with $w_t^{k-1}=1$, a perfect separation of the form $(O(E_t^{k-1}),\{\overline{i}_t^{k-1}\},I(E_t^{k-1}))$, this would lead to a final partition $\mathcal{E}^K$ with $\pi(i)<\pi(i')$ if $i\in E_{s}^K$, $i'\in E_{s'}^K$ for $s<s'$. If we apply $\tau \coloneqq \pi^{-1}\circ \hat \pi$ on some $i\in E_s^K$, the definition of $\hat\pi$ therefore implies $\hat{i}=\pi^{-1}(\hat\pi(i))\in E_s^K$. Note that $M_{i\cdot}$ and $M_{\hat{i}\cdot}$ differ on at most $|Q^*(E_s^K)|$ entries. All $E_s^K$ are passive sets, meaning they are either of the form $\{\overline i_t^{k-1}\}$ (no error arises from exchanging an element with itself); or $E_s^K=O(E_t^{k-1})$ or $E_s^K=I(E_t^{k-1})$, where either $|E_s^K|\leq 4\rho^2/\lambda_1 h^2$ or $|Q^*(E_s^K)|\leq 4\rho^2/\lambda_1 h^2$. It is easy to see that for each $E_s^K$ we have
\begin{align*}
        \sum_{i\in E_s^K} \left|\{ j\in[d]: M_{ij}\geq p+h,\ M_{\hat i j}\leq p-h\}\right|+\left|\{ j\in[d]: M_{ij}\leq p-h,\ M_{\hat i j}\geq p+h\}\right|\\
    \leq |E_s^K|\cdot |Q^*(E_s^K)|\leq \frac{4\rho^2}{\lambda_1h^2}\left(|E_s^K|\vee |Q^*(E_s^K)|\right)\enspace 
\end{align*}
and in total this sums up to an error 
\begin{align}
    \mathcal{R}_{p,h}(\hat\pi)\leq \sum_{s\geq 1}|E_s^K|\cdot |Q^*(E_s^K)|\lesssim \frac{\rho^2}{\lambda_1h^2}(n\vee d)\enspace .\label{Eq:IdealizedBound}
\end{align}

\subsubsection{Valid control of the  error}

In general, $P(E_t^{k-1})$ does not simply contain the median $\overline i_t^{k-1}$. We indicated earlier, that an overlap of $P(E_t^{k-1})$ with $O(E_t^{k-1})$ or $I(E_t^{k-1})$ might be a source of error. In fact, the arguments for the error bounds we just stated cannot be formalized if $i\in O(E_t^{k-1})$ but $\hat i= \pi^{-1}(\hat\pi(i))\notin\tilde O(E_t^{k-1})$ or $i\in I(E_t^{k-1})$ but $\hat i\notin\tilde I(E_t^{k-1})$. The following two statements are quite technical, but take up exactly this idea. Morally they lead to a proof of Theorem~\ref{Thm:ErrorBoundRows}, where every error from confusing $i$ with $\hat i$, that we cannot handle as in \eqref{Eq:IdealizedBound}, can be dealt with by controlling the sets $\overline{P}(E_u^l)$, with $l\leq K$, $u\in\{1,\dots,r_l\}$. We postpone the proofs of those lemmas to Section~\ref{Sec:PermutationError}.

\begin{lemma}\label{Lemma:TauNotInO}
	Consider $k\leq K$, $s\in\{1,2,\dots,r_k\}$ such that $ E_s^k$ is of the form $E_s^k=O(E_t^{k-1})$ or $E_s^k=I(E_t^{k-1})$ for some $t\in \{1,2,\dots,r_{k-1}\}$. Consider $i\in E_s^k$ such that $\hat{i}\coloneqq \pi^{-1}(\hat\pi(i))\notin \tilde{E}_s^k$. Then on $\xi^K$, there exist $l<k$ and $u\in\{1,2,\dots,r_l\}$ such that $i,\hat{i}\in \overline{P}(E_u^l)$.
\end{lemma}

In other words, these lemmas entails that an error in some  $E_s^k$ corresponds to elements that belong to a super set $\overline{P}(E_u^l)$. This suggests that the size of $\overline{P}(E_u^l)$ is important to quantify the permutation loss.

\begin{lemma}\label{Lemma:TauInP}
	Consider $k,s$ such that $E_s^k$ is of the form $P(E_t^{k-1})$ for some $t\in\{1,2,\dots,r_{k-1}\}$ with $w_t^{k-1}=1$. Consider $i\in E_s^k$ and write $\hat{i}\coloneqq \pi^{-1}(\hat{\pi}(i))$. On the event $\xi^K$, there exist $l<k$ and $u\in\{1,2,\dots,r_l\}$ such that $i,\hat{i}\in \overline{P}(E_u^l)$.
\end{lemma}
Similarly, this lemma states that, if $i$ and $\hat{i}$ are mixed and if $i$ belongs to some $P(E_t^{k-1})$, then both $i$ and $\hat{i}$ belong to some $\overline{P}(E_u^l)$.

We now have the main ingredients to analyze the error
\begin{align*}
	\mathcal{R}_{p,h}(\hat \pi)=& \left|\left\{(i,j)\in [n]\times [d]:\ M_{\pi^{-1}(i)\eta^{-1}(j)}\leq p-h,\ M_{\hat\pi^{-1}(i)\eta^{-1}(j)}\geq p+h\right\}\right|\\
	&+\left|\left\{(i,j)\in [n]\times [d]:\ M_{\pi^{-1}(i)\eta^{-1}(j)}\geq p+h,\ M_{\hat\pi^{-1}(i)\eta^{-1}(j)}\leq p-h\right\}\right|\enspace .
\end{align*}
	Also define for $1\leq k<K$ the sets
	\begin{align*}
		\mathcal{O}^k\coloneqq \{E_t^k:\ \exists t'\in\{1,2,\dots,r_{k-1}\} \text{ with } E_t^k= O(E_{t'}^{k-1})\text{ and }w_t^k=0\}\enspace ,\\
				\mathcal{I}^k\coloneqq \{E_t^k:\ \exists t'\in\{1,2,\dots,r_{k-1}\} \text{ with } E_t^k= I(E_{t'}^{k-1})\text{ and }w_t^k=0\}\enspace .
	\end{align*}
	Then for $s\in\{1,2,\dots,r_K\}$, each set $E_s^K$ is either in $\bigcup_{k=1}^K\mathcal{O}^k\cup \mathcal{I}^k$ or has the form $P(E_{t'}^{k-1})$ for some $k\in\{1,2,\dots,K\}$ and $t'\in \{1,2,\dots,r_{k-1}\}$ with $w_{t'}^{k-1}=1$. By combining Theorem~\ref{Thm:Properties} with Lemmas~\ref{Lemma:TauNotInO} and~\ref{Lemma:TauInP}, we arrive, with some work, at the following control for the loss.

 \begin{lemma}\label{lem:upper_boud_error_permutation}
     On the event $\xi^K$, it holds that 
	\begin{align}
		\mathcal{R}_{p,h}(\hat{\pi})\leq 4\frac{\rho^2}{\lambda_1 h^2}\sum_{k=1}^K\sum_{\substack{t=1\\ E_t^k\in \mathcal{O}^k\cup\mathcal{I}^k}}^{r_k}|E_t^k|\vee |Q^*(E_t^k)|+2 \sum_{l=1}^{K}\sum_{\substack{u=1\\ w_u^{l-1}=1}}^{r_{l-1}} \tilde{\mathcal R}_{p,h,\overline{P}(E_u^{l-1}),[d]} \label{Eq:UpperBoundR2}\ .
	\end{align}
 \end{lemma}
The above bound contains two terms. The first one accounts for the size of groups and corresponding questions of the form $O(E_t^{k})$ or $I(E_t^{k})$ that are not to be cut anymore. The second accounts for the error $\mathcal{R}_{p,h,\overline{P},h}$ induced by the set of the form $\overline{P}=\overline{P}(E_u^{l-1})$, such an error being introduced and bounded in Section~\ref{sec:proof:trisection}.
We respectively call $(I)$ and $(II)$ the two terms in the rhs of~\eqref{Eq:UpperBoundR2}. We first focus on $(I)$. First, we claim that all $E_t^k$ in $\mathcal{O}^k\cup \mathcal{I}^k$ with $k=1,\ldots, K$ are disjoint and that all corresponding $|Q^*(E_t^k)|$ are also disjoint. It follows from this claim that 
\begin{align}
		(I) \leq 4\frac{\rho^2}{\lambda_1 h^2} (n+ d)\enspace .\label{Eq:UpperBoundOSet}
\end{align}
Let us show the above claim. Consider any two such sets $E_{t_1}^{k_1}$ and $E_{t_2}^{k_2}$. Since these sets are not cut anymore in the sorting tree, there exists $l< k_1\wedge k_2$ and $u\in [r_l]$ with  $w_u^l=1$ such that 
	\begin{align*}
		E_{t_1}^{k_1}\subseteq O(E_u^l),\ E_{t_2}^{k_2}\subseteq I(E_u^l) \hspace{.5cm}\text{or}\hspace{.5cm}E_{t_2}^{k_2}\subseteq O(E_u^l), \ E_{t_1}^{k_1}\subseteq I(E_u^l)\enspace .
	\end{align*}
In fact, $E_u^l$ corresponds to the closest common ancestor of $E_{t_1}^{k_1}$ and $E_{t_2}^{k_2}$ in the sorting tree. This proves that $E_{t_1}^{k_1}$ and $E_{t_2}^{k_2}$ are disjoint. This shows that $E_{t_1}^{k_1}$ and $E_{t_2}^{k_2}$ are disjoint and we can assume w.l.o.g. $E_{t_1}^{k_1}\subseteq O(E_u^l),\ E_{t_2}^{k_2}\subseteq I(E_u^l)$, so in particular we have by Theorem~\ref{ThmItem:MedianSeparates} that 
$\pi(i_1)<\pi(i_2)$ for all  $i_1 \in E_{t_1}^{k_1}$ and $i_2\in E_{t_2}^{k_2}$.	By the bi-isotonicity, it therefore holds for all $j_1\in Q^*(E_{t_1}^{k_1})$, $j_2\in Q^*(E_{t_2}^{k_2})$
	\begin{align*}
		\min_{i\in E_{t_1}^{k_1}}M_{ij_1}\leq p-h<p+h\leq \max_{i\in E_{t_2}^{k_2}}M_{ij_2}\leq \min_{i\in E_{t_1}^{k_1}}M_{ij_2}\enspace ,
	\end{align*}
	which implies $\eta(j_1)>\eta(j_2)$. Therefore, also $Q^*(E_{t_1}^{k_1})$ and $Q^*(E_{t_2}^{k_2})$ are disjoint. We have proved the claim.

Let us turn to the second term (II) in the rhs of~\eqref{Eq:UpperBoundR2}. It follows from Corollary~\ref{CorItem:ErrorBound} that 
\[
2\tilde{\mathcal R}_{p,h,\overline{P}(E_u^{l-1}),[d]}
 \leq \frac{7488\rho}{\lambda_1 h^2}(|E_u^{l-1}|\vee |Q_u^{l-1}|)\enspace .
\]
Then, for a fixed $l$, we know that the $E_u^{l-1}$ are disjoint by construction so that $\sum_u|E_u^{l-1}|\leq n$. Besides, Theorem~\ref{ThmItem:Envelope} implies  that $\sum_u|Q_u^{l-1}|\leq 3d$. Overall, we arrive at 
\begin{align}
(II) \leq 44928 K \rho(n\vee d)/\lambda_1h^2\enspace . \label{Eq:UpperBoundPSet}
\end{align}
Hence, gathering 
\eqref{Eq:UpperBoundOSet} and \eqref{Eq:UpperBoundPSet} in \eqref{Eq:UpperBoundR2}, we conclude that, on $\xi^K$, we have 
			\begin{align}\label{eq:loss_pi_hat_S}
				\mathcal{R}_{p,h}(\hat{\pi})\leq(44928K\rho+8\rho^2)\frac{n\vee d}{\lambda_1 h^2}\enspace .
			\end{align}
We are almost done with the proof of Theorem~\ref{Thm:ErrorBoundRows}. We only need to recall that  $\lambda_1=1-e^{-\lambda_0^-}$ with $\lambda_0^-=\frac{\lambda_0}{3\lceil \log_2(nd)\rceil}$. So consider the function $g(x)=1-e^{-x}$ for $x>0$. $g$ is monotonously increasing and we have $\lambda_1=g\left(\frac{\lambda_0}{3\lceil\log_2(nd)\rceil}\right)$. One can show that $\lceil \log_2(nd)\rceil \leq 2\log(nd)$ and that for $x\in(0,1]$ it holds $g(x)\geq x/e$. Since we assumed $\lambda_0\in(0,\log(nd)]$, we obtain $\lambda_1\geq \frac{\lambda_0}{6e\log(nd)}$ which, together with~\eqref{eq:loss_pi_hat_S} conclude the proof of Theorem~\ref{Thm:ErrorBoundRows}.

\section{Proofs of the intermediary results}

\subsection{Proof of Lemma~\ref{Lemma:LossBound}}\label{sec:proof:lemma:reduction}

    Consider $(i,j)\in[n]\times [d]$ such that
    \begin{align*}
         M_{\pi(i)\eta(j)}\leq p-2h \hspace{.5cm}\text{and}\hspace{.5cm}M_{\hat\pi(i)\hat\eta(j)}\geq p+2h \enspace.
    \end{align*}
    Then it holds either that
    \begin{align*}
         M_{\pi(i)\eta(j)}\leq p-2h,\ M_{\hat\pi(i)\eta(j)}\geq p \hspace{.5cm}\text{or}\hspace{.5cm}M_{\hat \pi(i)\eta(j)}\leq p, \ M_{\hat\pi(i)\hat\eta(j)}\geq p+2h \enspace .
    \end{align*}
    Since $\eta$ and $\hat{\pi}$ are both permutations and hence bijective, this proves
    \begin{align*}
        |\{(i,j)\in[n]\times [d]:\ M_{\pi(i)\eta(j)}\leq p-2h,\ M_{\hat\pi(i)\hat\eta(j)}\geq p+2h\}|\leq \mathcal{R}_{p-h,h}(\hat{\pi})+\mathcal{C}_{p+h,h}(\hat{\eta})\enspace .
    \end{align*}
    In the same way, one can prove
        \begin{align*}
        |\{(i,j)\in[n]\times [d]:\ M_{\pi(i)\eta(j)}\leq p-2h,\ M_{\hat\pi(i)\hat\eta(j)}\geq p+2h\}|\leq \mathcal{R}_{p+h,h}(\hat{\pi})+\mathcal{C}_{p-h,h}(\hat{\eta}),\\
        |\{(i,j)\in[n]\times [d]:\ M_{\pi(i)\eta(j)}\geq p+2h,\ M_{\hat\pi(i)\hat\eta(j)}\leq p-2h\}|\leq \mathcal{R}_{p-h,h}(\hat{\pi})+\mathcal{C}_{p+h,h}(\hat{\eta})\enspace ,\\
        \intertext{and}
        |\{(i,j)\in[n]\times [d]:\ M_{\pi(i)\eta(j)}\geq p+2h,\ M_{\hat\pi(i)\hat\eta(j)}\leq p-2h\}|\leq \mathcal{R}_{p+h,h}(\hat{\pi})+\mathcal{C}_{p-h,h}(\hat{\eta})\enspace .
    \end{align*}
    This concludes the first part of the statement.

    For the lower bound, we will only prove
    \begin{align*}
        \mathcal{L}_{p,h}(\hat{\pi},\hat{\eta})\geq \mathcal{R}_{p,h}(\hat{\pi})\enspace.
    \end{align*}
    Note that we can write 
    \begin{align*}
         \mathcal{L}_{p,h}(\hat{\pi},\hat{\eta})&=\sum_{i=1}^n |\{j\in[d]:\ M_{\pi^{-1}(i)\eta^{-1}(j)}\leq p-h,\ M_{\hat{\pi}^{-1}(i)\hat{\eta}^{-1}(j)}\geq p+h\}|\\
         &\phantom{=\sum_{i=1}^n}+|\{j\in[d]:\ M_{\pi^{-1}(i)\eta^{-1}(j)}\geq p+h,\ M_{\hat{\pi}^{-1}(i)\hat{\eta}^{-1}(j)}\leq p-h\}|\\ \intertext{and}
         \mathcal{R}_{p,h}(\hat{\pi})&=\sum_{i=1}^n |\{j\in[d]:\ M_{\pi^{-1}(i)\eta^{-1}(j)}\leq p-h,\ M_{\hat{\pi}^{-1}(i){\eta}^{-1}(j)}\geq p+h\}|\\
         &\phantom{=\sum_{i=1}^n}+|\{j\in[d]:\ M_{\pi^{-1}(i)\eta^{-1}(j)}\geq p+h,\ M_{\hat{\pi}^{-1}(i){\eta}^{-1}(j)}\leq p-h\}|\enspace .
    \end{align*}
    We will bound the respective summands. W.l.o.g., consider $i\in[n]$ with $\pi(\hat\pi^{-1}(i))\geq i$. This implies $M_{\pi^{-1}(i)\eta^{-1}(j)}\geq M_{\hat\pi^{-1}(i)\eta^{-1}(j)}$ for all $j\in[d]$ by the bi-isotonicity assumption. To prove our claim, it is therefore sufficient to show
    \begin{align*}
        |\{j\in[d]:\ M_{\pi^{-1}(i)\eta^{-1}(j)}\leq p-h,\ M_{\hat{\pi}^{-1}(i)\hat{\eta}^{-1}(j)}\geq p+h\}|\\
        \geq |\{j\in[d]:\ M_{\pi^{-1}(i)\eta^{-1}(j)}\leq p-h,\ M_{\hat{\pi}^{-1}(i){\eta}^{-1}(j)}\geq p+h\}|\enspace .
    \end{align*}
    To this end, define 
    \begin{align*}
        j_l\coloneqq \min\{j\in[d]:\ M_{\pi^{-1}(i)\eta^{-1}(j)}\leq p-h\},\\
        j_r\coloneqq \max\{j\in[d]:\ M_{\hat\pi^{-1}(i)\eta^{-1}(j)}\geq p+h\}\enspace .
    \end{align*}
    It holds
    \begin{align*}
        &\hphantom{==}|\{j\in[d]:\ M_{\pi^{-1}(i)\eta^{-1}(j)}\leq p-h,\ M_{\hat{\pi}^{-1}(i){\eta}^{-1}(j)}\geq p+h\}|\\
        &=|\{j\in[d]:\ j_l\leq j\leq j_r\}|\\
      & =|\{j\in[d]:\ j_l\leq j\leq j_r,\ \eta(\hat\eta^{-1}(j))\leq j_r\}|+|\{j\in[d]:\ j_l\leq j\leq j_r,\ \eta(\hat\eta^{-1}(j))> j_r\}|\\
      &\leq |\{j\in[d]:\ j_l\leq j\leq j_r,\ \eta(\hat\eta^{-1}(j))\leq j_r\}|+|\{j\in[d]:\  j\leq j_r,\ \eta(\hat\eta^{-1}(j))> j_r\}|\\
      &=|\{j\in[d]:\ j_l\leq j\leq j_r,\ \eta(\hat\eta^{-1}(j))\leq j_r\}|+|\{j\in[d]:\  j> j_r,\ \eta(\hat\eta^{-1}(j))\leq j_r\}|\\
      &=|\{j\in [d]: \ j_l\leq j,\ \eta(\hat\eta^{-1}(j))\leq j_r\}|\\
      &= |\{j\in[d]:\ M_{\pi^{-1}(i)\eta^{-1}(j)}\leq p-h,\ M_{\hat\pi^{-1}(i)\hat\eta^{-1}(j)}\geq p+h\}|\enspace ,
    \end{align*}
    where we used, that $\eta\circ\hat{\eta}^{-1}$ is a permutation and that $ \eta(\hat\eta^{-1}(j))\leq j_r$ implies
    \begin{align*}
        M_{\hat\pi^{-1}\hat\eta^{-1}(j)}=M_{\hat\pi^{-1}(i)\eta^{-1}(\eta(\hat\eta^{-1}(j)))}\geq M_{\hat\pi^{-1}(i)\eta^{-1}(j_r)}\geq p+h\enspace .
    \end{align*}
    This concludes the proof.

\subsection{Proofs of Lemmas~\ref{Lemma:Graph} and~\ref{Lemma:PropertiesTrisection}}\label{Sec:AnalysisTrisection}

We will derive properties of the trisection under Property~\ref{Assum1} for two observations $\tilde{Y}^{(a)}$ and $\tilde{Y}^{(b)}$, some set $E\subset [n]$ of experts, some set $Q\subset [d]$ of questions and a directed graph $G\in\{0,1\}^{n\times n}$. First,  the average quantities 
\begin{align}\label{eq:def:average_over_experts}
	\ya{E}{j}\coloneqq\frac{1}{|E|}\sum_{i\in E} \tilde{Y}^{(a)}_{ij} \text{ and } \ma{E}{j}\coloneqq\frac{1}{|E|}\sum_{i\in E}M_{ij} \text{ for } j\in Q\enspace .
\end{align}
We will prove that with high probability
\begin{align}
		|\ya{E}{j}-\lambda_1\ma{E}{j}|\leq\rho\sqrt{\lambda_1/ |E|}\hspace{.5cm} \forall j\in Q\enspace. \label{Eq:ConcentrationCols}
\end{align}
The algorithm now uses the $\ya{E}{j}$ to construct sets of questions of the following form: for $j\in Q$ and $c\in \{2,3,\dots,\lceil \sqrt{\lambda_1|E|}/2\rho+1\rceil\}$, let
\begin{align}
	Q_j^l(c)\coloneqq \{j^{\prime} \in Q: 2\rho \sqrt{\lambda_1 /|E|}<\ya{E}{j'}-\ya{E}{j} \leq 2c\rho \sqrt{\lambda_1 /|E|}\}\enspace ,\notag\\
	Q_j^r(c)\coloneqq \{j^{\prime} \in Q: 2\rho \sqrt{\lambda_1/ |E|}< \ya{E}{j}-\ya{E}{j'} \leq 2c\rho \sqrt{\lambda_1/ |E|}\}\enspace ,\label{Eq:DefLeftRight} \\ \intertext{and}
	A_j\coloneqq \{j^{\prime}\in Q: |\ya{E}{j'}-\ya{E}{j}| \leq 2\rho \sqrt{\lambda_1/ |E|}\}\enspace .\notag
\end{align} 
The collection of potential sets of interest is defined as 
\begin{align*}
	\mathcal{Q}\coloneqq \Bigg( &\left\{Q_j^l(c): j\in Q,\ c=2,3,\dots,\lceil \sqrt{\lambda_1|E|}/2\rho+1\rceil\right\}\\
&\cup \left\{Q_j^r(c): j\in Q,\ c=2,3,\dots,\lceil \sqrt{\lambda_1|E|}/2\rho+1\rceil\right\}
     \cup \{A_j: j\in Q\}\cup \{Q\}\Bigg)\cap\{Q'\subseteq Q:\ |Q'|\geq \gamma\}\enspace .
\end{align*}
By Property~\ref{Assum1}, we have $\lambda_1 |Q|>4\rho^2/h^2$ by  implies $Q\in \mathcal{Q}$. Since the collection $\mathcal{Q}$ depends on the realization of $\tilde{Y}^{(a)}$, we use the independent observation $\tilde{Y}^{(b)}$ to define
\begin{align*}
	\yb{i}{Q'}\coloneqq \frac{1}{|Q'|}\sum_{j\in Q'}\tilde{Y}^{(b)}_{ij} \text{ and } \mb{i}{Q'}\coloneqq \frac{1}{|Q'|}\sum_{j\in Q'}M_{ij}\hspace{.5cm} \forall Q' \in \mathcal{Q},\ \forall i\in E\enspace .
\end{align*}
For our analysis, we will assume that we are on the event $\xi$ where the concentration inequality in Equation~\eqref{Eq:ConcentrationCols} holds and where also
\begin{align}
	|\yb{i}{Q'}-\lambda_1\mb{i}{Q'}|\leq \rho\sqrt{\lambda_1/|Q'|}\hspace{.5cm} \forall Q' \in \mathcal{Q},\ \forall i\in E\enspace , \label{Eq:ConcentrationRows}
\end{align}
namely
$$\xi =\left\{\mathrm{Equation~\eqref{Eq:ConcentrationCols}~ holds} \right\} \cap \left\{ \mathrm{Equation~\eqref{Eq:ConcentrationRows}~holds}\right\}.$$

$\xi$ is an event of high probability, as proven below.
\begin{lemma}\label{Lemma:GoodEvent}
	Under Property~\ref{Assum1} for $\tilde Y^{(a)}$, $\tilde Y^{(b)}$, $E$, $Q$ and $G$, it holds that $\mathbb{P}(\xi)\geq 1-\frac{|Q|\delta}{12\lceil\log_2(n\vee d)\rceil d}$.
\end{lemma}
\begin{proof}[Proof of Lemma~\ref{Lemma:GoodEvent}]
	Consider $Q'\in \mathcal{Q}$. Since $|Q'|\geq \gamma=2\log\left(24\lceil \log_2(n\vee d)\rceil nd(n\vee d)^{1/2}/\delta\right)/{\lambda_1e^2}$, Lemma~\ref{Lemma:ConcentrationSE} gives us for any $i\in E$, with probability of at least $1-\frac{\delta}{12\lceil \log_2(n\vee d)\rceil nd(n\vee d)^{1/2}}$
	\begin{align*}
		|\yb{i}{Q'}-\lambda_1\mb{i}{Q'}|&\leq  \sqrt{2(1\vee \sigma^2) e^2\log\left(24\lceil \log_2(n\vee d)\rceil nd(n\vee d)^{1/2}/\delta\right)\lambda_1/|Q'|}\\&\hphantom{<}+2(1\vee \sigma)\log\left(24\lceil \log_2(n\vee d)\rceil nd(n\vee d)^{1/2}/\delta\right)/|Q'|\\
		&\leq \rho\sqrt{\lambda_1 /|Q'|}\enspace ,
	\end{align*}
	since $\rho = (1\vee \sigma)e\sqrt{8\log\left(24\lceil \log_2(n \vee d)\rceil nd(n\vee d)^{1/2}/\delta\right)}$.

    Note that by Property~\ref{Assum1} we also know $\lambda_1|E|>4\rho^2/h^2$ which again implies that $|E|\geq \gamma$. For any $j\in Q$, this implies by Lemma~\ref{Lemma:ConcentrationSE} that with probability of at least $1-\frac{\delta}{12\lceil\log_2(n\vee d)\rceil nd(n\vee d)^{1/2}}$
    \begin{align*}
    		|\ya{E}{j}-\lambda_1\ma{E}{j}|&\leq  \sqrt{2(1\vee \sigma^2) e^2\log\left(24\lceil \log_2(n\vee d)\rceil nd(n\vee d)^{1/2}/\delta\right)\lambda_1/|E|}\\
      &\hphantom{<}+2(1\vee \sigma)\log\left(24\lceil \log_2(n\vee d)\rceil nd(n\vee d)^{1/2}/\delta\right)/|E|\\
		&\leq \rho\sqrt{\lambda_1 /|E|}\enspace .
	\end{align*}
	Recall that the collection    $\mathcal{Q}$ contains the set $Q$ and sets of the form $A_j$ and $Q_j^l(c)$, $Q_j^r(c)$ with $j\in Q$ and $c\in\{2,3,\dots,\lceil \sqrt{\lambda_1|E|}/2\rho+1\rceil\}$. Therefore, the cardinality of $\mathcal{Q}$ is bounded by 
    \begin{align*}
        |Q|\cdot(2\lceil\sqrt{\lambda_1|E|}/2\rho\rceil+1)+1 &\leq |Q|\left( \sqrt{\lambda_1|E|}/\rho+3\right)+1\\
        &\leq 3|Q|\sqrt{\lambda_1|E|}/\rho \\
        &\leq |Q||E|^{1/2}/2 \enspace ,
    \end{align*}
    where we used $1\leq 1/h<\sqrt{\lambda_1 |E|}/2 \rho$, $\lambda_1<1$ and $\rho\geq e\sqrt{8\log(24)}>6$. At the same time, we have again $|E|>4\rho^2/h^2\lambda_1>4$, so $|Q|\leq |Q||E|^{3/2}/2$. In the end, the total number of inequalities in \eqref{Eq:ConcentrationCols} and \eqref{Eq:ConcentrationRows} is bounded by 
    \begin{align*}
        |E|\cdot|\mathcal{Q}|+|Q|\leq |Q||E|^{3/2}\leq |Q|n(n\vee d)^{1/2}\enspace .
    \end{align*}
Consequently, using the union bound, we have
	\begin{align*}
		\mathbb{P}(\xi)\geq1-\frac{|Q|\delta}{12\lceil \log_2(n\vee d)\rceil d}\enspace ,
	\end{align*}
 which concludes the proof.
\end{proof}

\begin{proof}[Proof of Lemma~\ref{Lemma:Graph}]
	Consider $i,i'\in E$ with $\pi(i)\geq \pi(i')$. Under Property~\ref{Assum1}, we have $G_{ii'}\neq 1$. For every $Q'\in\mathcal{Q}$, the bi-isotonicity assumption gives us $\mb{i}{Q'}\leq \mb{i'}{Q'}$ and on $\xi$, in particular by \eqref{Eq:ConcentrationRows}, it holds that
	\begin{align*}
		\yb{i}{Q'}-\yb{i'}{Q'}\leq \yb{i}{Q'}-\lambda_1\mb{i}{Q'}+\lambda_1\mb{i'}{Q'}-\yb{i'}{Q'}\leq 2\rho\sqrt{\lambda_1/|Q'|}\enspace .
	\end{align*}
	So it remains $\tilde G_{ii'}=G_{ii'}\neq 1 $ also after applying Algorithm \ref{Algo:Trisect}.
\end{proof}

\begin{proof}[Proof of Lemma~\ref{Lemma:PropertiesTrisection}]
	Consider $i\in E$ with $\pi(\overline i)\leq \pi(i)$. Under Property~\ref{Assum1} we can apply Lemma~\ref{Lemma:Graph} and obtain on $\xi$ that
	\begin{align*}
		\lfloor|E|/2\rfloor =|\{i'\in E:\ \pi(\overline i)<\pi(i')\}|\geq |\{i'\in E:\ \pi( i)<\pi(i')\}| \geq |\{i'\in E: \ \tilde G_{ii'}=1\}|\enspace .
	\end{align*}
	By the definition in \eqref{Eq:DefTrisection}, this implies $i\notin O$. Conversely, this means that if $i_O\in O$, we have $\pi(i_O)< \pi(\overline i)$. The second part of \eqref{Eq:RelationTrisectionMedian} follows in the same manner.
\end{proof}
\subsection{Proof of Lemma~\ref{Lemma:Blocks}}\label{Sec:TrisectErrorBound}

We start with a few lemmas and notation.
Recall that we defined for each $Q'\in\mathcal{Q}$ the empirical median $\med{Q'}$ via the property
	\begin{align*}
	|\{i\in E:\ \yb{i}{Q'}> \yb{\med{Q'}}{Q'}\}|< |E|/2 \text{ and } |\{i\in E:\ \yb{i}{Q'}< \yb{\med{Q'}}{Q'}\}|\leq |E|/2
\end{align*}
as an empirical counterpart to $\overline i$ defined in \eqref{Eq:DefMedian}.

The following lemma is a direct consequence of the definition of our trisection scheme in Equation~\eqref{Eq:DefTrisection}.
\begin{lemma}\label{Lemma:TrisectionEmpiricalMedian}
	Under Property~\ref{Assum1} for $\tilde Y^{(a)}$, $\tilde Y^{(b)}$, $E$, $Q$ and $G$, on the event $\xi$, it holds for any $i\in E$:
	\begin{itemize}
		\item If there exists $Q'\in\mathcal{Q}$ s.t. $\yb{i}{Q'}-\yb{\med{Q'}}{Q'}>2\rho\sqrt{\lambda_1/|Q'|}$, then $i\in O$.
		\item If there exists $Q'\in\mathcal{Q}$ s.t. $\yb{\med{Q'}}{Q'}-\yb{i}{Q'}>2\rho\sqrt{\lambda_1/|Q'|}$, then $i\in I$.
	\end{itemize}
	So that:
	\begin{itemize}
		\item If $i\in P$, then $|\yb{i}{Q'}-\yb{\med{Q'}}{Q'}|\leq2\rho\sqrt{\lambda_1/|Q'|}\ \forall Q'\in\mathcal{Q}$.
	\end{itemize}
\end{lemma}
\begin{proof}[Proof of Lemma~\ref{Lemma:TrisectionEmpiricalMedian}]
	We will first prove the first claim. Assume that there exists $Q'\in \mathcal{Q}$ such that
	\begin{align*}
		\yb{i}{Q'}-\yb{\med{Q'}}{Q'}>2\rho\sqrt{\lambda_1/|Q'|}\enspace.
	\end{align*} 
	Note that by definition, the number of all $i'$ such that $\yb{i'}{Q'}\leq \yb{i}{\med{Q'}}$ is at least $|E|/2$. For each such $i'$ it holds
	\begin{align*}
		\yb{i}{Q'}-\yb{i'}{Q'}>2\rho\sqrt{\lambda_1/|Q'|}\enspace ,
	\end{align*}
	and consequently $\tilde G_{ii'}=1$. This proves $i\in O$. The second claim follows analogously and the third claim is just the contraposition of the first two statements.
\end{proof}

Let us introduce the sets $\tilde{O}$, $\tilde{P}$, and $\tilde{I}$ which are to be interpreted as completions of $O$, $P$, and $I$. 
\begin{align}
	\tilde{O}\coloneqq \left\{i\in [n]:\ \min_{i'\in O}\pi(i')\leq \pi(i)\leq \max_{i'\in O}\pi(i')\right\}\enspace ,\notag\quad 
	\tilde{P}\coloneqq \left\{i\in [n]:\ \min_{i'\in P}\pi(i')\leq \pi(i)\leq \max_{i'\in P}\pi(i')\right\}\enspace ,\notag\\
		\tilde{I}\coloneqq \left\{i\in [n]:\ \min_{i'\in I}\pi(i')\leq \pi(i)\leq \max_{i'\in I}\pi(i')\right\}\notag\enspace .
\end{align}
We note the following relationships between the defined sets.
\begin{lemma}\label{Cor:SubsetOfTilde}
	Under Property~\ref{Assum1} for $\tilde Y^{(a)}$, $\tilde Y^{(b)}$, $E$, $Q$ and $G$, on the event $\xi$, it holds that $O\subseteq \tilde{O}$, $P\subseteq \tilde{P}$ and $I\subseteq \tilde{I}$. Besides, we have $\tilde{P}\cap E\subseteq \overline{P}$.
\end{lemma}
\begin{proof}[Proof of Lemma~\ref{Cor:SubsetOfTilde}]
	The first statement (three first inclusions) follows trivially from the respective definitions. We therefore focus on the second statement. Consider $i\in \tilde{P}\cap E$ and let $i_{\max}\coloneqq \mathrm{argmax}_{i'\in P}\pi(i')$ and $i_{\min}\coloneqq \mathrm{argmin}_{i'\in P}\pi(i')$. By definition it holds that
	\begin{align*}
		\pi(i_{\min})\leq \pi(i)\leq \pi(i_{\max})\enspace .
	\end{align*}
	From Lemma~\ref{Lemma:TrisectionEmpiricalMedian} we know that
	\begin{align}
		\left|\yb{i_{\max}}{Q'}-\yb{\med{Q'}}{Q'}\right|\leq 2\rho\sqrt{\lambda_1/|Q'|}\hspace{.5cm}\forall Q'\in \mathcal{Q} \enspace , \label{Eq:DistimaxMedian}\\ \intertext{and}
		\left|\yb{i_{\min}}{Q'}-\yb{\med{Q'}}{Q'}\right|\leq 2\rho\sqrt{\lambda_1/|Q'|}\hspace{.5cm}\forall Q'\in \mathcal{Q}\enspace . \notag
	\end{align}
	Because of the bi-isotonicity assumption, we have
	\begin{align*}
		\mb{i_{\max}}{Q'}\leq \mb{i}{Q'}\leq \mb{i_{\min}}{Q'} \hspace{.5cm} \forall Q'\in \mathcal{Q}\enspace .
	\end{align*}
	Consequently, by \eqref{Eq:ConcentrationRows}, it holds that
	\begin{align*}
		\yb{i_{\max}}{Q'}-\yb{i}{Q'}\leq \yb{i_{\max}}{Q'}-\lambda_1\mb{i_{\max}}{Q'}+\lambda_1\mb{i}{Q'}-\yb{i}{Q'}\leq 2\rho\sqrt{\lambda_1/|Q'|}\hspace{.5cm}\forall Q'\in \mathcal{Q}\enspace , \\ \intertext{and}
		\yb{i}{Q'}-\yb{i_{\min}}{Q'}\leq \yb{i}{Q'}-\lambda_1\mb{i}{Q'}+\lambda_1\mb{i_{\min}}{Q'}-\yb{i_{\min}}{Q'}\leq 2\rho\sqrt{\lambda_1/|Q'|}\hspace{.5cm}\forall Q'\in \mathcal{Q}\enspace .
	\end{align*}
	Combined with \eqref{Eq:DistimaxMedian}, this yields
	\begin{align*}
		\left|\yb{i}{Q'}-\yb{\med{Q'}}{Q'}\right|\leq 4\rho\sqrt{\lambda_1/|Q'|}\hspace{.5cm}\forall Q'\in\mathcal{Q}\enspace ,
	\end{align*}
 which concludes the proof by definition of $\bar P$.
\end{proof}

Next, we prove a lemma, that relates the empirical medians $\med{Q'}$ to the actual median $\overline{i}$.

\begin{lemma}\label{Lemma:Median}
	Let $N\in \mathbb{N}^*$, $a_1\geq a_2 \geq \dots \geq a_N$ and $b_1,b_2,\dots,b_N$ such that $|a_i-b_i|\leq R $ for $i=1,\dots,N$. Consider $\iota$ such that $|\{i:\ b_i> b_{\iota}\}|< N/2$ and $|\{i:\ b_i< b_{\iota}\}|\leq N/2$. Then $|a_{\lceil N/2\rceil}-a_{\iota}|\leq 2R$.
\end{lemma}
\begin{proof}[Proof of Lemma~\ref{Lemma:Median}]
	We prove the statement by contradiction. Assume first that $a_{\lceil N/2\rceil}-a_{\iota}>2R$. Then, for $i\leq \lceil N/2 \rceil$, we have
	\begin{align*}
		b_i - b_{\iota}\geq a_i-a_{\iota}-2R \geq a_{\lceil N/2\rceil}-a_{\iota}-2R>0\enspace ,
	\end{align*}
	so $b_i > b_{\iota}$, which is a contradiction to $|\{i:\ b_i> b_{\iota}\}|< N/2$. So that $a_{\lceil N/2\rceil}-a_{\iota} \leq 2R$. Similarly, we prove that $a_{\iota} - a_{\lceil N/2\rceil} \leq 2R$ and conclude the proof.
\end{proof}

\begin{proof}[Proof of Lemma~\ref{Lemma:Blocks}]
For the entire proof, assume that Property~\ref{Assum1} is satisfied, and that we are on the event $\xi$. To simplify notation, we assume in what follows w.l.o.g. that $\pi=\mathrm{id}_{[n]}$ and $\eta =\mathrm{id}_{[d]}$.

Property~\ref{Assum1} states that 
\begin{align*}
Q^*(E)=\{j\in[d]:\ \max_{i\in E}M_{ij}\geq p+h,\ \min_{i\in E}M_{ij}\leq p-h\}\subseteq Q\enspace .
\end{align*}
Hence, for any $j\in[d]\setminus Q$, it holds that 
\begin{align*}
	\max_{i\in \overline P}M_{ij}\leq \max_{i\in E}M_{ij}<p+h \hspace{.5cm}\text{or}\hspace{.5cm}
	\min_{i\in \overline P}M_{ij}\geq \min_{i\in E}M_{ij}>p+h\enspace , \\ \intertext{and consequently}
	|\{i\in \overline{P}:\ M_{ij}\leq p-h\}|\wedge |\{i\in \overline{P}:\ M_{ij}\geq p+h\}|=0\enspace.
\end{align*}
This proves the equality of the sums in the statement of the lemma.

Throughout this proof, we will consider the quantities $\overline{j}\coloneqq \max\{j \in Q:\ M_{\overline{i}j}\geq p\}$ and 
\begin{align}
    \overline{O}\coloneqq \left\{i\in E:\ \exists Q'\in\mathcal{Q}\text{ s.t. } \yb{i}{Q'}-\yb{\med{Q'}}{Q'}> 8\rho\sqrt{\lambda_1/|Q'|}\ \right\}\enspace . \label{Def:Conservative_Trisection}
\end{align}

 Recall that we have defined $\ya{E}{j}$ in~\eqref{eq:def:average_over_experts}, which we have used to define in \eqref{Eq:DefLeftRight} the set $A_{\overline j}=\{j^{\prime}\in Q: |\ya{E}{j'}-\ya{E}{\overline j}| \leq 2\rho \sqrt{\lambda_1/ |E|}\}$ and for $c\in \{2,3,\dots,\lceil \sqrt{\lambda_1|E|}/2\rho+1\rceil\}$ the sets $Q_{\overline j}^l(c)= \{j^{\prime} \in Q: 2\rho \sqrt{\lambda_1 /|E|}<\ya{E}{j'}-\ya{E}{\overline j} \leq 2c\rho \sqrt{\lambda_1 /|E|}\}$ and $Q_{\overline j}^r(c)=\{j^{\prime} \in Q: 2\rho \sqrt{\lambda_1/ |E|}< \ya{E}{\overline j}-\ya{E}{j'} \leq 2c\rho \sqrt{\lambda_1/ |E|}\}$. Let us define further
\begin{align*}
c_r^*\coloneqq \inf\{c\in\{1,2,\dots,\lceil \sqrt{\lambda_1|E|}/2\rho\rceil\}: \lambda_1| Q_{\overline{j}}^r(c+1)|>576\rho^2/h^2\}	\enspace , \\
c_l^*\coloneqq \inf\{c\in\{1,2,\dots,\lceil \sqrt{\lambda_1|E|}/2\rho\rceil\}: \lambda_1| Q_{\overline{j}}^l(c+1)|>576\rho^2/h^2\}\enspace ,
\end{align*}
with $\inf \emptyset =\infty$, $\infty +1 =\infty$ and the conventions $Q_{\overline j}^l(\infty)= \{j^{\prime} \in Q: 2\rho \sqrt{\lambda_1 /|E|}<\ya{E}{j'}-\ya{E}{\overline j} \}$, $Q_{\overline j}^r(\infty)=\{j^{\prime} \in Q: 2\rho \sqrt{\lambda_1/ |E|}< \ya{E}{\overline j}-\ya{E}{j'} \}$ and $Q_{\overline{j}}^r(1)=Q_{\overline{j}}^l(1)=\emptyset$.

 We will provide an error decomposition on the sets of questions
 	\begin{align}
		&\underline{R}\coloneqq Q_{\overline{j}}^r(c_r^*)\enspace , \quad R\coloneqq Q_{\overline{j}}^r(c_r^*+1)\enspace ,  \quad \overline{R}\coloneqq Q_{\overline{j}}^r(c_r^*+2)\enspace ,  \quad  \Delta_R\coloneqq \overline{R}\setminus \underline{R}\enspace , \label{Eq:DefSubsets} \\
  &\underline{L}\coloneqq Q_{\overline{j}}^l(c_l^*)\enspace ,  \quad L \coloneqq Q_{\overline{j}}^l(c_l^*+1)\enspace ,  \quad  \overline{L} \coloneqq Q_{\overline{j}}^l(c_l^*+2)\enspace ,  \quad  \Delta_L\coloneqq \overline{L}\setminus \underline{L}\enspace ,\notag \\
  &A\coloneqq A_{\overline{j}}\enspace, \notag
	\end{align} 
	\begin{figure}[h]
	\includegraphics{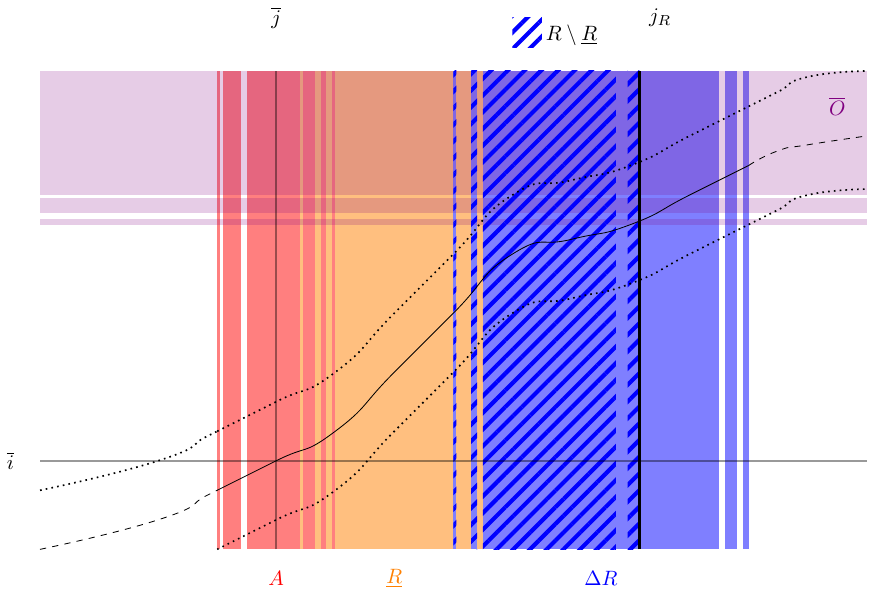}
\caption{Illustration of a part of $M_{E,Q}$. The dotted curves separate two areas of the matrix: That where the entries are at least $p+h$ (top left) and that where they are at most $p-h$ (bottom right). The curve in between separates values that are at least $p$ from values that are smaller than $p$. 
Question $\overline j$   is the ``last'' question for which the median expert $\overline{i}$ has value at least $p$. Our algorithm relies on two detection steps: First, we use $\tilde Y^{(a)}$ and the corresponding column averages $\yb{E}{j}$, to detect areas of interest left and right of $\overline j$ (see \eqref{Eq:DefSubsets} for a definition of these sets). Then, we detect from $\tilde Y^{(b)}$, whether an expert is above or below $\overline i$ and for that purpose, we focus on the following sets of questions: First, $A$ corresponds to the questions for which we cannot detect from our observation, whether they are left or right of $\overline j$. Second, $\underline{R}$ contains questions that are provably right $\overline{j}$, but the size of $\underline{R}$ is too small for reliably detecting experts above $\overline{i}$ from the given observation $\tilde Y^{(b)}$. Though, we can detect those relying on averages on the larger sets of questions $R$ and $\overline R$. As a consequence, experts that differ from $\overline i$ on this whole areas (in particular have all values at least $p+h$) are assigned to $\overline{O}$ by \eqref{Def:Conservative_Trisection}. Following this, experts that remained in $\overline{P}$ cannot ``perform better'' than $p+h$ on every question in $R\setminus \underline{R}$ and consequently there exists $j_R\in R$ with $M_{ij_R}<p+h$ for all $i\in \overline P$. By the bi-isotonicity, only questions $j< j_R$ can contribute to the error we want to bound, which is why we extend our analysis to $\Delta_R=\overline R\setminus R$.}
\label{Fig:Trisection}
\end{figure}

We will show following upper bounds:
	\begin{itemize}
		\item $	\tilde{\mathcal{R}}_{p,h,\overline P,\Delta_R}\leq 144\rho^2(|E|\vee|Q|)/\lambda_1h^2$, \quad $	\tilde{\mathcal{R}}_{p,h,\overline P,\Delta_L}\leq 144\rho^2(|E|\vee|Q|)/\lambda_1h^2$,
		\item $ \tilde{\mathcal{R}}_{p,h,\overline P,\underline R} \leq 576\rho^2(|E|\vee |Q|)/\lambda_1h^2$,  \quad  $ \tilde{\mathcal{R}}_{p,h,\overline P,\underline L} \leq 576\rho^2(|E|\vee |Q|)/\lambda_1h^2$,
		\item $\tilde{\mathcal{R}}_{p,h,\overline P,A} \leq 2304\rho^2(|E|\vee |Q|)/\lambda_1h^2$ and
		\item $\tilde{\mathcal{R}}_{p,h,\overline P,Q\setminus(\overline L\cup A\cup\overline R)} =0$.
	\end{itemize}
	Summing them up concludes the proof. We want to point out, that in the case $c_r^*=1$ or $c_l^*=1$, parts of the analysis become trivial, this is even more true if $c_r^*=\infty$ or $c_l^*=\infty$. In the case $c_r^*=1$ for example, we have $\underline R =\emptyset$ and therefore $\tilde{\mathcal{R}}_{p,h,\overline P,\underline R}=0$. In the case $c_r^*=\infty$, we obtain $\Delta_R=\emptyset$ which implies $\tilde{\mathcal{R}}_{p,h,\overline P,\Delta_R}=0$. Further, note that by \eqref{Eq:ConcentrationCols} and $M$ taking values in $[0,1]$, it holds
 \begin{align*}
     \ya{E}{\overline j}-\ya{E}{j}&\leq 2\rho\sqrt{\lambda_1  /|E|}+\lambda_1\ma{E}{\overline j}-\lambda_1\ma{E}{j}\\
     &\leq (2+\sqrt{\lambda_1|E|}/\rho)\rho\sqrt{\lambda_1/|E|}\\
     &\leq 2\lceil \sqrt{\lambda_1|E|}/2\rho+1\rceil \rho\sqrt{\lambda_1/|E|}\enspace,
 \end{align*}
 For $c_r^*=\infty$, this implies $\underline R=Q_{\overline j}^r(\infty)=Q_{\overline j}^r(\lceil \sqrt{\lambda_1|E|}/2\rho+1\rceil)$ and therefore $\lambda_1|\underline R|\leq 576\rho^2/h^2$ (otherwise, this would mean $c_r^*\leq \lceil \sqrt{\lambda_1|E|}/2\rho\rceil$), which is sufficient for the analysis on $\underline R$ we provide below.

Before providing the proofs, let us give some intuition for the analysis. During \texttt{ScandAndUpdate}, we construct sets $Q'\subseteq Q$ in order to have meaningful comparisons of experts $i\in E$, in particular with $\overline i$. Assume we want to detect $i<\overline i$. We will see, that given a set $Q'\subseteq Q$ with $|Q'|\gtrsim \frac{\rho^2}{\lambda_1 h^2}$ and $M_{\overline ij}\leq p$ for all $j\in Q'$, we are able to detect all $i\in E$ with $M_{ij}\geq p+h$ for all $j\in Q'$ which by the bi-isotonicity of $M$ would be sufficient for $\pi(i)<\pi(\overline i)$ and will lead to $i\in \overline O$. We will see that the constructed sets $R$ and $\overline R$ have such properties, see Figure~\ref{Fig:Trisection} for an illustration. The bi-isotonicity assumption will then yield that there exist $j_R\in R$ with $M_{ij}< p+h$ for all $i\in \overline{P}$. Informally, this means that questions $j\geq j_R$ don't contribute to the error in the sense that
\begin{align*}
    |\{i\in \overline{P}:\ M_{ij}\leq p-h\}|\wedge |\{i\in \overline{P}:\ M_{ij}\geq p+h\}|=0.
\end{align*}
In the same way, we can look at $L$ and construct $j_L\in L$ with $j\leq j_L$ not contributing to the error. By looking at $\overline R\cup A\cup \overline L$, we capture all $j$ with $j_L< j< j_R$, so all $j$ that might contribute to the error $\tilde {\mathcal R}_{p,h,\overline P,Q}$.

We further decompose $\overline R=\underline R\cup \Delta_R$. The set $\underline R$ was considered too small for a meaningful comparison of experts $i\in E$, but this will turn out to lead also to a simple bound of $\tilde{R}_{p,h,\overline P,\underline R}$. 

By construction, the questions in $\Delta_R$ cannot be much easier than $j_R$, this means that we have a good control of the number of $i\in \overline P$ with $M_{ij}\geq p+h$ for $j\in \Delta_R$. At the same times, $\overline P$ is by definition a set of experts $i\in \overline P$ for which we have a control over the deviation from $\overline i$ on $\overline R$. Since $M_{\overline ij}\leq p$ for all $j\in \Delta_R$, we can also bound the number of questions $j$ such that $M_{ij}\geq p+h$. We will see that combining these bounds will allow us to bound $\tilde{\mathcal{R}}_{p,h,\overline P,\Delta_R}$, and we can bound $\tilde{\mathcal{R}}_{p,h,\overline P,\Delta_L}$ in the same manner.

$A$ contains all questions $j$ for which we could not detect whether $j\geq \overline j$ or $j<\overline j$. Still, we can decompose $A$ into questions $j\in A$ with $j\geq \overline j$ and questions $j<\overline j$. We will see that depending on the respective sizes of the sets, either the ideas of the analysis on $\underline R$ and $\underline L$ or concepts of the analysis on $\Delta_R$ and $\Delta_L$ carry over.
 
	\paragraph{Analysis on $\Delta_R$ and $\Delta_L$.}
	We will only prove the upper bound for $\Delta_R$, since it can be easily adapted for $\Delta_L$. More precisely, we will use 
	\begin{align*}
		\tilde{\mathcal{R}}_{p,h,\overline P, \Delta_R}\leq  |\{(i,j)\in \overline P\times \Delta_R:\ M_{ij}\geq p+h\}|\enspace .
	\end{align*}
	If $i_{\min}\coloneqq \min \overline P$ and $j_{\min}\coloneqq \min \Delta_R$, the bi-isotonicity of $M$ yields
	\begin{align*}
		\{(i,j)\in \overline P\times \Delta_R:\ M_{ij}\geq p+h\}\subseteq \{i\in \overline P:\ M_{ij_{\min}}\geq p+h\}\times\{j\in  \Delta_R:\ M_{i_{\min}j}\geq p+h\}\enspace ,
	\end{align*}
	and we will bound
	\begin{align*}
		\tilde{\mathcal{R}}_{p,h,\overline P, \Delta_R}\leq|\{i\in \overline P:\ M_{ij_{\min}}\geq p+h\}|\cdot|\{j\in  \Delta_R:\ M_{i_{\min}j}\geq p+h\}|\enspace .
	\end{align*}
	This will be done in four steps:
	\begin{enumerate}[label=a\arabic*), ref=a\arabic*)]
		\item \label{Item:DeltaRa} $M_{\overline i j}<p $ for all $j\in \overline R$,
		\item \label{Item:DeltaRb} $M_{i j_{\max}}\leq p+h/2$ for $j_{\max}\coloneqq \max \overline R$ and all $i\in \overline P$,
		\item \label{Item:DeltaRc} $|\{i\in \overline P:\ M_{ij_{\min}}\geq p+h\}|\leq 12 \rho\sqrt{|E|}/\sqrt{\lambda_1}h$,
		\item \label{Item:DeltaRd}$|\{j\in \Delta_R:\ M_{i_{\min} j}\geq p+h\}|\leq 12\rho\sqrt{|\overline R|}/\sqrt{\lambda_1}h$.
	\end{enumerate}
	From \ref{Item:DeltaRc} and \ref{Item:DeltaRd}, it  follows directly the desired bound
	\begin{align*}
		\tilde{\mathcal{R}}_{p,h,\overline P, \Delta_R}\leq 144\rho^2 \sqrt{|E|\cdot|Q|}/\lambda_1h^2\leq 144\rho^2 (|E|\vee|Q|)/\lambda_1h^2\enspace .
	\end{align*}

	\subparagraph{Proof of \ref{Item:DeltaRa}:}
	By definition of $\overline R$ in Equation~\eqref{Eq:DefLeftRight} with Equation~\eqref{Eq:DefSubsets} and the concentration inequality in Equation~\eqref{Eq:ConcentrationCols}, we know for every $j\in \overline R$, that
	\begin{align*}
		\lambda_1 \left(\ma{E}{\overline j} -\ma{E}{j} \right)\geq \ya{E}{\overline j} -\ya{E}{j}-2\rho\sqrt{\lambda_1/|E|}>0\enspace .
	\end{align*}
	The bi-isotonicity of $M$ implies $j>\overline j$ and also $M_{\overline i j}<p$ by definition of $\overline j$.

	\subparagraph{Proof of \ref{Item:DeltaRb}:} Let $i\in \overline P$. From the concentration equality in Equation~\eqref{Eq:ConcentrationRows}, Lemma \ref{Lemma:Median} and the definition of $\overline P$ in Equation~\eqref{Eq:DefConservativeTrisection} we know that
	\begin{align}
		\lambda_1(\mb{i}{\overline R}-\mb{\overline i}{\overline R})&\leq \lambda_1(\mb{i}{\overline R}-\mb{\med{\overline R}}{\overline R})+2\rho\sqrt{\lambda_1/|\overline R|} \notag\\
		&\leq\yb{i}{\overline R}-\yb{\med{\overline R}}{\overline R}+4\rho\sqrt{\lambda_1/|\overline R|}\leq 12\rho\sqrt{\lambda_1/|\overline R|}\enspace . \label{Eq:DevR}
	\end{align}
	This together with \ref{Item:DeltaRa} implies $\mb{i}{\overline R}\leq p+12\rho/\sqrt{\lambda_1 /|\overline R|}\leq p+h/2$, since $\lambda_1|\overline R|\geq 576\rho^2/h^2$ by definition. Finally, by bi-isotonicity of $M$, we conclude that  $M_{ij_{\max}}\leq \mb{i}{\overline R}\leq p+h/2$.

	\subparagraph{Proof of \ref{Item:DeltaRc}:} Recall that $j_{\min} =\min \Delta_R$ while $j_{\max}=\max \overline R$, so $j_{\min}< j_{\max}$. Note that because of \ref{Item:DeltaRb}
	\begin{align}
		|E|\cdot(\ma{E}{j_{\min}}-\ma{E}{j_{\max}})&=\sum_{i\in E}M_{ij_{\min}}-M_{ij_{\max}} \notag\\
		&\geq \sum_{i\in \overline P,\ M_{ij_{\min}}\geq p+h}M_{ij_{\min}}-M_{ij_{\max}}\notag\\
		&\geq (h/2)\cdot |\{i\in \overline P:\ M_{ij_{\min}}\geq p+h\}|\enspace . \label{Eq:CountSum}
	\end{align}
	If $\ya{E}{j_{\max}} \geq \ya{E}{j_{\min}}$, concentration inequality \eqref{Eq:ConcentrationCols} yields
	\begin{align*}
		\lambda_1(\ma{E}{j_{\min}}-\ma{E}{j_{\max}})\leq \ya{E}{j_{\min}}-\ya{E}{j_{\max}}+2\rho\sqrt{\lambda_1/|E|}\leq 2\rho\sqrt{\lambda_1/|E|}\enspace .
	\end{align*} 
	Otherwise, by definition of $\Delta_R$ and $\overline R$ in \eqref{Eq:DefLeftRight} and \eqref{Eq:DefSubsets}, 
	\begin{align*}
		\ya{E}{\overline j}- 2(c_r^*+2)\rho\sqrt{\lambda_1/|E|}\leq \ya{E}{j_{\max}}<\ya{E}{j_{\min}}<\ya{E}{\overline j}-2c_r^*\rho\sqrt{\lambda_1/|E|}\enspace ,
	\end{align*}
	which yields 
	\begin{align*}
		\lambda_1(\ma{E}{j_{\min}}-\ma{E}{j_{\max}})\leq \ya{E}{j_{\min}}-\ya{E}{j_{\max}}+2\rho\sqrt{\lambda_1/|E|}\leq 6\rho\sqrt{\lambda_1/|E|}\enspace .
	\end{align*} 
	This together with \eqref{Eq:CountSum} gives us $|\{i\in \overline P:\ M_{ij_{\min}}\geq p+h\}|\leq 12 \rho\sqrt{|E|}/\sqrt{\lambda_1}h$.

	\subparagraph{Proof of \ref{Item:DeltaRd}:} In \ref{Item:DeltaRa} we have seen $M_{\overline i j}<p $ for all $j\in \overline R$. So we can argue like in \eqref{Eq:CountSum} and obtain 
	\begin{align*}
		|\{j\in \Delta_R:\ M_{i_{\max} j}\geq p+h\}\leq |\overline R|\cdot (\mb{i_{\min}}{\overline R}-\mb{\overline i}{\overline R})/h\enspace .
	\end{align*}
	From \eqref{Eq:DevR} we know
	\begin{align*}
		\lambda_1(\mb{i_{\min}}{R}-\mb{\overline i}{R})\leq 12\rho\sqrt{\lambda_1/|\overline R|}\enspace , 
	\end{align*}
	and end up with $|\{j\in \Delta_R:\ M_{i_{\min} j}\geq p+h\}|\leq 12\rho\sqrt{|R|}/\sqrt{\lambda_1}h$\enspace .

	\paragraph{Analysis on $\underline R$ and $\underline L$.}
	
	We will only consider $\underline R$ as the proof is similar for $\underline L$. From the construction of $\underline R$ in Equation~\eqref{Eq:DefSubsets}, it follows that $\lambda_1|\underline R|\leq 576\rho^2/h^2$. We obtain the trivial upper bound 
	\begin{align}
		\tilde{\mathcal{R}}_{p,h,\overline P,\underline R}\leq|\overline P\times\underline R|  \leq 576\rho^2 |\overline P|/\lambda_1h^2 \leq 576\rho^2 (|E|\vee |Q|)/\lambda_1h^2\enspace ,\label{Eq:UBBySize}
	\end{align}
	where for the last step we just used the rough inequality $|\overline P|\leq |E|\leq |E|\vee |Q|$.

	\paragraph{Analysis on $A$.} Recall that $\overline j\in A$ was the last question $j$ such that $M_{\overline ij}\geq p$ holds. We will split up $A$ into $A^l\coloneqq \{j\in A:\ j\leq \overline j\}=\{j\in A:\ M_{\overline i j}\geq p\}$ and $A^r\coloneqq\{j\in A:\ j>\overline j\}=\{j\in A:\ M_{\overline i j}<p\}$. If $\lambda_1|A|\leq 2304\rho^2/h^2$, we can conclude just like in \eqref{Eq:UBBySize} that
	\begin{align*}
		\tilde{\mathcal{R}}_{p,h,\overline P,A}\leq 2304 \rho^2(|E|\vee |Q|)/\lambda_1 h^2\enspace .
	\end{align*}
	So assume w.l.o.g. that $\lambda_1|A|>2304\rho^2/h^2$ and $|A^r|\geq |A^l|$, so in particular $\lambda_1|A^r|\geq\lambda_1 |A|/2>1152 \rho^2/h^2$. Again, we will derive the upper bound in several steps:
	\begin{enumerate}[label=b\arabic*),ref=b\arabic*)]
		\item \label{Item:Aa}$|\{j\in A^r:\ M_{i_{\min}j}\geq p+h\}|\leq 12 \rho\sqrt{|A|}/\sqrt{\lambda_1}h$ for $i_{\min}=\min \overline P$,
		\item \label{Item:Ab}$M_{ij_{\max}}< p+h/2$ for $j_{\max}=\max A$ and all $i\in \overline P$,
		\item \label{Item:Ac}$|\{i\in \overline P:\ M_{i\overline{j}}\geq p+h\}|\leq8\rho\sqrt{|E|/\lambda_1h^2}$ and $|\{i\in \overline P:\ M_{ij_{\min}}\geq p+h\}|\leq12\rho\sqrt{|E|/\lambda_1h^2}$ for $j_{\min}=\min A$,
		\item \label{Item:Ad} $|\{j\in A^l:\ M_{i_{\max}j}\leq p-h\}|\leq12 \rho\sqrt{|A|/\lambda_1 h^2}$ for $i_{\max}=\max \overline P$.
	\end{enumerate}
	Once this is proven, one can bound $\tilde{\mathcal{R}}_{p,h,\overline{P},A}=\tilde{\mathcal{R}}_{p,h,\overline{P},A^r}+\tilde{\mathcal{R}}_{p,h,\overline{P},A^l}$. It holds
	\begin{align*}
		\tilde{\mathcal{R}}_{p,h,\overline{P},A^r}&\leq |\{(i,j)\in \overline P\times A^r:\ M_{ij}\geq p+h\}|\\
		&\leq |\{i\in\overline P:\ M_{i\overline j}\geq p+h\}|\cdot|\{j\in A^r:\ M_{i_{\min}j}\geq p+h\}|\leq 96\rho^2(|E|\vee |Q|)/\lambda_1h^2\enspace ,
	\end{align*}
	by \ref{Item:Aa} and \ref{Item:Ac} and 
	\begin{align*}
		\tilde{\mathcal{R}}_{p,h,\overline{P},A^l}&=\sum_{j\in A^l} |\{i\in \overline P:\ M_{ij}\leq p-h\}| \wedge |\{i\in \overline P:\ M_{ij}\geq p+h\}|\\
		&\leq \sum_{j\in A^l,\ M_{i_{\max}j}\leq p-h} |\{i\in \overline P:\ M_{ij}\geq p+h\}|\\
		&\leq |\{j\in A^l:\ M_{i_{\max}j}\leq p-h\}|\cdot |\{i\in \overline P:\ M_{ij_{\min}}\geq p+h\}|\leq 144\rho^2(|E|\vee |Q|)/\lambda_1h^2\enspace ,
	\end{align*}
	by \ref{Item:Ad} and \ref{Item:Ac}, so in total, we would end up with $\tilde{\mathcal{R}}_{p,h,\overline P,A}\leq 240\rho^2(|E|\vee |Q|)/\lambda_1h^2$.

	\subparagraph{Proof of \ref{Item:Aa} and \ref{Item:Ab}:} Let $i_{\min}=\min \overline P$. Like in \eqref{Eq:DevR} one can prove
	\begin{align*}
		\lambda_1(\mb{i_{\min}}{A}-\mb{\overline i}{A})\leq 12\rho\sqrt{\lambda_1/|A|}\enspace .
	\end{align*}
	This implies
	\begin{align*}
		|A^r|(\mb{i_{\min}}{A^r}-\mb{\overline i}{A^r})&=\sum_{j\in A^r}M_{i_{\min}j}-M_{\overline i j}\\ & \leq \sum_{j\in A}M_{i_{\min}j}-M_{\overline i j}\\
		&\quad =|A|(\mb{i_{\min}}{A}-\mb{\overline i}{A})\leq 12\rho \sqrt{|A|/\lambda_1}\leq \rho\sqrt{288|A^r|/ \lambda_1}\enspace .
	\end{align*}
	By definition of $A^r$, we know that $M_{\overline i j}<p$ for all $j \in A^r$, so we can show \ref{Item:Aa} by arguing like in \eqref{Eq:CountSum} that
	\begin{align*}
		|\{j\in A^r:\ M_{i_{\min}j}\geq p+h\}|\leq |A^r|(\mb{i_{\min}}{A^r}-\mb{\overline i}{A^r})/h\leq 12 \rho\sqrt{|A|}/(\sqrt{\lambda_1}h)\enspace .
	\end{align*}
	For showing \ref{Item:Ab}, note that it holds
	\begin{align*}
		\sum_{j\in A^r}M_{i_{\min}j}-M_{\overline i j}> |A^r|(\mb{i_{\min}}{A^r}-p)\enspace ,
	\end{align*}
	which implies 
	\begin{align*}
		\mb{i_{\min}}{A^r}< p+\rho\sqrt{288/(\lambda_1|A^r|)}< p+h/2
	\end{align*}
	by the assumption $\lambda_1|A^r|>1152\rho^2/h^2$. For any $i\in \overline P$ we conclude by the bi-isotonicity of $M$
	\begin{align*}
		M_{ij_{\max}}\leq M_{i_{\min}j_{\max}}\leq \mb{i_{\min}}{A^r}<p+h/2\enspace .
	\end{align*}

	\subparagraph{Proof of \ref{Item:Ac}:} By the definition of $A$ in \eqref{Eq:DefLeftRight} and \eqref{Eq:DefSubsets} we know that $|\ya{E}{\overline j}-\ya{E}{j_{\max}}|\leq 2\rho\sqrt{\lambda_1/|E|}$ and $|\ya{E}{j_{\min}}-\ya{E}{j_{\max}}|\leq 4\rho\sqrt{\lambda_1/|E|}$. With concentration inequality \eqref{Eq:ConcentrationCols} this yields
	\begin{align*}
		\lambda_1(\ma{E}{\overline j}-\ma{E}{j_{\max}})\leq 4\rho\sqrt{\lambda_1/|E|}\quad \text{ and }
		\lambda_1(\ma{E}{j_{\min}}-\ma{E}{j_{\max}})\leq 6\rho\sqrt{\lambda_1/|E|}\enspace .
	\end{align*} 
	Like in \eqref{Eq:CountSum}, using $M_{ij_{\max}}\leq p+h/2$ for $i\in \overline P$ from \ref{Item:Ab}, we obtain
	\begin{align*}
		|\{i\in \overline P:\ M_{i\overline{j}}\geq p+h\}|\leq 2|E|(\ma{E}{\overline j}-\ma{E}{j_{\max}})/h\leq 8\rho\sqrt{|E|/\lambda_1h^2}\\\intertext{and}
		|\{i\in \overline P:\ M_{ij_{\min}}\geq p+h\}|\leq 2|E|(\ma{E}{j_{\min}}-\ma{E}{j_{\max}})/h\leq 12\rho\sqrt{|E|/\lambda_1h^2}\enspace .	
	\end{align*}

	\subparagraph{Proof of \ref{Item:Ad}:} This follows by adapting the proof of \ref{Item:Aa}.

	\paragraph{Analysis on $Q\setminus (\overline L\cup A\cup\overline R)$.} If $j\in Q$ but $j\notin \overline L\cup A\cup\overline R$, then we have either 
	\begin{align*}
		\ya{E}{\overline j}-\ya{E}{ j}>2(c_r^*+2)\sqrt{\lambda_1/|E|}\quad \text{ or }\quad 
		\ya{E}{j}-\ya{E}{\overline j}>2(c_l^*+2)\sqrt{\lambda_1/|E|}\enspace .
	\end{align*}
	Assume w.l.o.g. that the first inequality holds. For $j_{\max}=\max R$, we know 
	\begin{align*}
	\ya{E}{\overline j}-\ya{E}{ j_{\max}}\leq 2(c_r^*+1)\sqrt{\lambda_1/|E|}\\ \intertext{ and therefore} 
	\ya{E}{ j_{\max}}-\ya{E}{j}\leq 2\sqrt{\lambda_1/|E|}\enspace .
	\end{align*}
	 The concentration inequality \eqref{Eq:ConcentrationCols} implies $\lambda_1(\ma{E}{ j_{\max}}-\ma{E}{j})\leq 0$ and the bi-isotonicity of $M$ yields $j>j_{\max}$. Like in \ref{Item:DeltaRb}, we deduce $M_{ij}\leq p+h/2$ for all $i\in \overline P$, so $|\{i\in \overline P:\ M_{ij}\geq p+h\}|=0$.
\end{proof}

\subsection{Proof of Lemma~\ref{Lemma:Envelopes}}\label{Sec:AnalysisQuestions}

Consider $\tilde Y$ sampled accordingly to observation model in Equation~\eqref{Eq:Model}, $\mathcal{E}=(E_1,E_2,\dots,E_r)$ and a vector $v\in\{0,1\}^r$. Let us redefine
\begin{align*}
	\ya{E_s}{j}=\frac{1}{|E_s|}\sum_{i\in E_s}Y_{ij} \text{ and } \ma{E_s}{j}=\frac{1}{|E_s|}\sum_{i\in E_s}M_{ij} \hspace{0.5cm} \forall s=1,2,\dots,r,\ \forall j\in[d]\enspace .
\end{align*}
Consider the event $\xi_{\mathrm{env}}$ where the inequalities 
\begin{align}
	|\ya{E_s}{j}-\lambda_1 \ma{E_s}{j}|\leq \rho\sqrt{\lambda_1/|E_s|} \hspace{.5cm} \forall s=1,2,\dots,r \text{ with }v_s=1,\ \forall j\in[d]\enspace , \label{Eq:ConcentrationEnvelopes}
\end{align}
hold true.
\begin{lemma}\label{Lemma:GoodEventEnvelopes}
	Under Property~\ref{Assum2} for $\tilde Y$, $\mathcal{E}$ and $v$, it holds that 
	\begin{align*}
		\mathbb{P}(\xi_{\mathrm{env}})\geq 1-\frac{\delta}{12\lceil \log_2(n\vee d)\rceil (n\vee d)^{1/2}}\enspace .
	\end{align*}
\end{lemma}
\begin{proof}[Proof of Lemma~\ref{Lemma:GoodEventEnvelopes}]
    Recall Property~\ref{Assum2} for $\tilde Y$, $\mathcal{E}$ and $v$ from Section \ref{Sec:Questions}. Then for each $s=1,2,\dots,r$ with $v_s=1$, we know that $\lambda_1|E_s|>4\rho^2/h^2$. Lemma~\ref{Lemma:ConcentrationSE} yields that each inequality 
	\begin{align*}
		|\ya{E}{j}-\lambda_1 \ma{E}{j}|\leq \rho\sqrt{\lambda_1/|E_s|}\enspace ,
	\end{align*}
	holds with a probability of at least $1-\frac{\delta }{12\lceil \log_2(n\vee d)\rceil nd(n\vee d)^{(1/2)}}$, just like in the proof of Lemma~\ref{Lemma:GoodEvent}. For each $s$ we are considering $d$ such inequalities. Because $\mathcal{E}$ is a partition of $[n]$ by Property~\ref{Assum2}, we can roughly upper bound $|\{s=1,2,\dots,r:\ v_s=1\}|\leq n$, so the total number of inequalities considered is bounded by $nd$. A union bound argument yields $\mathbb{P}(\xi_{\mathrm{env}})\geq 1-\frac{\delta}{12\lceil \log_2(n\vee d)\rceil (n\vee d)^{1/2}}$.
\end{proof} 
For $s=1,2,\dots,r$ we can define
\begin{align*}
	\underline{Q}(E_s)\coloneqq \{j\in [d]:\ \ya{E_s}{j}\geq \lambda_1(p+h)-\rho\sqrt{\lambda_1/|E_s|}\}\enspace ,\\
	\overline{Q}(E_s)\coloneqq \{j\in [d]:\ \ya{E_s}{j}\leq \lambda_1(p-h)+\rho\sqrt{\lambda_1/|E_s|}\}\enspace .
\end{align*}
If $v_s=1$, Algorithm \ref{Algo:Envelope} checks whether $\underline s=\max\{t<s:\ v_t=1\}$ exists. If this is the case, it sets $\underline {Q_s}=\underline{Q}(E_{\underline s})$, otherwise $\underline {Q_s}=[d]$. In the same way $\overline{Q_s}=\overline{Q}(E_{\overline s})$ or $\overline {Q_s}=[d]$, depending on the existence of $\overline s =\min\{t>s:\ v_t=1\}$. The algorithm then returns $Q_s=\underline{Q_s}\cap \overline{Q_s}$. See Figure \ref{Fig:Envelopes} for an illustration of the construction of $Q_s$.

\begin{figure}[h]
	\centering
\includegraphics{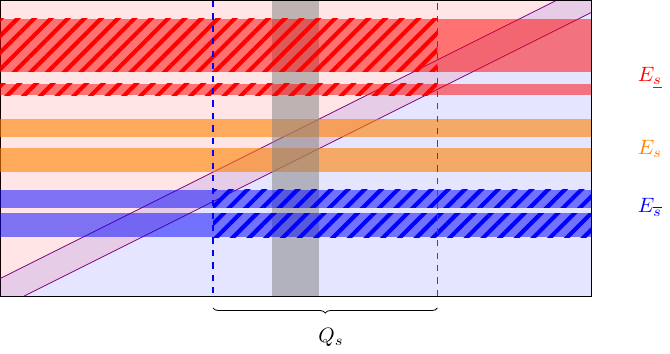}
	\caption{Illustration of some bi-isotonic matrix $M\in \mathbb{C}_{\mathrm{Biso}}(\mathrm{id}_{\lbrack n\rbrack},\mathrm{id}_{\lbrack d\rbrack})$. The part in the top left (light red background) corresponds to matrix values $\geq p+h$, the part on the bottom right (light blue background) to matrix values $\leq p-h$ and the area in between (purple background) to values in $(p-h,p+h)$. Assume we have $E_{\underline s}$, $E_s$, $E_{\overline{s}}\in \mathcal{E}$ such that each of the sets is larger than $4\rho^2/\lambda_1 h^2$ and $\underline i < i<\overline i$ for $\underline i\in E_{\underline s}$, $i\in E_s$ and $\overline i \in E_{\overline s}$. We are interested in estimating the questions $Q^*(E_s)$, which correspond to the gray area. To do so, we detect sets that contain all questions on which the success-probabilities are at least $p+h$ for all experts in $E_{\underline s}$ (red dashed area) and at most $p-h$ for all experts in $E_{\overline s}$ (blue dashed area). Intersecting these sets yields $Q_s$.} \label{Fig:Envelopes}
\end{figure}

\begin{proof}[Proof of Lemma~\ref{Lemma:Envelopes}]
	Consider $s\in\{1,2,\dots, r\}$ with $v_s=1$ and $j\in Q^*(E_s)$. Then 
	\begin{align*}
		\max_{i\in E_s}M_{ij}\geq p+h \hspace{.5cm}\text{and}\hspace{.5cm} \min_{i\in E_s}M_{ij}\leq p-h\enspace .
	\end{align*}
	Assume w.l.o.g. that $\underline{s}=\max\{t<s:\ v_t=1 \}$ exists; otherwise $j\in [d]=\underline{Q_s}$ would be trivial. By Property~\ref{Assum2}, $\pi(i)<\pi(i')$ for all $i\in E_{\underline s}$ and $i'\in E_s$. The bi-isotonicity assumption therefore yields $M_{ij}\geq p+h$ for all $i\in E_{\underline{s}}$, which implies $\ma{E_{\underline s}}{j}\geq p+h$. Since we assume that we are in the regime of event $\xi_{\mathrm{env}}$, it follows from \eqref{Eq:ConcentrationEnvelopes} 
	\begin{align*}
		\ya{E_{\underline s}}{j}\geq \lambda_1\ma{E_{\underline s}}{j}-\rho\sqrt{\lambda_1/|E_{\underline s}|}\geq\lambda_1(p+h)-\rho\sqrt{\lambda_1/|E_{\underline s}|}
	\end{align*} 
	 and consequently $j\in \underline{Q_s}$. In the same manner, one can show $j\in \overline{Q_s}$. This proves
	\begin{align*}
		Q^*(E_s)\subseteq \underline{Q_s}\cap \overline{Q_s}=Q_s
	\end{align*}

	For the second claim it suffices to show that for each $j\in [d]$ there are at most three $s\in\{1,2,\dots,r\}$ with $v_s=1$ such that $j\in Q_s$. We show this by contradiction. So assume $1\leq s<t<u<w\leq r$ exist with $v_{s}=v_{t}=v_u=v_w=1$ and $j\in Q_s\cap Q_t\cap Q_u\cap Q_w$.

	We obtain $j\in \overline{Q_s} \cap \underline{Q_w}$, so 
	\begin{align*}
	\ya{E_{\overline s}}{j} \leq \lambda_1(p-h)+\rho\sqrt{\lambda_1/|E_{\overline s}|}\hspace{.5cm}\text{and}\hspace{.5cm} \ya{E_{\underline w}}{j} \geq \lambda_1(p+h)-\rho\sqrt{\lambda_1/|E_{\underline w}|}\enspace .
	\end{align*}
	 It holds $\overline{s}\leq t<u\leq \underline w$ and $v_{\overline s}=v_{\underline w}=1$, and Property~\ref{Assum2} implies 
	\begin{align*}
		\lambda_1|E_{\overline s}|>4\rho^2/h^2 \hspace{.5cm}\text{and}\hspace{.5cm} \lambda_1|E_{\underline w}|>4\rho^2/h^2\enspace ,
	\end{align*}
	which is equivalent to 
		\begin{align*}
		\lambda_1 h>2\rho\sqrt{\lambda_1/ |E_{\overline s}|} \hspace{.5cm}\text{and}\hspace{.5cm} \lambda_1 h>2\rho\sqrt{\lambda_1/| E_{\underline w}|} \enspace .
	\end{align*}
	All this together with \eqref{Eq:ConcentrationEnvelopes} yields
	\begin{align*}
		\lambda_1\ma{E_{\overline s}}{j}\leq \ya{E_{\overline s}}{j}+\rho\sqrt{\lambda_1/|E_{\overline{s}}|}\leq \lambda_1(p-h)+2\rho\sqrt{\lambda_1/|E_{\overline{s}}|}<\lambda_1 p\enspace , \\ \intertext{and}
			\lambda_1\ma{E_{\underline w}}{j}\geq \ya{E_{\underline w}}{j}-\rho\sqrt{\lambda_1/|E_{\overline{s}}|}\geq \lambda_1(p+h)-2\rho\sqrt{\lambda_1/|E_{\overline{s}}|}>\lambda_1 p\enspace .
	\end{align*}
	So in particular $\ma{E_{\overline s}}{j}<\ma{E_{\underline w}}{j}$. This means that there exist $\overline i\in E_{\overline s}$ and $\underline i\in E_{\underline w}$ with $M_{\overline i j}<M_{\underline i j}$. The bi-isotonicity assumption implies $\pi(\overline i)>\pi(\underline i)$, and together with $\overline s < \underline w$, this yields a contradiction to Property~\ref{Assum2}.
\end{proof}

\subsection{Proofs of Theorem~\ref{Thm:Properties} and Corollary~\ref{Cor:PropertiesTree}}\label{Sec:AnalysisTree}

We will see, that for each active set $E_t^{k-1}$, i.e.~$w_t^{k-1}=1$, our algorithm relies on inequalities of the form
\begin{align}
	\frac{1}{|E_t^{k-1}|}\left|\sum_{i\in E_t^{k-1}}Y_{ij}^{(3k-1)}-\lambda_1M_{ij}\right|\leq \rho\sqrt{\lambda_1/|E_t^{k-1}|}\hspace{.5cm}\forall j\in Q_t^{k-1}\enspace ,\notag \\ \intertext{and}
	\frac{1}{|Q'|}\left|\sum_{j\in Q'}Y_{ij}^{(3k)}-\lambda_1M_{ij}\right|\leq \rho\sqrt{\lambda_1/|Q'|}\hspace{.5cm}\forall Q'\in \mathcal{Q}_t^{k-1},\ i\in E_t^{k-1}\enspace ,\label{Eq:ConcentrationInTree}
\end{align}
where $\mathcal{Q}_t^{k-1}$ is defined by sets of the form given in Equation~\eqref{Eq:DefLeftRight}.
\begin{proof}[Proof of Theorem~\ref{Thm:Properties}]
	We will prove the claims via induction over $k$.
	
	\paragraph{The case $k=1$:} If $\lambda_1n\leq 4\rho^2/h^2$, there is nothing to prove, so assume $\lambda_1 n > 4\rho^2/h^2$.	By definition, Algorithm \ref{Algo:Envelope} returns $Q_1^0=[d]$ which shows \ref{ThmItem:Envelope}. If $\lambda_1 d \leq 4\rho^2/h^2$ it sets $w^0=(0)$, $\mathcal{E}^1=([n])$, $v^1=(0)$ and $G^1=G^0$. One can easily check that in this case, all claims are fulfilled.
	If $\lambda_1d> 4\rho^2/h^2$, let us consider the concentration inequalities
	\begin{align*}
		\frac{1}{n}\left|\sum_{i=1}^n Y_{ij}^{(2)}-\lambda_1 M_{ij}\right|\leq \rho\sqrt{\lambda_1/n} \hspace{.5cm} \forall j\in[d]\enspace ,\\
		\intertext{and}
		\frac{1}{|Q'|}\left|\sum_{j\in Q'} Y_{ij}^{(3)}-\lambda_1 M_{ij}\right|\leq \rho\sqrt{\lambda_1/|Q'|} \hspace{.5cm} \forall Q'\in\mathcal{Q}_1^0\enspace ,\ i\in[n] \enspace.
	\end{align*}
	Here, $\mathcal{Q}_1^0$ is defined as $\mathcal{Q}$ in Section \ref{Sec:AnalysisTrisection} by sets of the form \eqref{Eq:DefLeftRight} with respect to $E=[n]$. We call $\xi^1$ the event, on which these inequalities hold.	Since Property~\ref{Assum1} for $Y^{(2)}$, $Y^{(3)}$, $[n]$, $[d]$ and $G$ is obviously fulfilled, Lemma~\ref{Lemma:GoodEvent} yields
	\begin{align*}
		\mathbb{P}(\xi^1)\geq 1-\frac{\delta}{12\lceil \log_2(n\vee d)\rceil}\enspace .
	\end{align*}

	The claims \ref{ThmItem:TrisectIsPartition} (and therefore \ref{ThmItem:Partition}) and \ref{ThmItem:MedianSeparates} follow directly from Lemma~\ref{Lemma:PropertiesTrisection}. By \ref{ThmItem:MedianSeparates} and the definition of $\overline{i}_1^0$ in \eqref{Eq:DefMedian} we obtain $|O([n])|\leq n/2$ and $|I([n])|\leq n/2$, which yields \ref{ThmItem:SetsShrink}. Property \ref{ThmItem:Graph} for the graph $G^k$ is consequence of Lemma~\ref{Lemma:Graph}.
 In order to prove part \ref{ThmItem:Order}, note that $v_2^1=0$. If $v_1^1=v_3^1=1$, the claim follows from \ref{ThmItem:MedianSeparates}. Otherwise, there is nothing to prove.

	\paragraph{From $k-1$ to $k$:} Assume (induction assumption) we have constructed events $\xi^1\subset \xi^2\subset \dots\subset \xi^{k-1}$ that satisfy the theorem (up to step $k-1$). From now on we work on the event $\xi^{k-1}$. Algorithm \ref{Algo:Envelope} is applied on
 \vspace{-0.5cm}
	\begin{itemize}
		\item $Y^{(3k-2)}$, an unused sample from observation model \eqref{Eq:Model},
		\item $\mathcal{E}^{k-1}$, which by induction hypothesis (i.e.~\ref{ThmItem:Partition}) is a partition of $[n]$,
		\item $v^{k-1}$, for which it holds that
		\begin{itemize}
			\item $v_t^{k-1}=1 \ \Rightarrow \ \lambda_1|E_t^{k-1}|>4\rho^2/h^2$ by definition,
			\item $v_t^{k-1}=v_{t'}^{k-1}=1$ for $s<s'$ implies $\pi(i)<\pi(i')$ for $i\in E_t^{k-1}$, $i'\in E_{t'}^{k-1}$ by induction hypothesis (i.e.~\ref{ThmItem:Order}).
		\end{itemize}
	\end{itemize}
	This implies that, conditionally to $\xi^{k-1}$, the input $Y^{(3k-2)}$, $\mathcal{E}^{k-1}$ and $v^{k-1}$ satisfies Property~\ref{Assum2}. So, as in Section \ref{Sec:Questions}, we can construct an event $\Xi^{k-1}_{0}$ with
	\begin{align*}
		\mathbb{P}(\Xi_0^{k-1}\ |\ \xi^{k-1})\geq 1-\frac{\delta}{12\lceil\log_2(n\vee d)\rceil (n\vee d)^{1/2}} 
	\end{align*}
	by Lemma~\ref{Lemma:GoodEventEnvelopes} and such that, thanks to Lemma~\ref{Lemma:Envelopes}, it holds  on $\xi_0^{k-1}\coloneqq\xi^{k-1}\cap\Xi_0^{k-1}$ that
		\begin{align}
		Q^*(E_t^{k-1})\subseteq Q_t^{k-1}, \hspace{.5cm}\forall t\in\{1,2,\dots,r_{k-1}\} \text{ with } v_t^{k-1} =1,\label{Eq:EnvelopeInThm} \text{ and } \sum_{t=1,\ v_t^{k-1}=1}^r |Q_s|<3d\enspace .
	\end{align}
	This proves part \ref{ThmItem:Envelope} on $\xi_0^{k-1}$\enspace .
	In order to prove part \ref{ThmItem:Active}, consider $t_1$ to be the smallest $t\in \{1,2,\dots, r_{k-1}\}$ with $w_{t}^{k-1}=1$. Conditionally to $\xi^{k-1}_0$, we apply Algorithm \ref{Algo:Trisect} to
	\begin{itemize}\vspace{-0.5cm}
		\item $Y^{(3k-1)}$ and $Y^{(3k)}$, which are independent samples from observation model \eqref{Eq:Model},
		\item $Q^*(E_{t_1}^{k-1})\subseteq Q_{t_1}^{k-1}$ by \eqref{Eq:EnvelopeInThm},
		\item $\lambda_1(|E_{t_1}^{k-1}|\wedge |Q_{t_1}^{k-1}|)> 4\rho^2/h^2$ by the induction assumption (i.e.~\ref{ThmItem:Envelope}),
		\item $G_{t_1}^k=G^{k-1}$ with $G_{i,i'}=1\ \Rightarrow \pi(i)<\pi(i')$ for all $i,i'\in [n]$ by the induction assumption (i.e.~\ref{ThmItem:Graph}).
	\end{itemize} 
 \vspace{-0.5cm}
 Then, conditionally to $\xi^{k-1}_0$, Property~\ref{Assum1} is again satisfied by $Y^{(3k-1)}$, $Y^{(3k)}, E_{t_1}^{k-1}$, $Q_{t_1}^{k-1}$ and $G_{t_1-1}^k$. Consider the concentration inequalities of the form
	\begin{align*}
		\frac{1}{|E_{t_1}^{k-1}|}\left|\sum_{i\in E_{t_1}^{k-1}} Y_{ij}^{(3k-1)}-\lambda_1 M_{ij}\right|\leq \rho\sqrt{\lambda_1/|E_{t_1}^{k-1}|} \hspace{.5cm} \forall j\in Q_{t_1}^{k-1}\enspace ,\notag\\
		\intertext{and}
		\frac{1}{|Q'|}\left|\sum_{j\in Q'} Y_{ij}^{(3k)}-\lambda_1 M_{ij}\right|\leq \rho\sqrt{\lambda_1/|Q'|} \hspace{.5cm} \forall Q'\in\mathcal{Q}_{t_1}^{k-1},\ i\in E_{t_1}^{k-1}\enspace .
	\end{align*}
	Again, $\mathcal{Q}_{t_1}^{k-1}$ is defined as $\mathcal{Q}$ in Section \ref{Sec:AnalysisTrisection} by sets as in Equation~\eqref{Eq:DefLeftRight}, this time with respect to $E=E_{t_1}^{k-1}$. Lemma~\ref{Lemma:GoodEvent} again tells us that given $\xi_0^k$ the event $\Xi_{t_1}^{k-1}$ on which these inequalities hold satisfies
	\begin{align*}
		\mathbb{P}(\Xi_{t_1}^{k-1}\ | \ \xi_0^{k-1})\geq 1-\frac{|Q_{t_1}^{k-1}|\delta}{12\lceil \log_2(n)\rceil d}\enspace .
	\end{align*}
	We define $\xi_{t_1}^{k-1}\coloneqq \xi_0^{k-1}\cap \Xi_{t_1}^{k-1}$ and obtain that conditional to $\xi_{t_1}^{k-1}$, Lemma~\ref{Lemma:PropertiesTrisection} implies \ref{ThmItem:TrisectIsPartition} and \ref{ThmItem:MedianSeparates}. Since $w_{t_1}=1$, we know that on $\xi^{k-1}$, the set $E_{t_1}^{k-1}$ itself results from a trisection such that by induction hypothesis (i.e.~\ref{ThmItem:SetsShrink}) $|E_{t_1}^{k-1}|\leq 2^{-(k-1)}n$. With \ref{ThmItem:MedianSeparates} we obtain again that 
	\begin{align*}
		|O(E_{t_1}^{k-1})|\vee |I(E_{t_1}^{k-1})|\leq \frac{1}{2}|E_{t_1}^{k-1}|\leq 2^{-k}n\enspace , 
	\end{align*}
	which yields \ref{ThmItem:SetsShrink}. The properties \ref{ThmItem:Graph} hold for $G_{t_1}^{k-1}$ by Lemma~\ref{Lemma:Graph}.

	We then move on to the smallest $t_2>t_1$ with $w_{t_2}=1$. It is worth mentioning, that although we use again $Y^{(3k-1)}$ and $Y^{(k)}$ for the trisection, since $E_{t_2}^{k-1}$ and $E_{t_1}^{k-1}$ are disjoint, these observations can still be considered as unused. Then, together with Lemma~\ref{Lemma:Graph}, conditional to $\xi_{t_1}^{k-1}$, Property~\ref{Assum1} again holds for $Y^{(3k-1)}$, $Y^{(3k)}$, $E_{t_2}^{k-1}$, $Q_{t_2}^{k-1}$ and $G_{t_2-1}^k$. We then define $\Xi_{t_2}^{k-1}$ and $\xi_{t_2}^{k-1}\coloneqq \xi_{t_1}^{k-1} \cap \Xi_{t_2}^{k-1}$ as in the step before and check that all properties of \ref{ThmItem:Active} hold when we trisect $E_{t_2}^{k-1}$, and also \ref{ThmItem:Graph} for the updated graph $G_{t_2}^{k-1}$ follows in the same manner.

	By iterating over all $t$ with $w_t^{k-1}=1$ until the largest such index $t_{\alpha}$, we obtain an event $\xi^k=\xi_{t_{\alpha}}^{k-1}$ such that \ref{ThmItem:Active} holds uniformly and the graph $G^k\coloneqq G_{t_{\alpha}}^k$ that satisfies by iteration \ref{ThmItem:Graph}. Note that indeed, with \ref{ThmItem:Envelope}, it holds
	\begin{align*}
	\mathbb{P}(\xi^k)&= \mathbb{P}(\xi_{t_{\alpha-1}}^{k-1}\cap\Xi_{t_{\alpha}}^{k-1} )=\mathbb{P}(\xi_{t_{\alpha-1}}^{k-1})\cdot\mathbb{P}(\Xi_{t_{\alpha}}^{k-1}\ |\ \xi_{t_{\alpha-1}}^{k-1} )\geq \mathbb{P}(\xi_{t_{\alpha-1}}^{k-1})-\mathbb{P}((\Xi_{t_{\alpha}}^{k-1})^C\ |\ \xi_{t_{\alpha-1}}^{k-1} )\\
	&=\cdots\geq \mathbb{P}(\xi^{k-1})-\mathbb{P}((\Xi_0^{k-1})^C\ |\ \xi^{k-1}) -\mathbb{P}((\Xi_{t_1}^{k-1})^C\ | \ \xi_0^{k-1}) -\sum_{a=2}^{\alpha}\mathbb{P}((\Xi_{t_a}^{k-1})^C\ | \ {\xi}_{t_{a-1}}^{k-1}) \\
	&\geq 1 - \frac{(k-1)\delta}{2\lceil \log_2(n\vee d)\rceil}-\frac{\delta}{12\lceil \log_2(n\vee d)\rceil (n\vee d)^{1/2}}-\frac{\sum_{a=1}^{\alpha}|Q_{t_a}^{k-1}|\delta}{12\lceil \log_2(n\vee d)\rceil d} \\
	& \geq 1 - \frac{(k-1)\delta}{2\lceil \log_2(n\vee d)\rceil}-\frac{\delta}{4\lceil \log_2(n\vee d)\rceil }-\frac{3d\delta}{12\lceil \log_2(n\vee d)\rceil d} \geq 1-\frac{k\delta}{2\lceil \log_2(n\vee d)\rceil}
	\end{align*}
	and that, on $\xi^k$, the graph $G^k$ satisfies \ref{ThmItem:Graph}.

In order to prove claim \ref{ThmItem:Order}, consider $t$, $t'\in \{1,2,\dots,r_{k-1}\}$ such that $i\in E_t^{k-1}$ and $i'\in E_{t'}^{k-1}$. By definition, $i\in E_s^k$ and $i'\in E_{s'}^k$ with $s<s'$ and $v_s^k=v_{s'}^k=1$ imply $t\leq t'$ and $v_t^{k-1}=v_{t'}^{k-1}=1$. If $t<t'$, the claim $\pi(i)<\pi(i')$ follows directly by assumption. Otherwise, we know $i\in O(E_t^{k-1})$ and $i'\in I(E_t^{k-1})$ and \ref{ThmItem:MedianSeparates} proves the claim.

Claim \ref{ThmItem:Partition} follows from the fact that we assumed that $\mathcal{E}^{k-1}$ is a partition of $[n]$. By \ref{ThmItem:TrisectIsPartition}, $\mathcal{E}^k$ is obtained by replacing all active sets in $\mathcal{E}^{k-1}$ by a trisection, which in total yields again a partition of $[n]$.

\paragraph{Number of iterations:} Assume that for some $k$ the vector $v^k$ is non-zero. By \ref{ThmItem:SetsShrink} we know that for $t$ with $v_t^k=1$ it holds $|E_t^k|\leq 2^{-k}n\leq 1$. By definition, $|E_t^k|>4\rho^2/\lambda_1 h^2>1$. Therefore, $k\leq \log_2(n)$. This implies, that $K\leq \lceil \log_2(n)\rceil$ with $v_{K-1}\neq \mathbf{0}$ and $v_K=\mathbf{0}$ exists. So the algorithm terminates after round $K$ on the event $\xi^K$, which holds with probability at least $1-\delta/2$.
\end{proof}

\begin{proof}[Proof of Corollary~\ref{Cor:PropertiesTree}] 
On $\xi^K$, we assume that concentration inequalities of the form \eqref{Eq:ConcentrationInTree} hold. Therefore, Corollary~\ref{CorItem:InO}--\ref{CorItem:InP} follow like in Lemma~\ref{Lemma:TrisectionEmpiricalMedian}. Corollary~\ref{CorItem:TrisectionSubsets} and \ref{CorItem:PTildeSubset} are an adaptation of Lemma~\ref{Cor:SubsetOfTilde}. The last point Corollary~\ref{CorItem:ErrorBound} follows like Lemma~\ref{Lemma:Blocks}.
\end{proof}
\subsection{Proofs of Lemmas~\ref{Lemma:PermutationInInterval}--\ref{lem:upper_boud_error_permutation}}\label{Sec:PermutationError}

\begin{proof}[Proof of Lemma~\ref{Lemma:PermutationInInterval}]
	For $K$, the statement follows immediately from the definition of $\hat{\pi}$. Now consider any $k <K$, and $t < s$. By the hierarchical construction of the sorting tree (see Figure~\ref{Fig:Tree}), it follows that $E_t^{k}$ (resp. $E_s^k$) is the disjoint union of $E_{l}^{K}$ for indices $l\in \mathcal{L}_t$ (resp. $l\in \mathcal{L}_t$) with $\max\{l: l\in \mathcal{L}_t\}< \min\{l: l\in \mathcal{L}_s\}$. For any $i\in E_l^{k}$ and $j\in E_s^{k}$, we have $\hat{\pi}(i)< \hat{\pi}(j)$. It follows that $\hat{\pi}$ is coherence with the partition $\mathcal{E}^k$ and the result follows. 
\end{proof}

Recall that, for $E\subseteq [n]$, we denote $\tilde{E}$ as its completion with respect to $\pi$. 
\begin{align}
\tilde{E}\coloneqq\left\{i\in[n]:\ \min_{i'\in E}\pi(i')\leq \pi(i)\leq \max_{i'\in E}\pi(i')\right\}\enspace . \label{Eq:TildeSet}
\end{align}

\begin{proof}[Proof of Lemma~\ref{Lemma:TauNotInO}]
We consider in the entire proof that we are on the event  $\xi^K$.	Note that $\hat{i}\notin \tilde{E}_s^k$ implies by definition \eqref{Eq:TildeSet} either 
that $\pi(\hat{i})>\pi(i')\hspace{.5cm} \forall i'\in E_s^k $ or that $\pi(\hat{i})<\pi(i')\quad  \forall i'\in E_s^k\enspace$. Assume w.l.o.g.~that
\begin{align}
\pi(\hat{i})>\pi(i')\quad \forall i'\in E_s^k \label{Eq:iHatLarger}\enspace . 
\end{align}
holds. Lemma~\ref{Lemma:PermutationInInterval} implies that 
		$\pi(\hat{i})=\hat{\pi}(i)\leq \sum_{s'\leq s}|E_{s'}^k|$.
	Since $\pi$ is bijective on $[n]$ (as it is a permutation), it follows that there exists $i''\in E_{s'}^k$ with $s'<s$ such that $\pi(i'')>\pi(\hat{i})$ holds. Moreover, \eqref{Eq:iHatLarger} implies
	\begin{align}
		\pi(i)<\pi(\hat{i})<\pi(i'')\enspace . \label{Eq:iSmalleri''}
	\end{align}
	By definition of the algorithm, there must be a maximal number $l<k$ such that $u\in \{1,2,\dots,r_l\}$ exists with $i,i''\in E_u^l$. By this maximality assumption, $E_u^l$ must be an active set and $i$ and $i''$ must end up in different sets of the partition after step $l$ (in terms of the sets $O(E_u^l),P(E_u^l),I(E_u^l)$) of the corresponding trisection.

	Since by assumption $E_s^k$ has the form $O(E_t^{k-1})$ or $I(E_t^{k-1})$ for some $t\in\{1,2,\dots,r_{k-1}\}$, the definition of the algorithm does not allow $i\in P(E_u^l)$. Therefore, because $s'<s$, only $i\in I(E_u^l)$ is possible. Since we stated in Equation~\eqref{Eq:iSmalleri''} that $\pi(i)<\pi(i'')$, and since by definition of $E_u^l$, $i$ and $i''$ must end up in different sets after step $l$, by Theorem~\ref{ThmItem:MedianSeparates}, it must then hold that 
 \begin{align}
 i''\in P(E_u^l)\enspace .\label{Eq:i''inP}    
 \end{align}
	Again with Theorem~\ref{ThmItem:MedianSeparates}, we can extend the inequalities~\eqref{Eq:iSmalleri''} to 
	\begin{align*}
			\pi(\overline{i}_u^l)<\pi(i)<\pi(\hat{i})<\pi(i'')\enspace ,
	\end{align*}
	which implies that $i,\hat{i}\in \tilde{P}(E_u^l)$ since both $\overline{i}_u^l$ and $i''$ belong to $P(E_u^l)$ (Theorem~\ref{ThmItem:MedianSeparates}) and by definition of $\tilde{P}(E_u^l)$ in \eqref{Eq:TildeSetsTrisection}.	If $\hat{i}\in E_u^l$, then Corollary~\ref{CorItem:PTildeSubset} implies that $i,\hat{i}\in \overline{P}(E_u^l)$, which completes the proof. 
 
 It remains to consider the case where $\hat{i}\notin E_u^l$.  Consider the maximal $l'<l$ such that $u\in \{1,2,\dots,r_{l'}\}$ exists with $i,\hat{i}\in E_{u'}^{l'}$. Again $E_{u'}^{l'}$ must be an active set and after the trisection $i$ (and therefore $i''$) must be contained in an active set. So in particular $i,i''\in O(E_{u'}^{l'})$ or $i,i''\in I(E_{u'}^{l'})$. Note that in both cases, the maximality assumption on $l'$ and Theorem~\ref{ThmItem:MedianSeparates}, together with \eqref{Eq:iSmalleri''} imply $\hat{i}\in P(E_{u'}^{l'})$.

	The proof is completed if we can show $i\in \overline{P}(E_{u'}^{l'})$. In the case $i\in I(E_{u'}^{l'})$, note that, by Theorem~\ref{ThmItem:MedianSeparates} together with Equation~\eqref{Eq:iSmalleri''}, it follows that 
	\begin{align*}
		\pi(\overline{i}_{u'}^{l'})<\pi(i)<\pi(\hat{i})\enspace ,
	\end{align*}
	and therefore $i\in \tilde{P}(E_{u'}^{l'})$. Since $i\in E_{u'}^{l'}$, the statement follows from Corollary~\ref{CorItem:PTildeSubset}.

 Hence, it only remains  to consider the case where $i\in O(E_{u'}^{l'})$. To show that $i\in \overline{P}(E_{u'}^{l'})$, by definition of $\overline{P}(E_{u'}^{l'})$,  we only need to establish that 
	\begin{align}\label{eq:objective:final}
		\frac{1}{|Q'|}\left|\sum_{j\in Q'}\Big[Y_{ij}^{(3l')}-Y_{\mmed{u'}{l'}{Q'}j}^{(3l')}\Big]\right|\leq 8\rho\sqrt{\lambda_1/|Q'|}\hspace{0.5cm}\forall Q'\in \mathcal{Q}_{u'}^{l'}\enspace .
	\end{align}
Consider such a set $Q'\in \mathcal{Q}(E_{u'}^{l'})$. Note that 
	\begin{align*}
		\frac{1}{|Q'|}\sum_{j\in Q'}\Big[Y_{\mmed{u'}{l'}{Q'}j}^{(3l')}-Y_{ij}^{(3l')} \Big]>2\rho\sqrt{\lambda_1/|Q'|}\enspace ,
	\end{align*}
	leads to $i\in I(E_{u'}^{l'})$ by Corollary~\ref{CorItem:InI}, a contradiction. So we know that
	\begin{align*}
		\frac{1}{|Q'|}\sum_{j\in Q'}\Big[Y_{\mmed{u'}{l'}{Q'}j}^{(3l')}-Y_{ij}^{(3l')}\Big]\leq2\rho\sqrt{\lambda_1/|Q'|}\enspace ,
	\end{align*}
	so it remains to be shown that
	\begin{align*}
		\frac{1}{|Q'|}\sum_{j\in Q'}\Big[Y_{ij}^{(3l')}-Y_{\mmed{u'}{l'}{Q'}j}^{(3l')}\Big]\leq8\rho\sqrt{\lambda_1/|Q'|}\enspace .
	\end{align*}

	We have seen earlier in the proof that also $i\in I(E_u^l)$ holds. We can adapt the proof of Lemma~\ref{Lemma:TrisectionEmpiricalMedian}: if for some $Q'\in \mathcal{Q}(E_{u'}^{l'})$ we have
		\begin{align*}
		\frac{1}{|Q'|}\sum_{j\in Q'}\Big[Y_{ij}^{(3l')}-Y_{\mmed{u}{l}{Q'}j}^{(3l')}\Big]>2\rho\sqrt{\lambda_1/|Q'|}\enspace ,	
	\end{align*}
	this would mean $G^{l'}_{ii'}=1$ for at least half of the $i'\in E_u^l\subset E_{u'}^{l'}$. Since the algorithm does not change entries of the graph that are already set to $1$, this would imply $i\in O(E_u^l)$, a contradiction. So we also have
	\begin{align}\label{eq:upper_1}
		\frac{1}{|Q'|}\sum_{j\in Q'}\Big[Y_{ij}^{(3l')}-Y_{\mmed{u}{l}{Q'}j}^{(3l')}\Big]\leq2\rho\sqrt{\lambda_1/|Q'|}\enspace .	
	\end{align}
	In the same spirit, we can show that
		\begin{align}\label{eq:upper_2}
		\frac{1}{|Q'|}\left|\sum_{j\in Q'}Y_{\mmed{u}{l}{Q'}j}^{(3l')}-Y_{i''j}^{(3l')}\right|\leq2\rho\sqrt{\lambda_1/|Q'|}\enspace ,	
	\end{align}
	because otherwise we would end up with a contradiction to $i''\in P(E_u^l)$ from \eqref{Eq:i''inP} by Corollary~\ref{CorItem:InP}.

	Recall from Equation~\eqref{Eq:iSmalleri''} that $\pi(\hat{i})<\pi(i'')$. The bi-isotonicity assumption and \eqref{Eq:ConcentrationInTree} yield
	\begin{align}\label{eq:upper_3}
		\frac{1}{|Q'|}\sum_{j\in Q'}\Big[Y_{i''j}^{(3l')}-Y_{\hat{i}j}^{(3l')}\Big]\leq\frac{1}{|Q'|}\sum_{j\in Q'}\Big[Y_{i''j}^{(3l')}-\lambda_1M_{i''j}^{(3l')}\Big]+\frac{1}{|Q'|}\sum_{j\in Q'}\Big[\lambda_1M_{\hat{i}j}^{(3l')}-Y_{\hat{i}j}^{(3l')}\Big]\leq 2\rho \sqrt{\lambda_1/|Q'|}\enspace .
	\end{align}

	Finally, from Lemma~\ref{Lemma:TrisectionEmpiricalMedian}, since we have $\hat{i}\in P(E_{u'}^{l'})$, it follows
	\begin{align}\label{eq:upper_4}
		\frac{1}{|Q'|}\left|\sum_{j\in Q'}\Big[Y_{\hat{i}j}^{(3l')}-Y_{\mmed{u'}{l'}{Q'}j}^{(3l')}\Big]\right|\leq 2\rho\sqrt{\lambda_1/|Q'|}\enspace .
	\end{align}
 Hence, we derive from~(\ref{eq:upper_1}--\ref{eq:upper_4}) that 
 \[
		\frac{1}{|Q'|}\left|\sum_{j\in Q'}\Big[Y_{ij}^{(3l')}-Y_{\mmed{u'}{l'}{Q'}j}^{(3l')}\Big]\right|\leq 8\rho\sqrt{\lambda_1/|Q'|}\enspace .
 \]
Then consider all $Q'\in \mathcal{Q}_{u'}^{l'}$, we obtain~\eqref{eq:objective:final}, which concludes the proof.

\end{proof}
\begin{proof}[Proof of Lemma~\ref{Lemma:TauInP}]
	If $\hat{i}\in E_s^k=P(E_t^{k-1})$, by Corollary~\ref{CorItem:TrisectionSubsets} there is nothing to prove. So assume w.l.o.g. that $\hat{i}\in E_{s_1}^k$ with $s_1<s$. We will consider the cases $\pi(i)<\pi(\hat{i})$ and $\pi(i)>\pi(\hat{i})$ separately.
	
	\paragraph{Case 1: $\pi(i)<\pi(\hat{i})$.} 	Consider $l$ maximal such that $u\in\{1,2,\dots,r_u^l\}$ exists with $i,\hat{i}\in E_u^l$. 
	Assume first, that $l=k-1$. Since $i\in P(E_t^{k-1})$ and we just assumed $\hat{i}\in E_t^{l}\cap \bigcup_{s'<s}E_{s'}^k$, only $\hat{i}\in O(E_t^{k-1})$ is possible. But then, by Theorem~\ref{ThmItem:MedianSeparates}, we have 
		$\pi(i)<\pi(\hat{i})<\pi(\overline{i}_t^{k-1})$,
	so $\hat{i}\in \tilde{P}(E_t^{k-1})\cap E_t^{k-1}\subseteq \overline{P}(E_t^{k-1})$ by Corollary~\ref{CorItem:PTildeSubset}.

	If $l<k-1$, then Theorem~\ref{ThmItem:MedianSeparates} does not allow that $\hat{i}\in O(E_u^l)$ and $i\in I(E_u^l)$. Furthermore, $i\in P(E_u^l)$ would lead to a contradiction to $w_t^{k-1}=1$. So from $s_1<s$, only $\hat{i}\in P(E_u^l)$ and $i\in I(E_u^l)$ is possible. From Theorem~\ref{ThmItem:MedianSeparates} it follows that $\pi(\overline{i}_u^l)<\pi(i)<\pi(\hat{i})$. Hence, we conclude that $i\in \tilde{P}(E_u^l)\cap E_u^l\subseteq \overline{P}(E_u^l)$ by Corollary~\ref{CorItem:PTildeSubset}.

	\paragraph{Case 2: $\pi(i)>\pi(\hat{i})$.} Note that we assumed $\hat{i}\in E_{s_1}^k$ with $s_1<s$, but that from Lemma~\ref{Lemma:PermutationInInterval} it follows
	\begin{align*}
		\pi(\hat{i})=\hat{\pi}(i)> \sum_{s'<s}|E_{s'}^k|\enspace . 
	\end{align*}
	By the bijectivity of $\pi$ and the Pigeonhole principle, this means that there must be $s_2\geq s$ and $i'\in E_{s_2}^k$ such that 
	\begin{align*}
		\pi(i')\leq \sum_{s'<s}|E_{s'}^k|<\pi(\hat{i})\enspace .
	\end{align*}
	Consider  now $l$ maximal such that $u\in\{1,2,\dots,r_u^l\}$ exists with $i',\hat{i}\in E_u^l$.

	If $l=k-1$, $s_1<s\leq s_2$ implies $i',\hat{i}\in E_t^{k-1}$ and $\hat{i}\in O(E_t^{k-1})$. By Theorem~\ref{ThmItem:MedianSeparates} this means $i'\in P(E_t^{k-1})$. Recall that
	\begin{align*}
		\pi(i')<\pi(\hat{i})<\pi(i)\enspace ,
	\end{align*}
	so $\hat{i}\in \tilde{P}(E_t^{k-1})\cap E_t^{k-1}\subseteq \overline{P}(E_t^{k-1})$ by Corollary~\ref{CorItem:PTildeSubset}.

	If $l<k-1$, note that again, by Theorem~\ref{ThmItem:MedianSeparates}, $\hat{i}\in O(E_u^l)$ and $i'\in I(E_u^l)$ is impossible, since we have $\pi(i')<\pi(\hat i)$. This leaves us with the possible cases $\hat{i}\in O(E_u^l)$,  $i'\in P(E_u^l)$ or $\hat{i}\in P(E_u^l)$, $i'\in I(E_u^l)$. Note that $s_1<s\leq s_2$ implies $i\in E_u^l$.

	If $\hat{i}\in O(E_u^l)$ and $i'\in P(E_u^l)$, then by our assumptions and Theorem~\ref{ThmItem:MedianSeparates} it follows
	\begin{align*}
		\pi(i')<\pi(\hat{i})<\pi(i)<\pi(\overline{i}_u^l)
	\end{align*}
	and therefore $i\in \tilde{P}(E_u^l)\cap E_u^l\subseteq \overline{P}(E_u^l)$ with Corollary~\ref{CorItem:PTildeSubset}. 
	
	Hence, it remains to consider the subcase where $\hat{i}\in P(E_u^l)$ and $i'\in I(E_u^l)$. Since $s_1<s$, we know that $i\in I(E_u^l)$. Consider $Q'\in \mathcal{Q}_u^l$. We will prove that 
	\begin{align*}
		\frac{1}{|Q'|}\left|\sum_{j\in Q'}Y^{(3l)}_{\mmed{u}{l}{Q'}j}-Y^{(3l)}_{ij}\right|\leq 8\rho\sqrt{\lambda_1/|Q'|}\enspace .
	\end{align*}
	Indeed, this property will imply that $i\in \overline{P}(E_u^l)$ and conclude the proof.

	We have seen that $i\notin O(E_u^l)$, so Corollary~\ref{CorItem:InO} implies
	 	\begin{align*}
	 	\frac{1}{|Q'|}\sum_{j\in Q'}\Big[Y^{(3l)}_{ij}-Y^{(3l)}_{\mmed{u}{l}{Q'}j}\Big]\leq 2\rho\sqrt{\lambda_1/|Q'|}\enspace ,
	 \end{align*}
	 and we are left with the proof of
	 	\begin{align}\label{eq:objective:sub}
	 	\frac{1}{|Q'|}\sum_{j\in Q'}\Big[Y^{(3l)}_{\mmed{u}{l}{Q'}j}-Y^{(3l)}_{ij}\Big]\leq 8\rho\sqrt{\lambda_1/|Q'|}\enspace .
	 \end{align}
	 Since $\hat{i}\in P(E_u^l)$, Corollary~\ref{CorItem:InI} yields 
	 \begin{align*}
	 	\frac{1}{|Q'|}\left|\sum_{j\in Q'}Y^{(3l)}_{\mmed{u}{l}{Q'}j}-Y^{(3l)}_{\hat{i}j}\right|\leq 2\rho\sqrt{\lambda_1/|Q'|}\enspace .
	 \end{align*}
Since  $\pi(i')<\pi(\hat{i})$, the bi-isotonicity assumption and Equation~\eqref{Eq:ConcentrationInTree} imply that 
	 \begin{align*}
	 	\frac{1}{|Q'|}\sum_{j\in Q'}\Big[Y^{(3l)}_{\hat{i}j}-Y^{(3l)}_{i'j}\Big]\leq 2\rho\sqrt{\lambda_1/|Q'|}\enspace .
	 \end{align*}
	Hence, we have shown~\eqref{eq:objective:sub} if we can prove that  
	 	 \begin{align}
	 	\frac{1}{|Q'|}\sum_{j\in Q'}\Big[Y^{(3l)}_{i'j}-Y^{(3l)}_{ij}\Big]\leq4\rho\sqrt{\lambda_1/|Q'|}\enspace . \label{Eq:DistanceiAndi'}
	 \end{align}

	 Observe that 
	 	 \begin{align*}
	 	\frac{1}{|Q'|}\sum_{j\in Q'}\Big[Y^{(3l)}_{\mmed{t}{k-1}{Q'}j}-Y^{(3l)}_{ij}\Big]\leq 2\rho\sqrt{\lambda_1/|Q'|}
	 \end{align*}
	 must hold. Otherwise, one could argue like in the proof of Lemma~\ref{Lemma:TrisectionEmpiricalMedian} and would obtain $i\in I(E_t^{k-1})$.

	 In the case $i'\in E_t^{k-1}$, note that $s_2\geq s$ implies $i'\notin O(E_t^{k-1})$. By an adaptation of the proof of Lemma~\ref{Lemma:TrisectionEmpiricalMedian}, this implies
	 	 	 \begin{align*}
	 	\frac{1}{|Q|'}\sum_{j\in Q'}\Big[Y^{(3l)}_{i'j}-Y^{(3l)}_{\mmed{t}{k-1}{Q'}j}\Big]\leq 2\rho\sqrt{\lambda_1/|Q'|}
	 \end{align*}
	 and we obtain Equation~\eqref{Eq:DistanceiAndi'}.

	 In the case $i'\notin E_t^{k-1}$, consider $l'$ maximal such that $u'\in\{1,2,\dots, r_{l'}\}$ exists with $i,i'\in E_{u'}^{l'}$. Again, by Theorem~\ref{ThmItem:MedianSeparates} and 
	 \begin{align*}
	 	\pi(i')<\pi(\hat{i})<\pi(i)\enspace ,
	 \end{align*}
	 it cannot hold $i\in O(E_{u'}^{l'})$ and $i'\in I(E_{u'})$. Since $w_t^{k-1}=1$ we can also exclude $i\in P(E_{u'}^{l'})$. So $s<s_2$ only permits $i\in O(E_{u'}^{k'})$ and $i'\in P(E_{u'}^{k'})$.

	 Like before, one can show 
	 	 	 \begin{align*}
	 	\frac{1}{|Q'|}\sum_{j\in Q'}\Big[Y^{(3l)}_{\mmed{u'}{l'}{Q'}j}-Y^{(3l)}_{ij}\Big]\leq 2\rho\sqrt{\lambda_1/|Q'|}\ , 
	 \end{align*}
	 otherwise $i$ would be in $I(E_{u'}^{l'})$. Similarly, we have 
	 \begin{align*}
	 	\frac{1}{|Q'|}\sum_{j\in Q'}\Big[Y^{(3l)}_{\mmed{u'}{l'}{Q'}j}-Y^{(3l)}_{ij}\Big]\leq 2\rho\sqrt{\lambda_1/|Q'|}
	 \end{align*}
	 otherwise, $i'$ would be in $O(E_{u'}^{l'})$). This shows Equation~\eqref{Eq:DistanceiAndi'} and completes the proof.
\end{proof}

 \begin{proof}[Proof of Lemma~\ref{lem:upper_boud_error_permutation}]
First, for $i_1,i_2\in [n]$, let us define
	\begin{align*}
		R_{p,h}(i_1,i_2)\coloneqq \left|\{j\in[d]:\ M_{i_1j}\geq p+h,\ M_{i_2j}\leq p-h\}\right|+ \left|\{j\in[d]:\ M_{i_1j}\leq p-h,\ M_{i_2j}\geq p+h\}\right|\enspace .
	\end{align*}

	Also, consider $\tau\coloneqq \pi^{-1}\circ \hat{\pi}$ and note that $\tau$ as permutation is bijective, as well as $\eta$. Therefore, it holds that 
    \begin{align}
        \mathcal R_{p,h}(\hat{\pi})&=\sum_{i=1}^n\left[\left|\{j\in[d]:\  M_{\pi^{-1}(i)\eta^{-1}(j)}\leq p-h,\ M_{\hat\pi^{-1}(i)\eta^{-1}(j)}\geq p+h\}\right|\right. \notag\\
        &\hphantom{=\sum}+\left.\left|\{j\in[d]:\  M_{\pi^{-1}(i)\eta^{-1}(j)}\geq p+h,\ M_{\hat\pi^{-1}(i)\eta^{-1}(j)}\leq p-h\}\right|\right]\notag\\
        &=\sum_{i=1}^n\left[\left|\{j\in[d]:\  M_{ij}\leq p-h,\ M_{\tau(i)j}\geq p+h\}\right|\right. +\left.\left|\{j\in[d]:\  M_{ij}\geq p+h,\ M_{i\tau(j)}\leq p-h\}\right|\right]\notag\\
        &=\sum_{s=1}^{r_K}\sum_{i\in E_s^K}R_{p,h}(i,\tau(i))\notag\\
		&=\sum_{k=1}^K\sum_{\substack{t=1\\ E_t^k\in \mathcal{O}^k\cup\mathcal{I}^k}}^{r_k}\sum_{i\in E_t^k}R_{p,h}(i,\tau(i))+ \sum_{k=1}^K\sum_{\substack{t'=1\\ w_{t'}^{k-1}=1}}^{r_{k-1}}\sum_{i\in P(E_{t'}^{k-1})}R_{p,h}(i,\tau(i))\enspace , \label{Eq:ErrorDecomposition}
	\end{align}
 where, in the last line, we used the property that any set $E_s^{K}$ is either belongs to some $\mathcal{O}^k\cup\mathcal{I}^k$ or is of the form $P(E_{t'}^{k-1})$ for some $t$ and $k\leq K$. 
    We will first provide  upper bounds for 
    \begin{align}
        \sum_{i\in E_t^k}R_{p,h}(i,\tau(i)) \hspace{.5cm}\text{with}\hspace{.5cm}E_t^k\in \mathcal{O}^k\cup\mathcal{I}^k \label{Eq:CaseIandO}\\ \intertext{and}
        \sum_{i\in P(E_{t'}^{k-1})}R_{p,h}(i,\tau(i)) \hspace{.5cm}\text{with}\hspace{.5cm}w_{t'}^{k-1}=1\enspace .\label{Eq:CaseP}
    \end{align}

	\paragraph{Upper Bound for \eqref{Eq:CaseIandO}.} W.l.o.g., let us consider $E_t^k=O(E_{t'}^{k-1})$ with $w_t^k=0$. Then, 
	\begin{align*}
		\sum_{i\in E_t^k}R_{p,h}(i,\tau(i))=\sum_{i\in O(E_{t'}^{k-1}),\ \tau(i)\in \tilde{O}(E_{t'}^{k-1})}R_{p,h}(i,\tau(i))+\sum_{i\in O(E_{t'}^{k-1}),\ \tau(i)\notin \tilde{O}(E_{t'}^{k-1})}R_{p,h}(i,\tau(i))\enspace . 
	\end{align*}
	Recall that we defined 
	\begin{align*}
		Q^*(E_t^k)=\{j\in[d]:\ \max_{i\in E_t^k}M_{ij}\geq p+h,\ \min_{i\in E_t^k}M_{ij}\leq p-h\}\enspace ,
	\end{align*}
	and consider $j\in [d]$, $i_1\in O(E_{t'}^{k-1})=E_t^k$ and $i_2\in \tilde{O}(E_{t'}^{k-1})$. By the bi-isotonicity assumption,
	\begin{align*}
		M_{i_2j}\leq p-h \Rightarrow \min_{i\in E_t^k}M_{ij}\leq p-h \quad \text{ and }\quad 
		M_{i_2j}\geq p+h \Rightarrow \max_{i\in E_t^k}M_{ij}\geq p+h\enspace .
	\end{align*}
	So if 
	\begin{align*}
		M_{i_1j}\geq p+h,\ M_{i_2j}\leq p-h \hspace{.5cm}\text{or}\hspace{.5cm}M_{i_1j}\leq p-h,\ M_{i_2j}\geq p+h\enspace ,
	\end{align*}
	we deduce that $j\in Q^*(E_t^k)$. This implies that  $R_{p,h}(i_1,i_2)\leq |Q^*(E_t^k)|$ and consequently
	\begin{align*}
		\sum_{i\in O(E_{t'}^{k-1}),\ \tau(i)\in \tilde{O}(E_{t'}^{k-1})}R_{p,h}(i,\tau(i))\leq |E_t^k|\cdot |Q^*(E_t^k)|\enspace ,
	\end{align*}
	Recall that we assumed $w_t^k=0$ because $E_t^k$ is not to be refined anymore. If $|E_t^k|>4\rho^2/\lambda_1 h^2$, the algorithm computes $Q_t^k$ but it must hold $|Q_t^k|\leq 4\rho^2/\lambda_1 h^2$. By Theorem~\ref{ThmItem:Envelope}, it holds $Q_t^k\subseteq Q^*(E_t^k)$. We conclude that 
		\begin{align*}
		\sum_{i\in O(E_{t'}^{k-1}),\ \tau(i)\in \tilde{O}(E_{t'}^{k-1})}R_{p,h}(i,\tau(i))\leq 4\frac{\rho^2}{\lambda_1 h^2}\cdot \left(|E_t^k|\vee |Q^*(E_t^k)|\right)\enspace .
	\end{align*}

	Next, note that $i\in O(E_{t'}^{k-1})$ but $\tau(i)\notin \tilde{O}(E_{t'}^{k-1})$ implies, by Lemma~\ref{Lemma:TauNotInO}, that it exists $l\leq k$ and $u\in\{1,2,\dots,r_{l-1}\}$ such that $i,\tau(i)\in \overline{P}(E_u^{l-1})$. We can therefore upper bound 
	\begin{align*}
		\sum_{i\in O(E_{t'}^{k-1}),\ \tau(i)\notin \tilde{O}(E_{t'}^{k-1})}R_{p,h}(i,\tau(i))\leq \sum_{l=1}^{K}\sum_{\substack{u=1\\ w_u^{l-1}=1}}^{r_{l-1}}\sum_{\substack{i\in E_t^k\cap\overline{P}(E_u^{l-1})\\ \tau(i)\in \overline{P}(E_u^{l-1})}}R_{p,h}(i,\tau(i))\enspace .
	\end{align*}
	Gathering the two previous inequalities, we conclude that 
	\begin{align}
		\sum_{i\in E_t^k}R_{p,h}(i,\tau(i))\leq 4\frac{\rho^2}{\lambda_1 h^2}\cdot \left(|E_t^k|\vee |Q^*(E_t^k)|\right)+\sum_{l=1}^{K}\sum_{\substack{u=1\\ w_u^{l-1}=1}}^{r_{l-1}}\sum_{\substack{i\in E_t^k\cap\overline{P}(E_u^{l-1})\\ \tau(i)\in \overline{P}(E_u^{l-1})}}R_{p,h}(i,\tau(i))\enspace . \label{Eq:UpperBoundOandI}
	\end{align}

	\paragraph{Upper Bound for \eqref{Eq:CaseP}.} In the case $E_t^k=P(E_{t'}^{k-1})$, we can apply Lemma~\ref{Lemma:TauInP} and obtain again some $l\leq k$ and $u\in\{1,2,\dots,r_{l-1}\}$ such that $i,\tau(i)\in \overline{P}(E_u^{l-1})$. This yields 
	\begin{align}
		\sum_{i\in E_t^k}R_{p,h}(i,\tau(i))\leq \sum_{l=1}^{K}\sum_{\substack{u=1\\ w_u^{l-1}=1}}^{r_{l-1}}\sum_{\substack{i\in E_t^k\cap\overline{P}(E_u^{l-1})\\ \tau(i)\in \overline{P}(E_u^{l-1})}}R_{p,h}(i,\tau(i))\enspace .  \label{Eq:UpperBoundP}
	\end{align}
Consider $E_t^{l-1}$ with $w_t^{l-1}=1$ and an arbitrary bijection $\tilde{\tau}:\overline{P}(E_u^{l-1})\to \overline{P}(E_u^{l-1})$ with $\tilde{\tau}(i)=\tau(i)$ if $\tau(i)\in \overline{P}(E_u^{l-1})$. Then
	\begin{align*}
		\sum_{\substack{i\in \overline{P}(E_u^{l-1})\\ \tau(i)\in \overline{P}(E_u^{l-1})}}R_{p,h}(i,\tau(i))&\leq \sum_{i\in \overline{P}(E_u^{l-1})}R_{p,h}(i,\tilde{\tau}(i))\\
		&= \sum_{j\in [d]} \sum_{i\in \overline{P}(E_t^{k-1})}\mathbbm{1}\{M_{ij}\geq p+h,\ M_{\tilde{\tau}(i)j}\leq p-h\}\\
		&\hphantom{= \sum_{i\in \overline{P}(E_t^{k-1})}\sum_{j\in [d]}}
		+\mathbbm{1}\{M_{ij}\leq p-h,\ M_{\tilde{\tau}(i)j}\geq p+h\}\enspace .
	\end{align*}
	Note that 
	\begin{align*}
	\sum_{i\in \overline{P}(E_t^{k-1})} \mathbbm{1}\{M_{ij}\geq p+h,\ M_{\tilde{\tau}(i)j}\leq p-h\}
	 &\leq |\{i\in \overline{P}(E_t^{k-1}):\ M_{ij}\geq p+h\}|\wedge|\{i\in \overline{P}(E_t^{k-1}):\ M_{\tilde{\tau}(i)j}\leq p+h\}|\\
	&=|\{i\in \overline{P}(E_t^{k-1}):\ M_{ij}\geq p+h\}|\wedge|\{i\in \overline{P}(E_t^{k-1}):\ M_{ij}\leq p+h\}|\enspace .
	\end{align*}
 Recall the definition~\eqref{eq:definition:loss:Rtilde} of $\tilde{\mathcal R}_{p,h,\overline{P}(E_u^{l-1})[d]}$. We obtain
\begin{align}\label{Eq:UpperBoundP2}
 \sum_{\substack{i\in \overline{P}(E_u^{l-1})\\ \tau(i)\in \overline{P}(E_u^{l-1})}}R_{p,h}(i,\tau(i))\leq 2 \tilde{\mathcal R}_{p,h,\overline{P}(E_u^{l-1})[d]}\enspace . 
\end{align}
	\paragraph{Combining the upper bounds.} Combining \eqref{Eq:UpperBoundOandI}, \eqref{Eq:UpperBoundP}, and \eqref{Eq:UpperBoundP2} with \eqref{Eq:ErrorDecomposition} gives us
	\begin{align*}
		\mathcal{R}_{p,h}(\hat{\pi})\leq 4\frac{\rho^2}{\lambda_1 h^2}\sum_{k=1}^K\sum_{\substack{t=1\\ E_t^k\in \mathcal{O}^k\cup\mathcal{I}^k}}^{r_k}|E_t^k|\vee |Q^*(E_t^k)|+2\sum_{l=1}^{K}\sum_{\substack{u=1\\ w_u^{l-1}=1}} \tilde{\mathcal R}_{p,h,\overline{P}(E_u^{l-1})[d]}\enspace .
	\end{align*}
\end{proof}

\section{Proof of Corollary~\ref{cor:error_bound:permutation:multiple}} 

In this section, we explain how the proof of Theorem~\ref{Thm:ErrorBoundRows} easily extends to the case of multiples thresholds and tolerances up to a minor change. 
Consider some $l\in\{1,\dots,m\}$. Note that the only difference for Algorithm~\ALGO{} with multiple thresholds instead of one threshold is that we consider multiple envelope sets, constructed in line~\ref{line:envelope}. To this end, define
\begin{align*}
    Q^*_{p_l,h_l}(E) \coloneqq \{j\in[d]:\ \max_{i\in E}M_{ij}\geq p_l+h_l,\ \min_{i\in E}M_{ij}\leq p_l-h_l\}\enspace ,
\end{align*}
and consider the envelope with respect to $p_l,h_l$ for $E_t^{k-1}$ with $v_t^{k-1}$ as $Q_{t,l}^{k-1}$.
Again, we only use this set for the \texttt{ScanAndUpdate} procedure if it is the case that Property~\ref{Assum1} is fulfilled, and construct each time a new $\mathcal{Q}_{t,l}^{k-1}$ as in \eqref{Eq:DefLeftRight}. Therefore, Lemma~\ref{Lemma:GoodEvent}, Lemma~\ref{Lemma:Graph} and Lemma~\ref{Lemma:PropertiesTrisection} are not effected. It is not hard to see that Lemma~\ref{Lemma:TrisectionEmpiricalMedian} also holds, if our trisection is based on running \texttt{ScanAndUpdate} for multiple envelope sets. One way of thinking is, that the resulting graph is more informative than a graph being updated with respect to a single envelope set only. The first three points of Lemma~\ref{Cor:SubsetOfTilde} are again an immediate consequence, and Corollary~\ref{CorItem:PTildeSubset} can be adapted. The proof of Lemma~\ref{Lemma:Blocks} can be adapted, such that we obtain the error bound stated in \eqref{Eq:JointBound}.
For each $l$, we can construct a good event exactly as in the case where only one threshold $p$ and tolerance $h$ is considered, so Lemma~\ref{Lemma:GoodEventEnvelopes} and Lemma~\ref{Lemma:Envelopes} can be applied to all sets $Q^*_{p_l,h_l}(E_t^{k-1})$ and $Q_{t,l}^{k-1}$ with $v_t^{k-1}$ as well.

So in the proof of Theorem~\ref{Thm:Properties}, what changes is the number of concentration inequalities considered, leading to a final event of probability at least $1-m\delta/2$ instead of $1-\delta/2$ due to a union bound, and a slightly changed statement of Theorem~\ref{ThmItem:Envelope}, such that
\begin{align*}
    Q_{t,l}^{k-1}\subseteq Q^*_{p_l,h_l}(E_t^{k-1}).
\end{align*}
and again
\begin{align*}
    \sum_{t\geq 1, \ w_t^{k-1}=1}|Q_{t,l}^{k-1}|\leq 3d
\end{align*}
for $l=1,\dots,m$. On this newly constructed event, Corollary~\ref{Cor:PropertiesTree} changes slightly. For the first three points, our results can be reformulated with respect to sets $Q'\in\bigcup_{l=1}^m\mathcal{Q}_{t,l}^{k-1}$. Corollary~\ref{CorItem:PTildeSubset} then holds if we consider 
\begin{align*}
    \overline P _l(E_t^{k-1})\coloneqq \left\{i\in E_t^{k-1}:\ \frac{1}{|Q'|}\left|\sum_{j\in Q'} Y_{ij}^{(3k)}-Y_{\mmed{t}{k-1}{Q'}j}^{(3k)}\right|\leq 8\rho\sqrt{\lambda_1/|Q'|}\hspace{.5cm }\forall Q'\in \mathcal{Q}_{t,l}^{k-1}\right\}
\end{align*}
as replacement for $\overline P$, and for \ref{CorItem:ErrorBound} we have
\begin{align}
    \tilde{R}_{p_l,h_l,\overline P_l(E_t^{k-1})[d]}\leq 3744\rho(|E_t^{k-1}|\vee|Q_{t,l}^{k-1}|)/\lambda_1h_l^2\enspace ,\label{Eq:JointBound}
\end{align}
both jointly for all $l=1,\dots,m$.

Lemma~\ref{Lemma:TauNotInO} and Lemma~\ref{Lemma:TauInP} can also be stated with respect to sets of the form $\overline P_l(E_u^l)$, and with the adaptations in Corollary~\ref{Cor:PropertiesTree} just stated we can also adapt the proofs of the lemmas.

The proof of Corollary~\ref{cor:error_bound:permutation:multiple} can be concluded just like the proof of Theorem~\ref{Thm:ErrorBoundRows} and Theorem~\ref{Thm:ErrorBound} for each $l=1,\dots,m$, and substituting $\delta$ by $\delta/m$ finally yields the claimed bound with probability $\geq 1-\delta$.

    \section{Proof for the classification matrix}\label{Sec:proof:reconstruction}

\begin{proof}[Proof of Theorem~\ref{Thm:ErrorBound_classificatoin}]
This theorem is a straightforward corollary of \Cref{cor:error_bound:permutation:multiple} and \Cref{prp:reconstruction}.
\end{proof}

\begin{proof}[Proof of \Cref{prp:reconstruction}]
We consider three cases depending on the values of  $(k_h,l_h)$. First, we assume that $k_h\vee l_h >1$ and that $k_h\leq n$ and $l_h\leq d$. The simple cases where $k_h=l_h=1$ or $k_h\geq n$ or $l_h\geq d$ are postponed to the end of the proof. 

\medskip

 For the sake of simplicity we simply write $\mathcal{L}_{p,h}$ for $\mathcal{L}_{p,h}(\hat{\pi},\hat{\eta})$
Let us upper bound the reconstruction error of $\hat{R}_{p,h}$.
\begin{equation}\label{eq:upper_l_0_1}
L_{0,1,\texttt{NA}}[\hat R_{p,h}]= \sum_{(i,j): M_{ij}\leq p-h} \mathbbm{1}\{(\hat{R}_{p,h})_{ij}=1\} + \sum_{(i,j): M_{ij}\geq p+h} \mathbbm{1}\{(\hat{R}_{p,h})_{ij}=0\}= (I) + (II)\enspace . 
\end{equation}
By symmetry, we focus on the second term $(II)$ of the rhs. Define $U_0\coloneqq \{(i,j): M_{\hat{\pi}^{-1}(i)\hat{\eta}^{-1}(j)}\geq p+h\}$ as the level set of the matrix $M_{\hat{\pi}^{-1},\hat{\eta}^{-1}}$ ordered according to the estimated rankings. We also define $V_1\coloneqq\{(i,j): M_{\pi^{-1}(i)\eta^{-1}(j)}\geq p+h/2\}$ as the level set of the oracle ordered matrix $M_{\pi^{-1},\eta^{-1}}$ at $p+h/2$. 

By definition of the loss function, we know that 
$|U_0\setminus V_1|\leq \mathcal{L}_{p+3h/4,h/4}$. Since $M_{\pi^{-1},\eta^{-1}}$ is a bi-isotonic matrix, as long as $V_1$ is non-empty, $V_1$ is a connected subset of $[n]\times [d]$ that contains $(1,1)$. We call $\mathcal{B}_0$ the collection of $k_h\times l_h$ blocks that are fully included in $V_1$ and $V'_1$ the subset of $V_1$ which does not belong to any of the block in $\mathcal{B}_0$. In fact, $V'_1$ corresponds to the boundary of the level set $V_1$ so that  
\[
|V'_1|\leq k_hl_h (\lceil n/k_h\rceil+ \lceil d/l_h\rceil )\leq c (\sigma^2\vee 1) \frac{n\vee d}{\lambda_0 h^2}\log(nd)\ , 
\]
for some universal constant $c$ since we assume that $k_h\vee l_h >1$. Since $(II)$ corresponds to the number of entries in $U_0$ of the block constant matrix $\overline{Y}^{\hat{\pi}, \hat{\eta}}_B$ that are below $p$, we arrive for the error $(II)$ at the following bound in~\eqref{eq:upper_l_0_1}
\begin{align} \nonumber 
 (II)&\leq \sum_{B\in \mathcal{B}_0} |B|\mathbbm{1}\{\overline{Y}^{\hat{\pi}, \hat{\eta}}_B\leq p\}  + |V'_1| + |U_0\setminus V_1|\\
 &\leq  \sum_{B\in \mathcal{B}_0}|B| \mathbbm{1}\{\overline{Y}^{\hat{\pi}, \hat{\eta}}_B\leq p\} + \mathcal{L}_{p+3h/4,h/4} + c (\sigma^2\vee 1)\frac{n\vee d}{\lambda_0 h^2}\log(nd) \label{eq:upper_(I)}\enspace . 
\end{align}
Let us now define the level sets  $W_{0}\coloneqq  \{(i,j)\in V_1: M_{\hat{\pi}^{-1}(i)\hat{\eta}^{-1}(j)}\in (p-h, p+h/4]\}$ and, for $s=1,\ldots , s_{\max}$ where $s_{\max}:= \lceil \log_2(1/h)\rceil$, we define 
$W_{s}\coloneqq \{(i,j)\in V_1: (M_{\hat{\pi}^{-1}(i)\hat{\eta}^{-1}(j)}\in [p-h2^{s}, p-h2^{s-1})\}$. Again by definition of the loss functions and of $V_1$, we have 
\begin{align}\label{eq:control:W_0}
|W_0\cap V_1|\leq \mathcal{L}_{p+3h/8,h/8}; \quad \quad  |W_s\cap V_1|\leq \mathcal{L}_{p-h2^{s-2},h2^{s-2}}\ , 
\end{align}
for $s=1,\ldots, s_{\max}$. Recall that the observations are $\sigma^2$-subGaussians. Conditionally to $\hat{\pi}$ and $\hat{\eta}$, we have, with probability higher than $1/(nd)^2$, for any $B\in \mathcal{B}_0$, 
\[
\overline{Y}^{\hat{\pi}, \hat{\eta}}_{B} -p \geq h/4 - \sigma \sqrt{\frac{4\log(nd)}{N^{\hat{\pi}, \hat{\eta}}_B}} - \frac{3}{2N^{\hat{\pi}, \hat{\eta}}_B}\sum_{s=0}^{s_{\max}}N^{\hat{\pi}, \hat{\eta}}_{W_s\cap B}h 2^s \ , 
\]
where $N^{\hat{\pi}, \hat{\eta}}_{W_s\cap B}= \sum_{t=1}^{N'}\sum_{(k,l)\in  W_s\cap B}\mathbbm{1}\{(I_t,J_t)=(\hat{\pi}^{-1}(k),\hat{\eta}^{-1}(l))\}$ is the number of observations in $W_s\cap B$.
The random variables $N_B$ follow independent Poisson distribution with parameter $\lambda_0k_h l_h$. By Bennett's inequality, $\mathbb{P}[N_B\leq \lambda_0/2 k_h l_h]\leq \exp[-0.15 \lambda_0 k_h l_h]$.  Hence, with probability higher than $1/(nd)^2$, we have $N_B\geq \lambda_0/2 k_h l_h$ for all $B\in \mathcal{B}_0$. Hence, it follows that 
\begin{align*}
\overline{Y}^{\hat{\pi}, \hat{\eta}}_{B} -p &\geq h/4 -\sigma \sqrt{\frac{8\log(nd)}{\lambda_0k_hl_h}} - \frac{3h}{\lambda_0k_hl_h}\sum_{s=1}N^{\hat{\pi}, \hat{\eta}}_{W_s\cap B} 2^s \\ 
&\geq    h/8 +  \frac{3h}{\lambda_0k_hl_h}\sum_{s=1}N^{\hat{\pi}, \hat{\eta}}_{W_s\cap B} 2^s \ , 
\end{align*}
by definition of $k_h$ and $l_h$. Coming back to~\eqref{eq:upper_l_0_1} and applying Markov inequality, we deduce that 
\[
(II)\leq \frac{24}{\lambda_0 }\sum_{s=0}^{s_{\max}} N^{\hat{\pi}, \hat{\eta}}_{W_s} 2^s + c(\sigma^2\vee 1)\frac{n\vee d}{\lambda_0 h^2}\log(nd)+ \mathcal{L}_{p+3h/4,h/4}\ . 
\]
The random variables $N^{\hat{\pi}, \hat{\eta}}_{W_s}$ follow a Poisson distribution with parameters $\lambda_0|W_s|$.  By Bennett inequality, with probability higher than $1-1/(nd)^2$, we have, for any $s=0,\ldots, s_{max}$,  $N^{\hat{\pi}, \hat{\eta}}_{W_s}\leq c'\lambda_0 |W_s| + c''\log(nd)$, where $c'$ and $c''$ are absolute constants. Gathering this bound with~\eqref{eq:control:W_0}, we conclude that, with probability higher than $1-3/(nd)^2$, we have 
\[
(II)\leq c \left[ \mathcal{L}_{p+3h/4,h/4} +  \mathcal{L}_{p+3h/8,h/8}+  \sum_{s=1}^{s_{\max}} 2^{s} \mathcal{L}_{p-h2^{s-2},h2^{s-2}} + (\sigma^2\vee 1) \frac{n\vee d}{\lambda_0 h^2}\log^2(nd)\right] \enspace . 
\] 
By handling analogously $(I)$, we conclude that, with probability higher than $1-6/(nd)^2$, we have 
\begin{align*}
L_{0,1,\texttt{NA}}[ \hat{R}_{p,h}]& \leq  c\Big[\sum_{s=1}^{s_{\max}}2^s\left(\mathcal{L}_{p+h2^{s-2},h2^{s-2}} + \mathcal{L}_{p-h2^{s-2},h2^{s-2}}\right)\\ & + \mathcal{L}_{p+3h/4,h/4} +  \mathcal{L}_{p+3h/8,h/8}+ \mathcal{L}_{p-3h/4,h/4} +  \mathcal{L}_{p-3h/8,h/8}+  (\sigma^2\vee 1) \frac{n\vee d}{\lambda_0 h^2}\log^2(nd) \Big]\enspace . 
\end{align*}
The result follows since, on the remaining event, the loss is smaller than $nd$.

Finally, we consider the extreme cases.  First, assume $k_h=l_h=1$. In this case, each of the block has size $1$. Hence, for any entry $(i,j)$, $(\hat{R}_{p,h})_{ij}$ is simply $\mathbbm{1}\{Y_{ij}\geq p\}$. By Bennett's inequality, with probability higher than $1-1/(nd)^2$, we have $N'_{ij}\geq \lambda_0/2$. Then, by standard deviation inequality for subGaussian variables, we deduce, that with probability higher than $1-3/(nd)^2$, we have 
\[
|Y_{ij}-M_{ij}|\leq 4\frac{\sigma}{\lambda_0} \sqrt{\log(nd)} < 1/ h\ . 
\]
Hence, under this event, we have $L_{0,1,\texttt{NA}}[ \hat{R}_{p,h}]=0$. This concludes the proof for this case. Finally, we assume that either $k_h\geq n$ or $l_h\geq d$. Then, we use the trivial bound $L_{0,1,\texttt{NA}}[ \hat{R}_{p,h}]\leq nd$, since we have here $nd\lesssim (\sigma^2\vee 1)\frac{n\vee d}{\lambda_0 h^2}\log^2(nd)$.

\end{proof}

\section{Proof of the lower bound}\label{Sec:LowerBound}

\begin{proof}[Proof of Theorem~\ref{Thm:LowerBound}]
This proof is based on  Fano's method, see e.g. \cite{wainwright_high-dimensional_2019}, that we recall here. Let $(S,d_S)$ be a pseudometric space, $\theta_1,\theta_2,\dots,\theta_{\mathcal{M}}\in S$ such that $d_S(\theta_k,\theta_l)\geq 2D$ for some $D>0$ and $k\neq l$. Consider a family of probability measures $\mathbb{P}_{\theta_1},\mathbb{P}_{\theta_2},\dots,\mathbb{P}_{\theta_{\mathcal{M}}}$ with $KL(\mathbb{P}_{\theta_k}|\mathbb{P}_{\theta_l})\geq \kappa $ for $k\neq l$. Then for any estimator $\hat\theta\in S$, it holds \begin{align}      \max_{k=1,\dots,\mathcal{M}}\mathbb{E}_{\theta_k}\left[d_S(\hat{\theta},\theta)\right]\geq D\left(1-\frac{\kappa+\log 2}{\log|\mathcal{M}|}\right)\enspace .\label{Eq:Fano}
\end{align}

We recall the definition of our observation model from Section~\ref{Def:Model}. Consider $\lambda_0>0$ to be the expected number of observations we make for every entry $(i,j)\in[n]\times [d]$, so that in total we expect to have $N=\lambda_0nd$ observations.
In our setting, $\mathbb{P}_M$ is a distribution such that given the matrix $M$ we have observations $(N',I,J,Y')$ of the following form:
\begin{itemize}
	\item $N'\sim \mathrm{Pois}(N)$ is the number of observations we have,
	\item $(I,J)|(N'=m)\sim \mathcal{U}([n]^m\times[d]^m)$ is a vector of entries we observe
	\item $Y|(N'=m, (I,J)=((i_1,\dots,i_m),(j_1,\dots,j_m)))\sim \bigotimes_{k=1}^m \mathcal{N}(M_{i_kj_k},\sigma^2)$
\end{itemize}
Finally, we want to bound the Kullback-Leibler divergence of two distributions $\mathbb{P}_M$ and $\mathbb{P}_{M'}$ relating to samples  from the described observation model:
\begin{lemma}\label{Lemma:KL}
	Consider the random vector $(N',I,J,Z)$ with distributions $\mathbb{P}_M$ and $\mathbb{P}_{M'}$. Then $KL(\mathbb{P}_M|\mathbb{P}_{M'})=\frac{\lambda_0}{2\sigma^2}\Vert M-M'\Vert_F^2$.
\end{lemma}

Let us consider $S=\{v\in \{0,1\}^n:\ \sum_{i=1}^nv_i=n/2\}$ (where we assume w.l.o.g. that $n$ is an even number), equipped with the Hamming distance $d_H(v,v')=|\{i\in[n]:\ v_i\neq v_i'\}|$. We define the surjective map 
\begin{align*}
	v:\ \mathcal{S}_n\to S,\ v(\pi)_i=\mathbbm{1}\{\pi(i)\leq n/2\}\enspace .
\end{align*}
For some $l>0$ to be chosen later, we also define the map
\begin{align*}
	M:\ S\to \bigcup_{\pi\in \mathcal{S}_n}\mathbb{C}_{\mathrm{Biso}}(\pi,\mathrm{id}_{[d]}),\ M(v)_{ij}=p-h+2h\cdot \mathbbm{1}\{v_i=1,\ j\leq l\}\enspace . 
\end{align*}
Then, for any $\pi\in \mathcal{S}_n$ and any estimator $(\hat{\pi},\hat{\eta})$, with $\hat{v}=v(\hat{\pi})$, we have 
\begin{align}\label{eq:lower:Lph}
	\mathbb{E}_{M(v(\pi))}\left[\mathcal{L}_{p,h}(\hat{\pi},\hat{\eta})\right]\geq l\cdot  \mathbb{E}_{M(v(\pi))}\left[d_H(v(\pi),\hat{v})\right]\enspace .
\end{align}
Similarly, for any estimator $\hat{R}_{p,h}$, if we define $\tilde{v}$ by $\tilde{v}_i=1$ if $\sum_{j=1}^d(\hat{R}_{p,h})_{ij}>l/2$. 
\begin{align}\label{eq:lower:L01}
	\mathbb{E}_{M(v(\pi))}\left[L_{0,1,\texttt{NA}}[\hat{R}_{p,h}]\right]\geq \frac{l}{2}\cdot  \mathbb{E}_{M(v(\pi))}\left[d_H(v(\pi),\tilde{v})\right]\enspace .
\end{align}

Next, we want to construct a packing of $S$. An adaptation of the Varshamov-Gilbert bound like in \cite{picard_concentration_2007} yields, that a collection of vectors $v_1$, $v_2$,$\dots$, $v_{\mathcal{M}}$ exists with $d_H(v_k,v_l)\geq n/4$ for $k\neq l$ and $\log\mathcal{M}\geq c\cdot n$ (where $c$ can be chosen as $0.08$).

Lemma~\ref{Lemma:KL} implies for  $k\neq l$, that
\begin{align*}
	KL(\mathbb{P}_{M(v_k)}|\mathbb{P}_{M(v_l)})=\frac{\lambda_0l h^2 d_H(v_k,v_l)}{2\sigma^2}\geq \frac{\lambda_0l h^2n}{8\sigma^2}\enspace .
\end{align*}
We now have  all ingredients to prove Theorem~\ref{Thm:LowerBound}. Indeed, it follows from Fano's method that, for any $\hat{v}$, we have  
\[
\max_{k=1,\dots,\mathcal{M}}\mathbb{E}_{M(v_k)}\left[d_H(v_k,\hat{v})\right]\geq \frac{n}{8}\left(1-\frac{\lambda_0l h^2n/8\sigma^2+\log(2)}{0.08\cdot n}\right)\ . 
\]
Let us choose $l =\lfloor  0.16\sigma^2/(\lambda_0h^2)\rfloor\wedge d$. Assuming that $\lambda_0 h^2\leq 0.16 \sigma^2$, we deduce that $l\geq 1$. If $n\geq 35$, we have $\log(2)\leq 0.08n/4$ and it follows that 
\[
\inf_{\hat{v}}\max_{k=1,\dots,\mathcal{M}}\mathbb{E}_{M(v_k)}\left[d_H(v_k,\hat{v})\right]\geq \frac{n}{16} \enspace . 
\]
In light of~\eqref{eq:lower:Lph} and~\eqref{eq:lower:L01}, this implies that 
   \begin{align}\label{eq:lower_bound_11}
    \inf_{\hat{R}_{p,h}} \quad    \sup_{M\in \mathbb{C}_{\mathrm{Biso}}} \mathbb{E}_{M}\left[L_{0,1,\texttt{NA}}[\hat{R}_{p,h}]\right]\geq c'\left(\frac{n\sigma^2}{\lambda_0h^2}\wedge nd\right)\enspace \ ,\\ \label{eq:lower_bound_12}
        \inf_{\hat{\pi},\hat{\eta}}   \quad  \sup_{ M\in \mathbb{C}_{\mathrm{Biso}}} \mathbb{E}_{M}\left[\mathcal{L}_{p,h}(\hat{\pi},\hat\eta)\right]\geq c'\left(\frac{n\sigma^2}{\lambda_0h^2}\wedge nd\right)\enspace \ ,
    \end{align}
where $c'$ is a positive universal constant. We handle the case where $n\in [2,35]$ slightly differently by simply considering two vector $v_1$ and $v_2$ such that $d_H(v_1,v_2)=n/2$ and using Le Cam's approach together with Pinsker's inequality. Hence, as long as, $n\geq 2$, we have proven~\eqref{eq:lower_bound_11} and~\eqref{eq:lower_bound_12} 
\begin{align*}
    \inf_{\hat{R}_{p,h}} \quad    \sup_{M\in \mathbb{C}_{\mathrm{Biso}}} \mathbb{E}_{M}\left[L_{0,1,\texttt{NA}}[\hat{R}_{p,h}]\right]\geq c'\left(\frac{n\sigma^2}{\lambda_0h^2}\wedge nd\right)\enspace \ ,\\
        \inf_{\hat{\pi},\hat{\eta}}   \quad  \sup_{ M\in \mathbb{C}_{\mathrm{Biso}}} \mathbb{E}_{M}\left[\mathcal{L}_{p,h}(\hat{\pi},\hat\eta)\right]\geq c'\left(\frac{n\sigma^2}{\lambda_0h^2}\wedge nd\right)\enspace \ ,
    \end{align*}
Exchanging the role of $n$ and $d$ concludes the proof. 
\end{proof}

\begin{proof}[Proof of Lemma~\ref{Lemma:KL}]
 	Consider the likelihoods
 	\begin{align*}
 		L_M(D)=e^{-N}\frac{N^m}{m!}\cdot\frac{1}{(nd)^m}\cdot \prod_{k=1}^m\frac{1}{\sqrt{2\pi\sigma}}\exp\left(-\frac{(y_k-M_{i_kj_k})^2}{2\sigma^2}\right)\enspace ,\\
 		L_{M'}(D)=e^{-N}\frac{N^m}{m!}\cdot\frac{1}{(nd)^m}\cdot \prod_{k=1}^m\frac{1}{\sqrt{2\pi\sigma}}\exp\left(-\frac{(y_k-{M'}_{i_kj_k})^2}{2\sigma^2}\right)\enspace .
 	\end{align*}
  with $D=(m, (i_1,\dots,i_m), (j_1,\dots,j_m),(y_1,\dots,y_m))$ being a possible realization of $(N',I,J,Y')$.
 	Then we obtain as log-likelihood-ratio the term
 	\begin{align*}
 		\log\left(\frac{L_M(N',I,J,Y')}{L_{M'}(N',I,J,Y')}\right)=\sum_{k=1}^{N'} \left(\frac{M_{I_kJ_k}-M'_{I_kJ_k}}{\sigma^2}Y_k-\frac{M_{I_kJ_k}^2-{M'}_{I_kJ_k}^2}{2\sigma^2}\right)\enspace .
 	\end{align*}
 	For the Kullback-Leibler divergence, this means
 	\begin{align*}
 		KL(\mathbb{P}_M|\mathbb{P}_{M'})&=\mathbb{E}_M\left[\log\left(\frac{L_M(N',I,J,Y')}{L_{M'}(N',I,J,Y')}\right)\right]\\
 		&=\mathbb{E}_{M}\left[\sum_{k=1}^{N'} \left(\frac{M_{I_kJ_k}-M'_{I_kJ_k}}{\sigma^2}Y_k-\frac{M_{I_kJ_k}^2-{M'}_{I_kJ_k}^2}{2\sigma^2}\right)\right]\\
 		&=\mathbb{E}_{M}\left[\mathbb{E}_M\left[\sum_{k=1}^{N'} \left(\frac{M_{I_kJ_k}-M'_{I_kJ_k}}{\sigma^2}Y_k-\frac{M_{I_kJ_k}^2-{M'}_{I_kJ_k}^2}{2\sigma^2}\right)\mid N', I, J\right] \right]\\
 		&=\mathbb{E}_M\left[\sum_{k=1}^{N'}\frac{(M_{I_kJ_k}-M'_{I_kJ_k})^2}{2\sigma^2}\right]\\
 		&=\sum_{m\geq 0}e^{-N}\frac{N^m}{m!}\cdot\mathbb{E}_M\left[\sum_{k=1}^{N'}\frac{(M_{I_kJ_k}-M'_{I_kJ_k})^2}{2\sigma^2}\mid N'= m\right]\\
 		&=\sum_{m\geq 0}e^{-N}\frac{N^m}{m!}\cdot\frac{1}{(nd)^m}\sum_{k=1}^m\sum_{(i_1,\dots,i_m)\in [n]^m}\sum_{(j_1,\dots,j_m)\in[d]^m}\frac{(M_{i_kj_k}-M'_{i_kj_k})^2}{2\sigma^2}\\
 		&= \sum_{m\geq 0}e^{-N}\frac{N^m}{m!}\cdot\frac{m}{nd}\frac{\Vert M-M'\Vert_F^2}{2\sigma^2}\\
 		&=\frac{N}{nd}\frac{\Vert M-M'\Vert_F^2}{2\sigma^2}\enspace ,
 	\end{align*}
 	where we used in the second last step the identity 
 	\begin{align*}
 		\sum_{k=1}^m\sum_{(i_1,\dots,i_m)\in [n]^m}\sum_{(j_1,\dots,j_m)\in[d]^m}\frac{(M_{i_kj_k}-M'_{i_kj_k})^2}{2\sigma^2}=m(nd)^{m-1}\sum_{i=1}^n\sum_{j=1}^d\frac{ (M_{ij}-M'_{ij})^2}{2\sigma^2}
 	\end{align*}
 	which can be proven with an induction over $m$.
\end{proof}

\section{Concentration of partial sums}\label{Sec:Concentration}
In this section, we state and prove a deviation bound for local averages of the observation matrix.  In what follows, we will have to control the deviations of partial row and column sums of an observed matrix $\tilde{Y}$ to their mean. \begin{lemma}\label{Lemma:ConcentrationSE}
	Consider $\tilde{Y}:= Y^{(1)}$ as described in Equation~\eqref{Eq:Model}. Let $T\subseteq[n]\times [d]$. For any $\delta>0$, 
	\begin{align*}
		\frac{1}{|T|}\left|\sum_{(i,j)\in T}\tilde{Y}_{ij}-\lambda_1 M_{ij}\right|\leq \sqrt{2(1\vee \sigma^2) e^2\log(2/\delta)\lambda_1/|T|}+2(1\vee \sigma)\log(2/\delta)/|T|
	\end{align*}
	holds with probability $\geq 1-\delta$, where $\lambda_1 =1-e^{-\lambda_0^-}$.
\end{lemma}

\begin{proof}[Proof of Lemma~\ref{Lemma:ConcentrationSE}]
	Write $\tilde Y_{i,j,k}$ for the $k$-th observation of coordinate $i,j$ that is used to construct $\tilde{Y}$  as described in Equation~\eqref{Eq:Model} - i.e.~$\tilde{Y}_{ij} = \frac{1}{N^{(i,j)}}\sum_{k=1}^{N^{(i,j)}} \tilde Y_{i,j,k}$. We use the decomposition $\tilde{Y}_{ij}-\lambda_1 M_{ij}=\alpha_{ij}+\beta_{ij}$ with 
	\begin{align*}
		\alpha_{ij} \coloneqq B^{(i,j)}\cdot \left(\frac{1}{N^{(i,j)}}\sum_{k=1}^{N^{(i,j)}}[\tilde Y_{i,j,k}-M_{ij}]\right), \\ \intertext{and}
		\beta_{ij}\coloneqq M_{ij}\cdot(B^{(i,j)}-\lambda_1)\enspace ,
	\end{align*}
	where the $N^{(i,j)}$ are i.i.d. $\mathrm{Pois}(\lambda_0^-)$ distributed and $B^{(i,j)}=\mathbbm{1}\{N^{(i,j)}>0\}\sim\mathrm{Ber}(\lambda_1)$  Bernoulli distributed with $\mathbb{P}(B^{(i,j)}=1)=1-\mathbb{P}(B^{(i,j)}=0)=\lambda_1$.
	By Markov's inequality and Cauchy--Schwarz, we obtain for $t>0$
	\begin{align}
		\mathbb{P}\left(\frac{1}{|T|}\sum_{(i,j)\in T}\tilde{Y}_{ij}-\lambda_1 M_{ij}>t\right)&=\mathbb{P}\left(\sum_{(i,j)\in T}\alpha_{ij}+\beta_{ij}>|T|t\right) \notag\\
		&\leq e^{-x|T|t}\mathbb{E}\left[\exp\left(x\cdot\left\{\sum_{(i,j)\in T}\alpha_{ij}+\beta_{ij}\right\}\right)\right] \notag \\
		&\leq e^{-x|T|t}\left\{\mathbb{E}\left[\exp\left(2x\cdot\sum_{(i,j)\in T}\alpha_{ij}\right)\right]\cdot \mathbb{E}\left[\exp\left(2x\cdot\sum_{(i,j)\in T}\beta_{ij}\right)\right]\right\}^{1/2}\enspace , \label{Eq:TailBoundSE}
	\end{align}
	where we will choose a suitable $x>0$.

	Now note that
	\begin{align*}
		&\mathbb{E}\left[\exp\left(2x\cdot\sum_{(i,j)\in T}\alpha_{ij}\right)\right]= \prod_{(i,j)\in T}\mathbb{E}\left[\exp\left(2x\cdot B^{(i,j)}\cdot \left(\frac{1}{N^{(i,j)}}\sum_{k=1}^{N^{(i,j)}}\tilde Y_{i,j,k}-M_{ij}\right)\right)\right]\enspace , \\ \intertext{and for $(i,j)\in T$}
		&\mathbb{E}\left[\exp\left(2x\cdot B^{(i,j)}\cdot \left(\frac{1}{N^{(i,j)}}\sum_{k=1}^{N^{(i,j)}}\tilde Y_{i,j,k}-M_{ij}\right)\right)\right] \\
		&= (1-\lambda_1)+\sum_{K\geq 1}\mathbb{P}(N^{(i,j)}=K)\mathbb{E}\left[\exp\left(2x\left(\frac{1}{K}\sum_{k=1}^K\tilde Y_{i,j,k}-M_{ij}\right)\right)\right]\\
		&\leq (1-\lambda_1)+\sum_{K\geq 1}\mathbb{P}(N^{(i,j)}=K)\exp(2x^2\sigma^2/K)\leq \lambda_1(\exp(2x^2\sigma^2)-1)+1\\
		&\leq \lambda_1 (\sigma x)^2 e^2 +1\leq \exp(\lambda_1 e^2 (\sigma x)^2)\enspace ,
	\end{align*}
	for $x\in[-1/\sigma, 1/\sigma]$, where we used $e^{2s^2}-1\leq s^2e^2$ for $|s|\leq 1$. So in total, we obtain
	\begin{align}
		\mathbb{E}\left[\exp\left(2x\cdot\sum_{(i,j)\in T}\alpha_{ij}\right)\right]\leq \exp(\lambda_1 |T|e^2 (\sigma x)^2) \ \text{for} \ x\in [-1/\sigma,1/\sigma]\enspace . \label{Eq:MGFAlpha}
	\end{align}
	For the second factor, note that for a random variable $B\sim \mathrm{Ber}(p)$ the moment generating function can be bounded via
	\begin{align*}
		\mathbb{E}\left[\exp(2y\cdot(B-p))\right]&=p\cdot e^{2y\cdot(1-p)}+(1-p)\cdot  e^{-2yp}\\
		&= e^{-2yp}(p(e^{2y}-1)+1)\\
		&\leq \exp(-2yp+pe^{2y}-1)\\
		&\leq \exp(p(e^{2y}-2y-1))\\
		&\leq \exp(pe^2y^2) \ \text{for}\ |y|\leq 1\enspace ,
	\end{align*}
 	where we used that $e^{2y}-2y-1\leq e^2y^2$ for $|y|\leq 1$. This implies
	\begin{align}
		\mathbb{E}\left[\exp\left(2x\cdot\sum_{(i,j)\in T}\beta_{ij}\right)\right]&=\prod_{(i,j)\in T}	\mathbb{E}\left[\exp\left(2xM_{ij}\cdot(B^{(i,j)}-\lambda_1)\right)\right]\notag\\
		&\leq \prod_{(i,j)\in T}	\exp(\lambda_1e^2x^2M_{ij}^2) \hspace{1cm} (\text{for $|x|\leq 1$)}\notag\\
		&\leq\exp(\lambda_1e^2x^2|T|) \hspace{1cm} (\text{for $|x|\leq 1$})\label{Eq:MGFBeta}\enspace .
	\end{align}
	Let $\kappa\coloneqq 1\vee \sigma$. Then plugging \eqref{Eq:MGFAlpha} and \eqref{Eq:MGFBeta} into \eqref{Eq:TailBoundSE} yields 
	\begin{align*}
		\mathbb{P}\left(\frac{1}{|T|}\sum_{(i,j)\in T}\tilde{Y}_{ij}-\lambda_1 M_{ij}>t\right)&\leq \exp\left(\lambda_1e^2x^2(\sigma^2+1)|T|/2-x|T|t\right)\ \text{for $x\in [-1/\kappa,1/\kappa]$}\\
		&\leq \exp\left(\lambda_1e^2x^2\kappa^2|T|-x|T|t\right)\ \text{for $x\in [-1/\kappa,1/\kappa]$}\\
		&\leq \exp\left(-\frac{|T|}{2\kappa}\left(\frac{t^2}{\lambda_1\kappa e^2}\wedge t \right)\right)\enspace .
	\end{align*}
	We conclude the proof by claiming that the right hand side is bounded by $\delta/2$ for the value $t=\sqrt{2\kappa^2 e^2\log(2/\delta)\lambda_1/|T|}+2\kappa\log(2/\delta)/|T|$ and that in the same way, we have shown
    \begin{align*}
		\mathbb{P}\left(\frac{1}{|T|}\sum_{(i,j)\in T}\tilde{Y}_{ij}-\lambda_1 M_{ij}>t\right)\leq \delta/2 \\ \intertext{one can show}
        \mathbb{P}\left(\frac{1}{|T|}\sum_{(i,j)\in T}\lambda_1 M_{ij} -\tilde{Y}_{ij}>t\right)\leq \delta/2 \enspace . &\qedhere
	\end{align*}
\end{proof}

\end{document}